%% file: lglasso-hal.tex
\documentclass[twoside,11pt,pdflatex]{article}

\usepackage{jmlr2e-hal}
\usepackage{amsfonts}
\usepackage{amsmath}
\usepackage{array}
\usepackage{colortbl}
\usepackage{algorithm}
\usepackage{algorithmic}
\usepackage{tabularx}
\usepackage[colorlinks]{hyperref}
\usepackage{hypernat}
\definecolor{NavyBlue}{RGB}{35,35,142}
\definecolor{RawSienna}{RGB}{199,97,20}
\hypersetup{
    colorlinks,%
    citecolor=NavyBlue,%
    filecolor=NavyBlue,%
    linkcolor=RawSienna,%
    urlcolor=NavyBlue
}

\newcommand{\thmref}[1]{Theorem~\ref{#1}}
\newcommand{\lemref}[1]{Lemma~\ref{#1}}

\newcommand{\secref}[1]{Section~\ref{#1}}
\newcommand{\appref}[1]{Appendix~\ref{#1}}
\newcommand{\figref}[1]{Figure~\ref{#1}}

\newcommand{\OMIT}[1]{}
\newcommand{\1}{{\bf 1}}
\newcommand {\abs}[1]{\left\vert\, #1 \,\right\vert}
\newcommand{\A}{\mathcal{A}}
\newcommand {\br}[1]{\left(#1\right)}
\newcommand {\sqb}[1]{\left[#1\right]}
\newcommand {\cbr}[1]{\left\{#1\right\}}

\newcommand{\E}{\varepsilon}
\newcommand{\eqdef}{\overset{\, \Delta}{=}}
\newcommand{\G}{\mathcal{G}}

\newcommand{\lonetwo}{\ell_1/\ell_2}
\newcommand{\lone}{\ell_1}
\newcommand{\ltwo}{\ell_2}
\newcommand {\nm}[1]{\left\|#1\right\|}
\newcommand{\supp}[1]{\text{supp}\left(#1\right)}
\newcommand{\gsupp}[1]{\text{gsupp} \left(#1\right)}
\newcommand{\ie}{\emph{i.e.}}

\newcommand{\Om}[1]{\Omega_{\cup}^{\G}\br{#1}}
\newcommand{\Omn}{\Omega_{\cup}^{\G}} 
\newcommand{\Omone}[1]{\Omega_{\cup}^{\Gw(\ws)}\br{#1}}
\newcommand{\Ogroup}[1]{\|#1\|_{\ell_1/\ell_2}}

\newcommand{\Oo}{\Omega}
\newcommand{\Op}{\mathcal{O}_p}
\newcommand{\PP}{\mathcal{P}}

\newcommand{\PPP}{\mathbb{P}}
\newcommand{\EE}{\mathbb{E}}
\newcommand{\RR}{\mathbb{R}}    

\newcommand{\vb}{\mathbf{\bar{v}}}
\newcommand{\Vb}{\mathbf{V}}
\newcommand{\Zb}{\mathbf{Z}}
\newcommand{\VV}{\mathcal{V}}
\newcommand{\XX}{\boldsymbol{\Sigma}}

\newcommand{\K}{\mathbf{K}}

\newcommand{\VG}{\mathcal{V}_\mathcal{G}}

\def\w{\mathbf{w}}
\def\v{\mathbf{v}}
\def\h{\mathbf{h}}
\def\s{\mathbf{s}}
\def\b{\mathbf{b}}
\def\sign{\text{sign}}
\def\st{\text{s.t.}}
\def\Gw{\G_1}  
\def\Gwc{\G_2} 
\def\Gs{\breve{\G}_1}  
\def\Jw{J_1}  
\def\Jwc{J_2}  
\def\Js{\breve{J}_1}  
\def\m{m} 
\def\wh{\hat{\w}}
\def\prox{\text{ST}}
\def\H{\mathcal{H}}
\def\Hb{\mathbf{H}}
\def\x{\mathbf{x}}
\def\y{\mathbf{y}}
\def\u{\mathbf{u}}
\def\ws{\w^\star}  
\def\vbs{\bar{\v}^\star} 
\def\vs{\v^\star} 
\def\lambdab{\boldsymbol{\lambda}}
\def\Lambdab{\boldsymbol{\Lambda}}
\def\alphab{\boldsymbol{\alpha}}
\def\zetab{\boldsymbol{\zeta}}
\def\lambdag{\lambdab_g}
\def\lambdapg{\lambdab'_g}
\def\betab{\boldsymbol{\beta}}
\def\etab{\boldsymbol{\eta}}
\def\Db{\mathbf{D}}
\def\Qb{\mathbf{Q}
}

\def\B{\mathbf{B}}
\def\q{\mathbf{q}}
\def\zv{\mathbf{0}} 
\def\X{\mathbf{X}}
\def\qn{\q} 
\def\epsb{\boldsymbol{\varepsilon}}
\def\al{\alpha}
\def\l{\lambda}
\def\Sign{{\rm Sign}}
\def\infconv{\star_{\inf}}

\def\Omax{\Oo^*}
\newcommand{\out}[1]{}

\newcolumntype{i}{>{\centering\arraybackslash} m{.2\linewidth} }

\definecolor{white}{rgb}{1., 1., 1.}
\definecolor{lightblue}{rgb}{0.,0.3,0.8}
\definecolor{lightgray}{rgb}{0.8,0.8,0.8}

\newcommand{\colorcellthree}[4]
{\fcolorbox{#1}{#2}{\makebox(7,21)[c]{\textcolor{#3}{$#4$}}}}
\newcommand{\colorcellfive}[4]
{\fcolorbox{#1}{#2}{\makebox(7,35)[c]{\textcolor{#3}{$#4$}}}}
\newcommand{\colorcellsix}[4]
{\fcolorbox{#1}{#2}{\makebox(7,42)[c]{\textcolor{#3}{$#4$}}}}

\firstpageno{1}

\begin{document}

\title{Group Lasso with Overlaps: \\ the Latent Group Lasso approach}

\author{}

 \author{{\name Guillaume Obozinski}\thanks{Equal contribution} \email guillaume.obozinski@ens.fr \\
        \addr Sierra team - INRIA\\
        Ecole Normale Sup\'{e}rieure\\
        (INRIA/ENS/CNRS UMR 8548)\\
        Paris, France
        \AND
        \name Laurent Jacob$^*$ \email laurent@stat.berkeley.edu \\
        \addr Department of Statistics \\
        University of California  \\
        Berkeley CA 94720, USA
        \AND
        \name Jean-Philippe Vert \email Jean-Philippe.Vert@mines.org \\
        \addr Centre for Computational Biology \\
        Mines ParisTech\\
        Fontainebleau, F-77300, France\\
        INSERM U900 \\
        Institut Curie \\
        Paris, F-75005, France
        }

\editor{}

\maketitle

\begin{abstract}
\input{abstract}
\end{abstract}


\section{Introduction}

\input{intro}

\section{Group Lasso with overlapping groups}
\input{lglasso_def_prop}

\subsection{Algorithms}
\input{algos}

\section{Group-support}
\input{group-support}

\section{Illustrative examples}\label{sec:examples}
\input{examples}

\section{Model selection consistency}\label{sec:consistency}
\input{consistency}

\section{Choice of the weights}\label{sec:weights}
\input{weights}

\section{Graph Lasso}
\input{graph_lasso}

\section{Experiments}
\input{exps}

\section{Conclusion}
\input{discussion}

\acks{LJ gratefully acknowledges the support of the Stand Up to Cancer
  Program. JPV was supported by ANR grants ANR-07-BLAN-0311-03 and
  ANR-09-BLAN-0051-04. GO acknowledges funding from the European Research Council grant SIERRA: Project 239993. The authors would like to thank Rodophe Jenatton, Julien Mairal and Francis Bach for useful discussions.
}

\vskip 0.2in
\bibliography{lglasso-jmlr}

\input{appendices2}

\end{document}

%% file: abstract.tex
We study a norm for structured sparsity which leads to sparse linear predictors whose supports are unions of predefined overlapping groups of variables. We call the obtained formulation \emph{latent group Lasso}, since it is based on applying the usual group Lasso penalty on a set of latent variables. A detailed analysis of the norm and its properties is presented and we characterize conditions under which the set of groups associated with latent variables are correctly identified. We motivate and discuss the delicate choice of weights associated to each group, and illustrate this approach on simulated data and on the problem of breast cancer prognosis from gene expression data.

%% file: intro.tex
Sparsity has triggered much research in statistics, machine learning and signal processing recently. Sparse models are attractive in many application domains because they lend themselves particularly well to interpretation and data compression. Moreover, from a statistical viewpoint, betting on sparsity is a way to reduce the complexity of inference tasks in large dimensions with limited amounts of observations.
While sparse models have traditionally been estimated with
greedy feature selection approaches, more recent formulations as optimization problems involving a non-differentiable convex penalty have proven very successful both theoretically and practically. The canonical example is
the penalization of a least-square criterion by the $\lone$ norm of the
estimator, known as \emph{Lasso} in statistics \citep{Tibshirani1996Regression} or
\emph{basis pursuit} in signal processing \citep{Chen1998Atomic}. Under appropriate assumptions, the Lasso can be shown to recover the exact support of a sparse model from data
generated by this model if the covariates are not too correlated
\citep{Zhao2006model,Wainwright2009Sharp}. It is consistent even in high dimensions, with fast rates of convergence \citep{Lounici2008Sup-norm,Bickel2009Simultaneous}. We refer the reader to \citet{Geer2010L1} for a detailed review.

While the $\lone$ norm penalty leads to sparse models, it does not
encode any prior information about the structure of the sets of
covariates that one may wish to see selected jointly, such as predefined groups of covariates. An extension of the Lasso for the
selection of variables in groups was proposed under the name group
Lasso by \cite{Yuan2006Model}, who considered the case where the
groups form a partition of the sets of variables. The group Lasso
penalty, also called $\lonetwo$ penalty, is defined as the sum (i.e.\
, $\lone$ norm) of the $\ltwo$ norms of the restrictions of the
parameter vector of the model to the different groups of covariates.
The work of several authors shows that when the support can be encoded well by the groups
defining the norm, support recovery and estimation are improved \citep{Lounici2010Oracle,Huang2010benefit,Obozinski2010Joint,Negahban2011Simultaneous,Lounici2009Taking,Kolar2011Union}.

Subsequently, the notion of \emph{structured sparsity} emerged  as a natural generalization of the selection in groups, where the support of the model one wishes to recover is not anymore required to be just sparse but also to display certain structure. One of the first natural approaches to \emph{structured sparsity} has been to consider extensions of the $\lonetwo$ penalty to situations in which the set of groups considered overlap, so that the possible support pattern exhibit some structure
\citep{Zhao2008Grouped,Bach2009Exploring}. \citet{Jenatton2011Proximal} formalized this approach and proposed an $\lonetwo$ norm construction for families of allowed supports stable by \emph{intersection}. Other approaches to structured sparsity are quite diverse: Bayesian or non-convex approaches that directly exploit the recursive structure of some sparsity patterns such as trees~\citep{He2009Exploiting, Baraniuk2010Model}, greedy approaches based on \emph{block-coding} \citep{Huang2009Learning}, relaxation of submodular penalties \citep{Bach2010Structured}, generic variational formulations \citep{Micchelli2011Regularizers}.

While \citet{Jenatton2011Proximal} proposed a norm inducing supports
that arise as \emph{intersections} of a sub-collection of groups
defining the norm, we consider in this work norms which, albeit
defined as well by a collection of overlapping groups, induce supports
that are rather \emph{unions} of a sub-collection of the groups
encoding prior information. The main idea is that instead of directly
applying the $\lonetwo$ norm to a vector, we apply it to a set of
latent variables each supported by one of the groups, which are
combined linearly to form the estimated parameter vector. In the
regression case, we therefore call our approach \emph{latent group
  Lasso}.

The corresponding decomposition of a parameter vector into latent
variables calls for the notion of \emph{group-support}, which we
introduce and which corresponds to the set of non-zero latent
variables. In the context of a learning problem regularized by the
norm we propose, we study the problem of \emph{group-support
  recovery}, a notion stronger than the classical support recovery. Group-support recovery typically
implies support recovery (although not always) if the support of a parameter
vector is exactly a union of groups. We provide sufficient conditions for consistent group-support recovery.

In the definition of our norm, a
weight is associated with each group. These weights play a much more
important role in the case of overlapping groups than in the case of
disjoint groups, since in the former case they determine the set of
recoverable supports and the complexity of the class of possible
models. We discuss the delicate question of the choice of these
weights.

While the norm we consider is quite general and has potentially many applications, we illustrate its potential on the particular problem of learning sparse predictive models for cancer prognosis from high-dimensional gene expression data. The problem of identifying a predictive molecular signature made of a small set of genes is often
ill-posed and so noisy that exact variable selection may be elusive. We propose that, instead,
selecting genes in groups that are involved in the same
biological process or connected in a functional or interaction network could be
performed more reliably, and potentially lead to better predictive models. We
empirically explore this application, after extensive experiments on
simulated data illustrating some of the properties of our norm.

To summarize, the main contributions of this paper, which rephrases and extends a preliminary version published in  \citet{Jacob2009Group}, are the following:
\begin{itemize}
\item We define the latent group Lasso penalty to infer sparse models with unions of predefined groups as supports, and analyze in details some of its mathematical properties.
\item We introduce the notion of \emph{group-support} and group-support recovery results. Using correspondence theory, we show under appropriate conditions, that, in a classical asymptotic setting, estimators for the linear regression regularized with $\Omn$ are consistent for the estimation of a sufficiently sparse \emph{group-support}. 
\item We discuss in length the choice of weights associated to each group, which play a crucial role in the presence of overlapping groups of different sizes.
\item We provide extended experimental results both on simulated data --- addressing support-recovery, estimation error and role of weights --- and on breast cancer data, using biological pathways and genes networks as prior information to construct latent group Lasso formulations.
\end{itemize}
 
The rest of the paper is structured as follows.
We first introduce the latent group Lasso penalty and position it in the context of related work in Section~\ref{sec:def}. In Section~\ref{sec:prop} we show that it is a norm and provide several characterizations and variational formulations; we also show that regularizing with this norm is equivalent to covariate duplication (Section~\ref{sec:duplication}) and derive a corresponding multiple kernel learning formulation (Section~\ref{sec:mkl}).
We briefly discuss algorithms in Section~\ref{sec:implementation}. 
In Section~\ref{sec:group-support}, we introduce the notion of \emph{group-support} and consider in Section~\ref{sec:examples} a few toy examples to illustrate the concepts and properties discussed so far. We study group support-consistency in Section~\ref{sec:supp_rec}. The difficult question of the choice of the weighting scheme is discussed in Section~\ref{sec:weights}. Section~\ref{sec:graphlasso} presents the latent graph Lasso, a variant of the latent group Lasso when covariates are organized into a graph. Finally, in Section~\ref{sec:experiments}, we present several experiments: first, on artificial data to illustrate the gain in support recovery and estimation over the classical Lasso, as well as the influence of the choice of the weights; second, on the real problem of breast cancer prognosis from gene expression data.

\section{Notations}
\label{sec:notations}
In this section we introduce notations that will be used throughout
the article. For any vector $\w \in \RR^p$ and any $q\geq 1$,
$\nm{\w}_q = \left(\sum_{i=1}^p \abs{\w_i}^q\right)^{1/q}$ denotes the
$\ell_q$ norm of $\w$. We simply use the notation $\nm{\w} =
\nm{\w}_2$ for the Euclidean norm. $\supp{\w}\subset[1,p]$ denotes the
support of $\w$, \ie, the set of covariates $i \in [1,p]$ such that
$\w_i \neq 0$.  A \emph{group} of covariates is a subset
$g\subset[1,p]$. The set of all possible groups is therefore
$\PP([1,p])$, the power set of $[1,p]$. For any group $g$, $g^c =
[1,p]\backslash g$ denotes the complement of $g$ in $[1,p]$, \ie, the
covariates which are not in $g$. $\Pi_g : \RR^p \rightarrow \RR^p$
denotes the projection onto $\cbr{\w\,:\,\w_i = 0\,\text{ for }\,i \in
  g^c}$, \ie, $\Pi_g \w$ is the vector whose entries are the same as
$\w$ for the covariates in $g$, and are $0$ for other other
covariates. We will usually use the notation $\w_g \eqdef \Pi_g
\w$. We say that two groups \emph{overlap} if they have at least one
covariate in common.

Throughout the article, $\G \subset \PP([1,p])$ denotes a set
of groups, usually fixed in advance for each application, and we
denote $\m\eqdef |\G|$ the number of groups in $\G$.  We require that all covariates belong to at
least one group, \ie,
\[
\bigcup_{g\in\G} g = [1,p]\,.
\]
We note $\VV_\G\subset \RR^{p\times\G}$ the set of
$\m$-tuples of vectors $\vb = \br{\v^g}_{g\in\G}$, where each $\v^g$ is a vector in $\RR^p$, that satisfy
$\supp{\v^g}\subset g$ for each $g\in\G$.

For any differentiable function
$f:\RR^p\rightarrow\RR$, we denote by $\nabla f(\w) \in\RR^p$ the
gradient of $f$ at $\w\in\RR^p$ and by $\nabla_{\!g} f(\w) \in\RR^g$ the
partial gradient of $f$ with respect to the covariates in $g$.

In optimization problems throughout the paper we will use the convention that $\frac{0}{0}=0$ so that the $\bar{\RR}$-valued function $(x,y) \mapsto \frac{x^2}{y}$ is well defined and jointly convex on $\RR \times\RR_+$.

%% file: lglasso_def_prop.tex
\label{sec:def}
Given a set of groups $\G$ which form a partition of $[1,p]$, the group Lasso penalty
\citep{Yuan2006Model} is a norm over $\RR^p$ defined as~:
\begin{equation}\label{eq:grouplasso}
\forall \w\in\RR^p\,,\quad\Ogroup{\w} = \sum_{g\in\G}\, d_g \, \nm{\w_g}\,,
\end{equation}
where $\br{d_g}_{g\in\G}$ are positive weights. This is a norm whose balls have singularities when some
$\w_g$ are equal to zero. Minimizing a smooth convex loss functional
$L:\RR^p\rightarrow\RR$ over such a ball, or equivalently solving the following optimization problem for some $\lambda>0$~:
\begin{equation}
  \min_{\w\in\RR^p} L(\w) + \lambda \sum_{g\in\G} \, d_g \, \nm{\w_g}\,,
\label{eq:glm}
\end{equation}
often leads to a solution that lies on a singularity, \ie, to
a vector $\w$ such that $\w_g=\zv$ for some of the groups $g$ in $\G$. Equivalently, the solution is sparse at the group level, in the sense that coefficients within a group are usually zero or nonzero together. The
hyperparameter $\lambda \geq 0$ in~\eqref{eq:glm} is used to adjust
the tradeoff between minimizing the loss and finding a solution which
is sparse at the group level.

\begin{figure}
  \centering
  \begin{tabular}{m{1cm}m{2.2cm}m{2cm}m{8cm}}
  \includegraphics[height=.22\textheight]{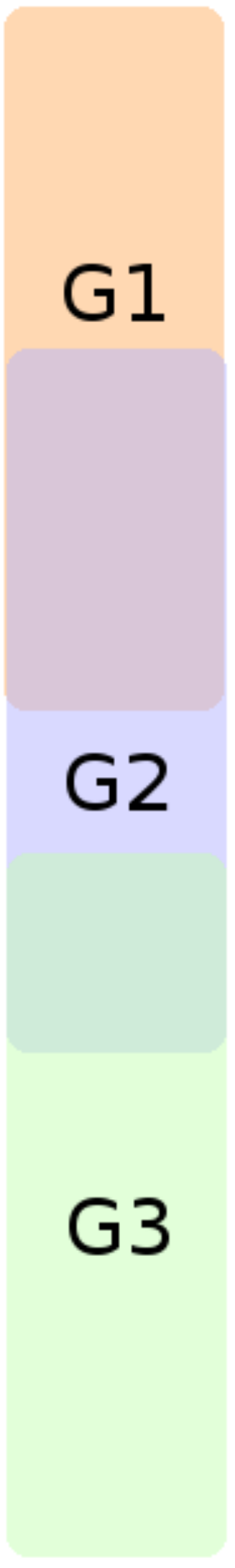}&
  {\centering 
$\displaystyle \begin{array}{c}
      \Rightarrow\\
      \nm{\w_{g_1}} =0\\
       \nm{\w_{g_3}} = 0
    \end{array}$}
  &  \includegraphics[height=.22\textheight]{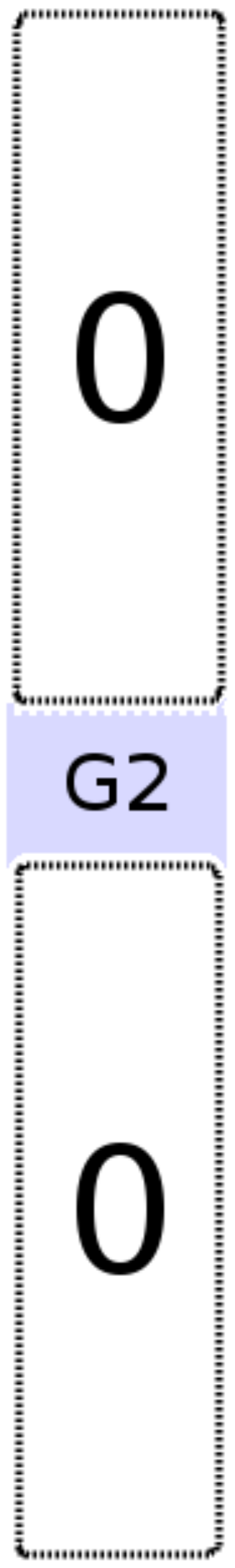}
  &   \hspace*{15mm}\includegraphics[height=.22\textheight]{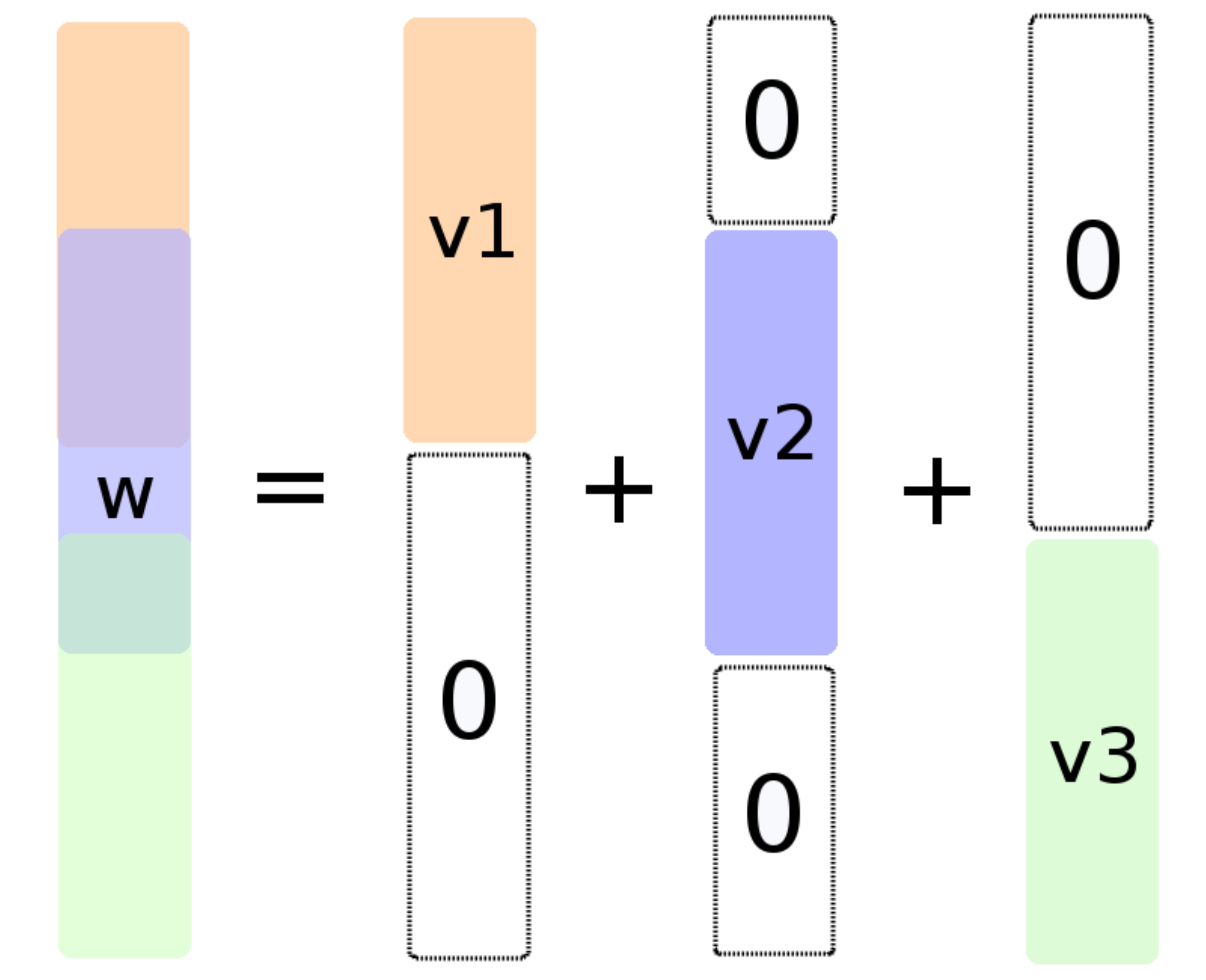}
  \end{tabular}
  \caption{(a) \textit{Left}: Effect of penalty~\eqref{eq:grouplasso} on the support:
    removing \emph{any} group containing a variable removes the
    variable from the support. When variables in groups 1 and 3 are shrunked to zero, the support of the solution consists of the variables of the second group which are neither in the first, nor in the third. (b) \textit{Right}: Latent
    decomposition of $\w$ over $(\v^g)_{v\in\G}$: applying the
    $\lonetwo$ penalty to the decomposition instead of applying
    it to the $\w_g$ removes only the variables which do not belong to
    \emph{any} selected group. The support of the solution if latent vectors $\v_1$ and $\v_3$ are shrunked to zero will be all variables in the second group.}
    \label{fig:removal_decomp}
\end{figure}

When $\G$ is not a partition anymore and some of its groups overlap, the
penalty~\eqref{eq:grouplasso} is still a norm, because we assume that all covariates
belong to at least one group. However, while the
Lasso is sometimes loosely presented as \emph{selecting} covariates
and the group Lasso as \emph{selecting} groups of covariates, the group Lasso
estimator~\eqref{eq:glm} does not necessarily select groups in that case. The reason is that the precise effect of non-differentiable
penalties is to \emph{set} covariates, or groups of covariates, to zero, and not to select them. When
there is no overlap between groups, setting groups to zero leaves
the other full groups to nonzero, which can give the impression that
group Lasso is generally appropriate to select a small number of
groups. When the groups overlap, however, setting one group to zero
shrinks its covariates to zero even if they belong to other groups, in
which case these other groups will not be entirely selected. This is
illustrated in Figure~\ref{fig:removal_decomp}(a) with three
overlapping groups of covariates. If the penalty leads to an estimate
in which the norm of the first and of the third group are zero, what
remains nonzero is not the second group, but the covariates of the
second group which are neither in the first nor in the third one. More formally, the overlapping case has been extensively studied
by \citet{Jenatton2009Structured}, who showed that in the case where $L(\w)$
is an empirical risk and under very general assumptions on the data,
the support of a solution $\hat{\w}$ of~\eqref{eq:glm} almost surely
satisfies
\[
\supp{\hat{\w}} = \left( \bigcup_{g\in\G_0} g \right)^c
\]
for some $\G_0\subset\G$, \ie, the support is almost surely the complement of a union of
groups. Equivalently, the support is an intersection of the
complements of some of groups considered.

In this work, we are interested in penalties which induce a different
effect: we want the estimator to select entire groups of covariate, or more
precisely we want the support of the solution $\hat{\w}$ to be a \emph{union of groups}. For that
purpose, we introduce a set of latent variables $\vb = \br{\v^g}_{g\in\G}$ such that 
$\v^g\in\RR^p$ and $\supp{\v^g}\subset g$ for each group $g\in\G$, and propose
to solve the following problem instead of~\eqref{eq:glm}:
\begin{equation}
\min_{\w\in\RR^p,\vb\in\VV_\G} L(\w) + \lambda \sum_{g\in\G} \, d_g \, \nm{\v^g}\
\qquad \text{s.t.} \qquad \w =\sum_{g\in\G} \v^g.
\label{eq:oglm}
\end{equation}
Problem~\eqref{eq:oglm} is always feasible since we assume
that all covariates belong to at least one group.
Intuitively, the vectors
$\vb=\br{\v_g}_{g\in\G}$ in \eqref{eq:oglm} represent a decomposition
of $\w$ as a sum of latent vectors whose supports are included in each
group, as illustrated in Figure~\ref{fig:removal_decomp}(b). Applying the $\lonetwo$ penalty to these latent vectors favors
solutions which shrink some $\v^g$ to $0$, while the non-shrunk components satisfy $\supp{\v^g} = g$. On the other hand, since
we enforce $\w =\sum_{g\in\G} \v^g$, a $\w_i$ can be nonzero as long as $i$
belongs to at least one non-shrunk group. More precisely, if we denote
by $\G_1\subset\G$ the set of groups $g$ with $\hat{\v}^g\neq \zv$ for the solution of~\eqref{eq:oglm}, then we
immediately get $\hat{\w}=\sum_{g\in\G_1} \hat{\v}^g$, and therefore we can expect:
\[
\supp{\hat{\w}} = \bigcup_{g\in\G_1} g\,.
\]
In other words, this formulation leads to sparse solutions whose
support is likely to be a union of groups.

Interestingly, problem~\eqref{eq:oglm} can be reformulated as the
minimization of the cost function $L(\w)$ penalized by a new regularizer
which is a function of $\w$ only. Indeed since the minimization over $\vb$ only involves the penalty term
and the constraints, we can rewrite~\eqref{eq:oglm} as
\begin{equation}
\min_{\w \in \RR^p} L(\w) + \lambda \Om{\w},
\label{eq:cogl}
\end{equation}
with
\begin{equation}
\label{eq:pendef}
\Om{\w} \eqdef \min_{\vb\in\VG , \sum_{g\in\G}\v^g = \w} \quad \sum_{g\in\G} \, d_g \, \nm{\v^g}\,.
\end{equation}
We call this penalty the \emph{latent group Lasso} penalty, in reference to its formulation as a group Lasso over latent variables. 
When the groups do not overlap and form a partition, there exists a unique decomposition of
$\w\in\RR^p$ as $\w=\sum_{g\in\G} \v^g$ with $\supp{\v^g}\subset g$, namely,
$\v^g=\w_g$ for all $g\in\G$. In that case, both
the group Lasso penalty~\eqref{eq:grouplasso} and the latent group Lasso penalty~\eqref{eq:pendef} are equal and boil down to the same standard group Lasso. When some groups overlap, however, the two penalties differ.
For example, Figure~\ref{fig:balls} shows the unit ball for both norms in $\RR^3$ with
groups $\G=\{\{1,2\},\{2,3\}\}$.  The pillow shaped ball of
$\Ogroup{\cdot}$ has four singularities corresponding to cases where
either only $\w_1$ or only $\w_3$ is nonzero. By contrast, $\Omn$
has two circular sets of singularities corresponding to cases where
$(\w_1,\w_2)$ only or $(\w_2,\w_3)$ only is nonzero. For comparison, we also show the unit ball when we consider the partition $\G=\{\{1,2\},\{3\}\}$, in which case both norms coincide: singularities appear for $(\w_1,\w_2)=0$ or $\w_3=0$.

\begin{figure}
  \begin{center}
      \includegraphics[width=.3\linewidth]{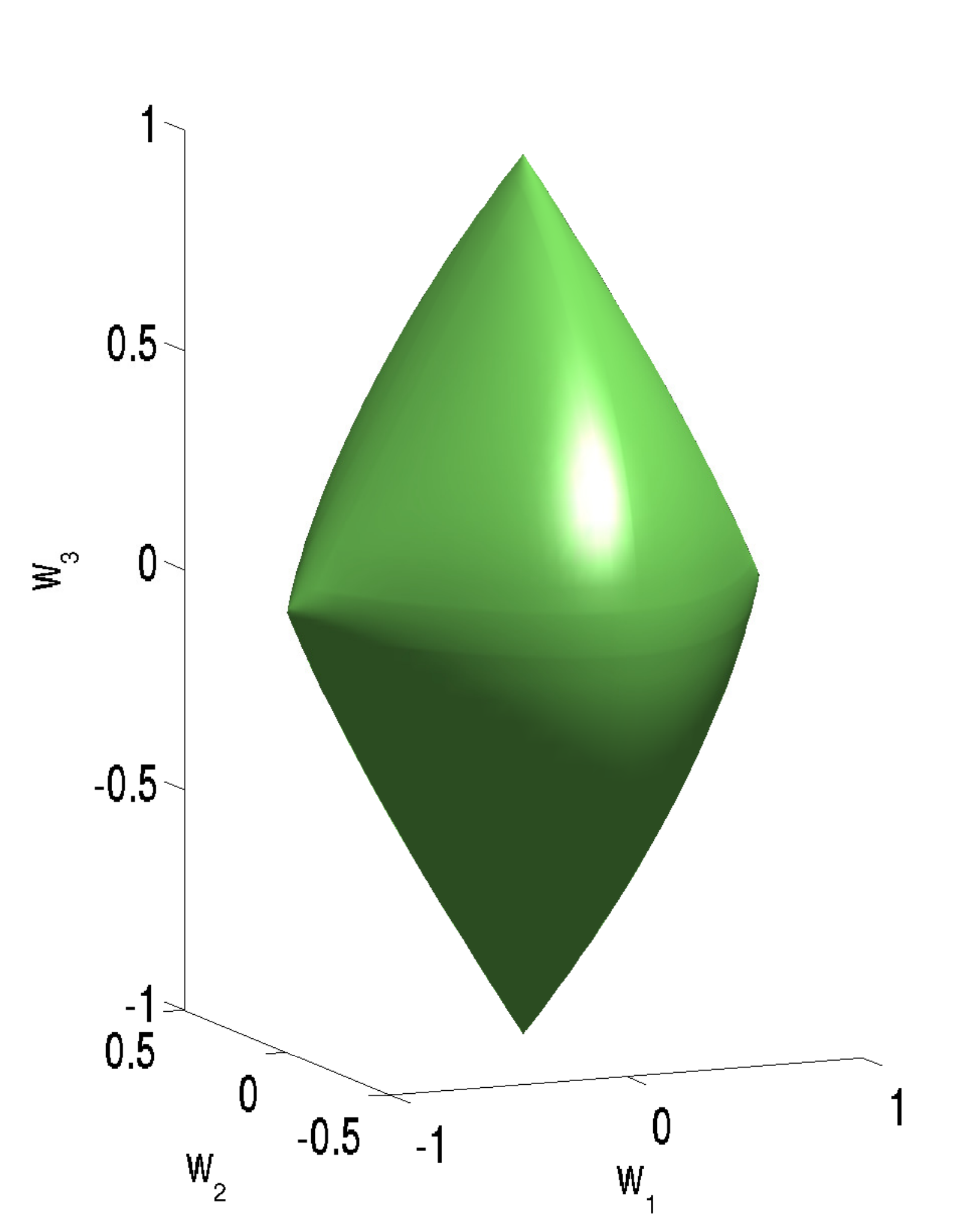}
      \includegraphics[width=.3\linewidth]{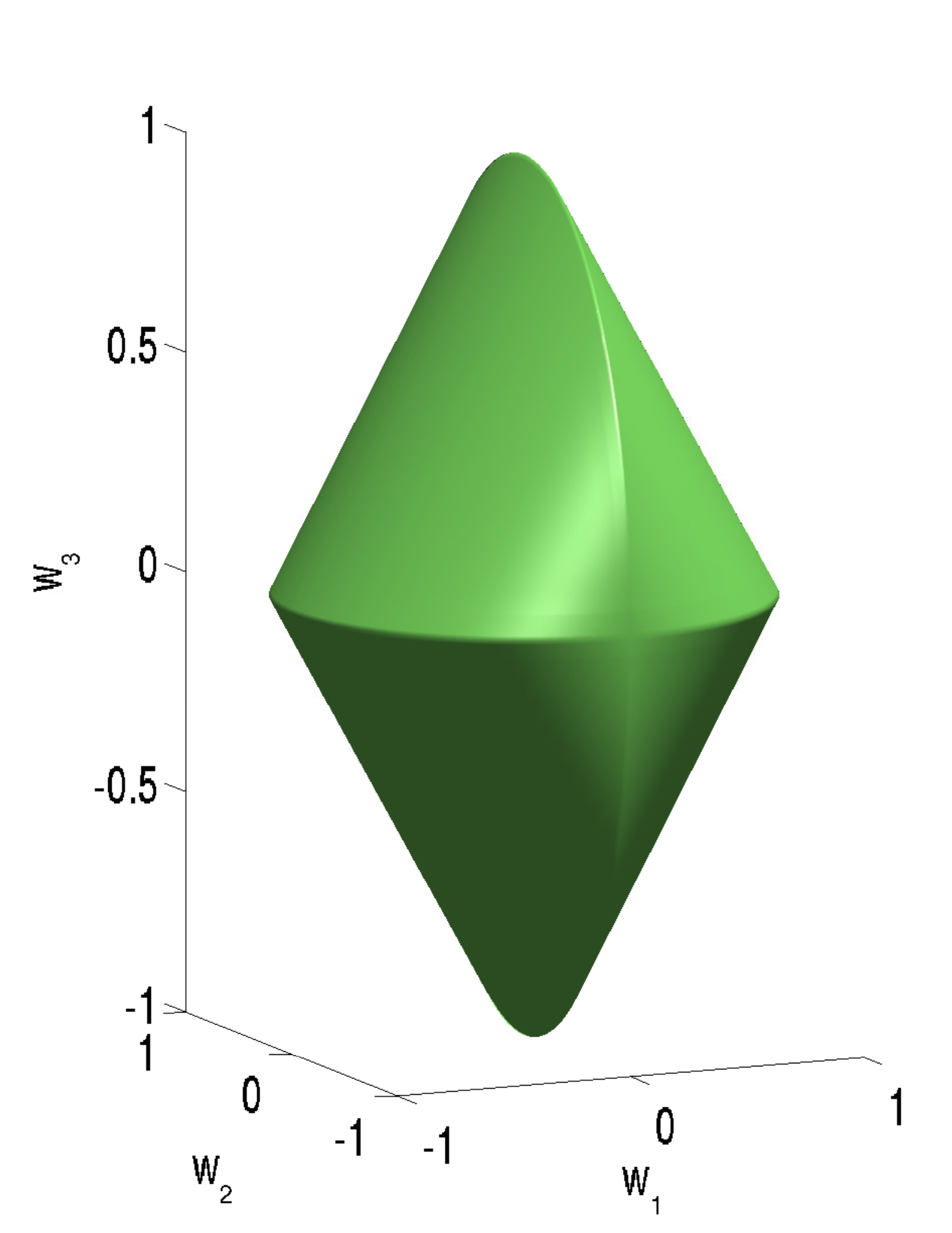}
      \includegraphics[width=.3\linewidth]{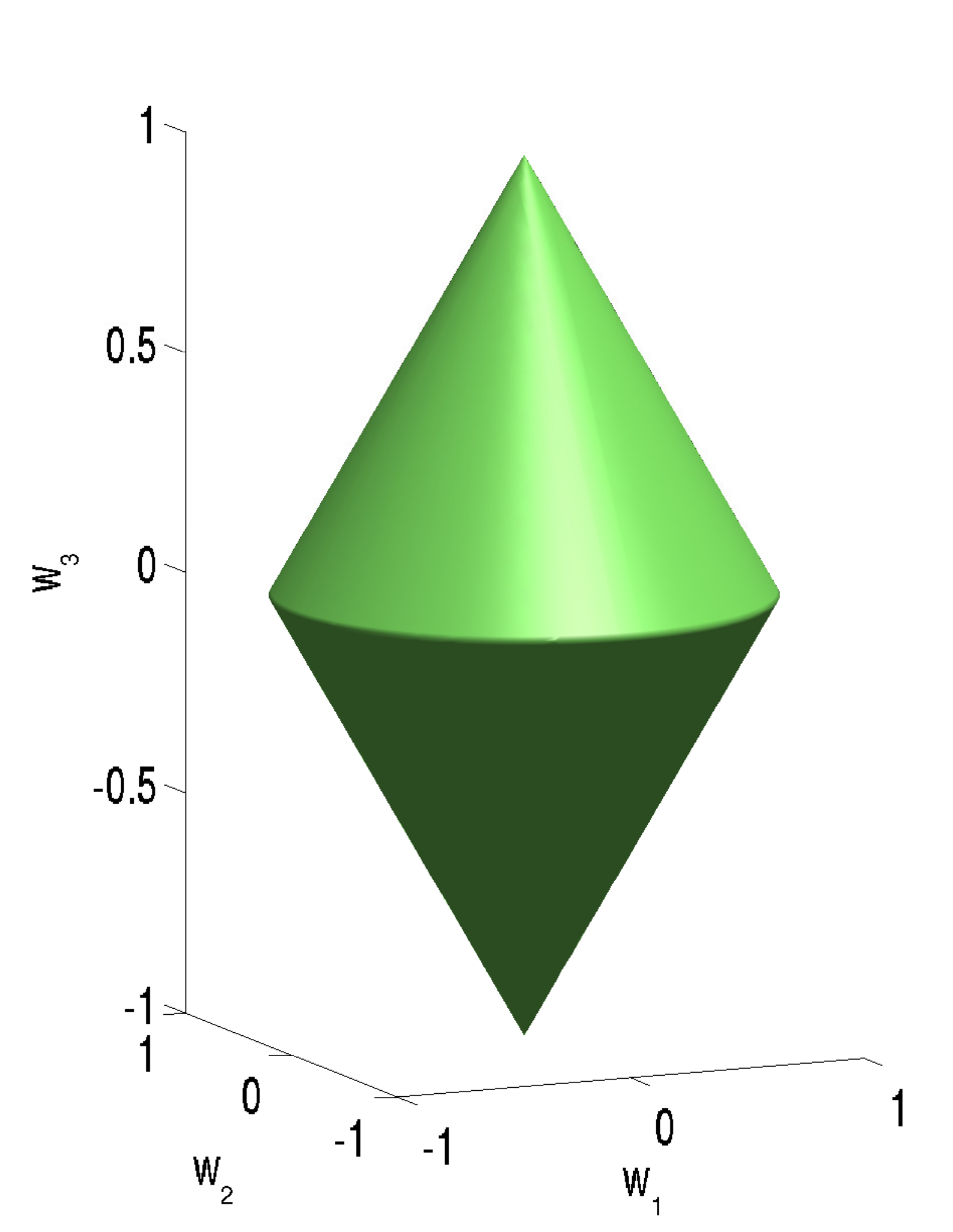}
  \end{center}
  \caption[Unit balls for $\Ogroup{\cdot}$, $\Omn$ and the
  group Lasso.]{Unit balls for $\Ogroup{\cdot}$ (left), proposed by \citet{Jenatton2009Structured}, and
    $\Omn$ (middle), proposed in this paper, for the groups $\G=\{\{1,2\},\{2,3\}\}$.
    $\w_2$ is represented as the vertical coordinate. We note that singularities exist in both cases, but occur at different positions: for $\Ogroup{\cdot}$ they correspond to situations where only $\w_1$ or only $\w_2$ is nonzero, \ie, where all covariates of one group are shrunk to $0$; for $\Omn$, they correspond to situations where only $\w_1$ or only $\w_3$ is equal to $0$, \ie, where all covariates of one group are nonzero. For comparison, we show on the right  the unit ball of both norms for the partition $\G=\{\{1,2\},\{3\}\}$, where they both reduce to the classical group Lasso penalty.}
\label{fig:balls}
\end{figure}

To summarize, we enforce a prior we have on $\w$ by introducing new variables in the optimization
problem~\eqref{eq:oglm}. The constraint we impose is that some groups should be shrunk to zero, and a covariate should have zero weight in $\w$ if all the
groups to which it belongs are set to zero. Equivalently, the support of $\w$ should be a union of groups. This new problem can be re-written as a classical
minimization of the empirical risk, penalized by a particular penalty
$\Omn$ defined in~\eqref{eq:pendef}. This penalty itself associates to each vector $\w$ the
solution of a particular constrained optimization problem. While this
formulation may not be the most intuitive, it allows to reframe the
problem in the classical context of penalized empirical risk
minimization.  In the remaining of this article, we
investigate in more details the latent group Lasso penalty $\Omn$, both theoretically and
empirically.

\subsection{Related work}

The idea of decomposing a parameter vector into some latent components and to regularize each of these components separately has appeared recently in the literature independently of this work.
In particular \cite{Jalali2010Dirty} proposed to consider such a decomposition in the case of multi-task learning, where each task specific parameter vector is decomposed into a first $\ell_1$ regularized vector and another vector, regularized with an $\ell_1/\ell_\infty$ norm; so as to share its sparsity pattern with all other tasks. The norm considered in that work could be interpreted as a special case of the latent group Lasso, where the set of groups consists of all singletons and groups of coefficients associated with the same feature across task.
The decomposition into latent variables is even more natural in the context of the work of 
\cite{Chen2011Robust}, \cite{Candes2009Robust}, or \cite{Agarwal2011Noisy} on robust PCA and matrix decomposition in which a matrix is decomposed in a low rank matrix regularized by the trace norm and a sparse or column-sparse matrix regularized by an $\ell_1$ or group $\ell_1$-norm. 

Another type of decompositions which is related to this norm is the idea of \emph{cover} of the support.
In particular it is interesting to consider the $\ell_0$ counterpart to this norm, which could be written as 
$$
\Omega^{\G}_0=\min_{\Gs \subset \G} \sum_{g \in \Gs} d_g \qquad \st \qquad \w=\sum_{g \in \Gs} \v^g, \: \supp{\v_g} \subset g\,.
$$

$\Omega^{\G}_0$ can then be interpreted as the value of a min set-cover. This penalization has been considered in \citet{Huang2009Learning} under the name \emph{block coding}, since, indeed, when $d_g$ is interpreted as a coding length, this penalization induces a code length on all sets, which can be interpreted in the MDL framework.

More generally, one could consider $\Omega^{\G}_q$ penalties, for all $q\geq 0$, by replacing the $\ell_2$ norm used in the definition of the latent group Lasso penalty~\eqref{eq:pendef} by a $\ell_q$ norm. It should be noted then that, unlike the support, the definition of group-support we introduce in Section~\ref{sec:group-support} changes if one considers the latent group Lasso with a different $\ell_q$-norm, and even if the weights $d_g$ change \footnote{We discuss the choice of weights in detail in Section~\ref{sec:weights}.}.

\citet{Obozinski2011Convex} considers the case of $\Omega^{\G}_q$, when the weights are given by a set function and shows that $\Omega^{\G}_q$ is then the tightest convex ``$\ell_q$ relaxation of the block-coding scheme of \citet{Huang2009Learning}. It also shows that when $\G=2^V$ and the weights are an appropriate power of a submodular function then $\Omega^{\G}_q$ is the norm that naturally extends the norm considered by \citet{Bach2010Structured}. 
 
It should be noted that recent theoretical analyses of the norm
studied in this paper have been proposed by
\citet{Percival2011Theoretical} and
\citet{Maurer2011Structured}. They adopt points of views or focus on
questions that are complementary of this work; we discuss those in
section~\ref{sec:related_theo}.

\section{Some properties of the latent group Lasso penalty}
\label{sec:prop}
In this section we study a few properties of the latent group Lasso $\Omn$, which will be in particular useful to prove consistency results in the next section. After showing that $\Omn$ is a valid norm, we compute its dual norm and provide two variational formulas. We then characterize its unit ball as the convex hull of basic disks, and compute its subdifferential. When used as a penalty for statistical inference, we further reinterpret it in the context of covariate duplication and multiple kernel learning. To lighten notations, in the rest of the paper we simply denote $\Omn$ by $\Oo$.

\subsection{Basic properties}
We first analyze the decomposition induced by~\eqref{eq:pendef} of a vector $\w\in\RR^p$ as
$\sum_{g\in\G} \v^g$. We denote by 
$\Vb(\w)\subset \VV_\G$ the set of $\m$-tuples of vectors $\vb
= \br{\v^g}_{g\in\G}\in\VV_\G$ that are solutions to the optimization problem in \eqref{eq:pendef},
\emph{i.e.}, which satisfy
\[
 \w=\sum_{g\in\G} \v^g \quad\text{and}\quad \Oo(\w) =
 \sum_{g\in\G}d_g\nm{\v^g}\,.
\] 
We first have that:
\begin{lemma}\label{lem:basic1}
For any $\w\in\RR^p$, $\Vb(\w)$ is
non-empty, compact and convex.
\end{lemma}
\begin{proof}
The objective of problem~\eqref{eq:pendef} is a proper closed convex function with no \emph{direction of recession}. Lemma \ref{lem:basic1} is then the consequence of classical results in convex analysis, such as Theorem 27.2 page 265 of \cite{Rockafellar1997Convex}.
\end{proof}
The following statement shows that, unsurprisingly, we can regard $\Oo$ as a classical norm-based penalty.
\begin{lemma} 
$\w \mapsto \Oo(\w)$ is a norm.
\end{lemma}
\begin{proof}
  Positive homogeneity and positive definiteness hold trivially. We
  show the triangular inequality.  Consider $\w,\w' \in \RR^p$, and let
  $\vb\in\Vb(\w)$ and $\vb'\in\Vb(\w')$ be respectively optimal
  decompositions of $\w$ and $\w'$, so that $\Oo(\w)=\sum_gd_g\nm{\v^g}$
  and $\Oo(\w')=\sum_g d_g\nm{{\v '}^g}$ with $\w=\sum_g \v^g$ and $\w' =
  \sum_{g} {\v'}^{g}$. Since $(\v^g+{\v'}^g)_{g \in \G}$ is a (a priori
  non-optimal) decomposition of $\w+\w'$, we clearly have~:
\[
\Oo(\w+\w') \leq \sum_{g \in \G} d_g\nm{\v^g+{\v'}^g} \leq \sum_{g\in\G}
d_g\left(\nm{\v^g}+\nm{{\v'}^g} \right)=\Oo(\w)+\Oo(\w')\,.
\]
\end{proof}

\subsection{Dual norm and variational characterizations}
\label{sec:variationalChar}
$\Oo$ being a norm, by Lemma \ref{lem:basic1}, we can consider its Fenchel dual norm $\Oo^*$ defined by:
\begin{equation}\label{eq:dualnorm}
\forall \alphab\in\RR^p,\quad\Oo^*(\alphab) = \sup_{\w\in\RR^p} \left\{ \w^\top \alphab \,|\, \Oo(\w) \leq 1\right\}\,.
\end{equation}
The following lemma shows that $\Oo^*$ has a simple closed form expression:
\begin{lemma}[dual norm]\label{lem:dualnorm}
The Fenchel dual norm $\Oo^*$ of $\Oo$ satisfies:
$$
\forall \alphab\in\RR^p,\quad \displaystyle \Oo^*(\alphab)=\max_{g \in \G} \: d_g^{-1}\, \nm{\alphab_g} \,.
$$
\end{lemma}
\begin{proof}
We start from the definition of the dual norm \eqref{eq:dualnorm} and compute:
  \begin{align*}
\Oo^*(\alphab) &=  \max_{\w \in \RR^p}& &\w^\top \alphab \quad & &\text{s.t.} \quad \Oo(\w) \leq 1 \\
&= \max_{\w\in\RR^p, \vb \in \VV_\G} & &\w^\top \alphab \quad & &\text{s.t.} \quad \w=\sum_{g \in \G} \v^g, \; \sum_{g\in\G} d_g\nm{\v^g} \leq 1\\
&= \max_{\vb\in\VV_\G} & \sum_{g\in\G} \: &{\v^g}^\top \alphab \quad & &\text{s.t.}\quad \sum_{g\in\G} d_g\nm{\v^g} \leq 1\\
&= \max_{\vb\in\VV_\G, \eta \in \RR_+^{\m}} & \sum_{g\in\G} \: &{\v^g}^\top \alphab \quad & &\text{s.t.}\quad \sum_{g\in\G} \eta_g \leq 1, \; \forall g\in\G, d_g\nm{\v^g} \leq \eta_g\\
&= \max_{\eta \in \RR_+^{\m}} & \sum_{g\in\G} \: &\eta_g \: d_g^{-1}\nm{\alphab_g} \quad & &\text{s.t.} \quad \sum_{g\in\G} \eta_g \leq 1\\
&= \max_{g \in\G} & & d_g^{-1}\nm{\alphab_g}\,.
\end{align*}
The second equality is due to the fact that~:
\[
\{\w \;|\; \Oo(\w) \leq 1 \} = \big \{\w \;|\; \exists \vb \in \VV_\G \quad \text{s.t.}\quad \w=\sum_g \v^g,\; \sum_g
d_g\nm{\v^g} \leq 1 \big \} \,,
\]
and the fifth results from the explicit solution $\v^g = \alphab_g \eta_g d_g^{-1} \nm{\alphab_g}^{-1}$ of the maximization in $\vb$ in the fourth line.
\end{proof}

\begin{remark}
Remembering that the infimal convolution $f \infconv g$ of two convex functions $f$ and $g$ is defined as 
$(f \infconv g)(\w)=\inf_{\v \in \RR^p}  \, \big \{ f(\v)+g(\w-\v) \big \}$ \citep[see][]{Rockafellar1997Convex}, it could be noted that $\Oo$ is the infimal convolution of all functions $\omega_g$ for $g \in \G$ defined as $\omega_g: \w \mapsto \|\w_g\| \iota_g(\w)$ with $\iota_g(\w)=0$ if $\supp{\w} \subset g$ and $+ \infty$ otherwise. One of the main properties motivating the notion of infimal convolution is the fact that it can be defined via $(f \infconv g)^*=f^*+g^*$, where $^*$ denotes Fenchel-Legendre conjugation. Several of the properties of $\Oo$ can be derived from this interpretation but we will however show them directly.
\end{remark}

The norm $\Oo$ was initially defined as the solution of an optimization problem in \eqref{eq:pendef}. From the characterization of $\Oo^*$ we can easily derive a second variational formulation:
\begin{lemma}[second variational formulation]\label{lem:variational1}
For any $\w\in\RR^p$, we have
\begin{equation}\label{eq:variational1}
\Oo(\w)=\max_{\alphab \in \RR^p} \: \alphab^\top \w \quad \text{s.t.} \quad \nm{\alphab_g} \leq d_g \quad\text{for all }g\in\G\,.
\end{equation}
\end{lemma}
\begin{proof}
Since the bi-dual of a norm is the norm itself, we have the variational form
\begin{equation}
\label{dualvar}
\Oo(\w)=\max_{\alphab \in \RR^p} \: \alphab^\top \w \quad \text{s.t.} \quad \Oo^*(\alphab) \leq 1\,.
\end{equation}
Plugging the characterization of $\Oo^*$ of Lemma \ref{lem:dualnorm} into this equation finishes the proof.
\end{proof}
For any $\w\in\RR^p$, we denote by $\A(\w)$ the set of $\alphab\in\RR^p$ in the dual unit sphere which solve the second variational formulation \eqref{eq:variational1} of $\Oo$, namely:
\begin{equation}\label{eq:defA}
\A(\w) \eqdef
\underset{\alphab \in \RR^p, \; \Oo^*(\alphab) \leq 1}{\text{argmax}}\:
\alphab^\top \w\,.
\end{equation}
With a few more efforts, we can also derive a third variational representation of the norm $\Oo$, which will be useful in Section \ref{sec:consistency} in the proofs of consistency:
\begin{lemma}[third variational formulation]\label{lem:trick}
For any $\w\in\RR^p$, we also have
\begin{equation}
\label{eq:trick}
\Oo(\w)= \frac{1}{2} \min_{\lambdab \in \RR^\m_+} \:\sum_{i=1}^p \:\frac{\w_i^2}{\sum_{g \ni i} \lambdag} + \sum_{g \in \G} \: d_g^2 \, \lambdag\,.
\end{equation}
\end{lemma}
\begin{proof}
For any $\w\in\RR^p$, we can rewrite the solution of the constrained optimization problem of the second variational formulation \eqref{eq:variational1} as the saddle point of the Lagrangian:
$$
\Oo(\w) = \min_{\lambdab \in \RR^{\m}_+}\max_{\alphab \in \RR^p} \: \w^\top \alphab - \frac{1}{2}\sum_{g \in \G} \lambdag \big ( \nm{\alphab_g}^2-d_g^2\big )\,.
$$
Optimizing in $\alphab$ leads to $\alphab$ being solution of  $\w_i=\alphab_i \sum_{g \ni i} \lambdag$, which (distinguishing the cases $\w_i=0$ and $\w_i \neq 0$) yields problem (\ref{eq:trick}) when replacing $\alphab_i$ by it optimal value.
\end{proof}
Let us denote by $\Lambdab(\w) \subset \RR^m_+$ the set of solutions to the third variational formulation \eqref{eq:trick}. Note that there is not necessarily a unique solution to \eqref{eq:trick}, because the Hessian of the objective function is not always positive definite (see lemma \ref{lem:uniqueness} in Appendix \ref{app:uniqueness} for a characterization of cases in which positive definiteness can be guaranteed). For any $\w\in\RR^p$, we now have three variational formulations for $\Oo(\w)$, namely \eqref{eq:pendef}, \eqref{eq:variational1} and \eqref{eq:trick}, with respective solutions sets $\Vb(\w)$, $\A(\w)$ and $\Lambdab(\w)$. The following lemma shows that $\Vb(\w)$ is in bijection with $\Lambdab(\w)$. 
\begin{lemma}\label{lem:bijection}
Let $\w\in\RR^p$. The mapping 
\begin{equation}\label{eq:lambdamapping}
\begin{split}
\lambdab : \,&\VV_\G  \rightarrow \RR^m \\
& \vb \mapsto   \lambdab(\vb) = \br{d_g^{-1}\nm{\v^g}}_{g\in\G}
\end{split}
\end{equation}
is a bijection from $\Vb(\w)$ to $\Lambdab(\w)$. For any $\lambdab\in\Lambdab(\w)$, the only vector $\vb\in\Vb(\w)$ that satisfies $\lambdab(\vb) = \lambdab$ is given by $\v^g_g=\lambdag\alphab_g$, where $\alphab$ is any vector of $\A(\w)$.
\end{lemma}
\begin{proof}
To express the penalty as a minimization problem, let us use the following basic equality valid for any $x\in\RR_+$:
$$
x = \frac{1}{2} \min_{\eta \geq 0} \sqb{\frac{x^2}{\eta} + \eta}\,,
 $$
where the unique minimum in $\eta$ is reached for $\eta=x$. From this we deduce that, for any $\v\in\RR^p$ and $d>0$~:
$$
\displaystyle d \nm{\v}= \frac{1}{2}\min_{\eta \geq 0} d \sqb{\frac{\nm{\v}^2}{\eta} + \eta }=\frac{1}{2} \min_{\lambda' \geq 0} \sqb{\frac{\nm{\v}^2}{\lambda'} + d^2 \, \lambda' }\,,
$$
where the unique minimum in the last term is attained for $\lambda'=d^{-1} \nm{\v}$. Using definition \eqref{eq:pendef} 
we can therefore write $\Oo(\w)$ as the optimum value of a jointly convex optimization problem in $\vb\in\VV_\G$ and $\lambdab' = \br{\lambdapg}_{g\in\G}\in\RR_+^m$~:
\begin{equation}\label{eq:kenta}
\Oo(\w)  = \min_{\vb \in \VV_\G,\: \sum_{g \in \G} \v^g=\w, \: \lambdab' \in \RR^{\m}_+} \quad \frac{1}{2} \sum_{g \in \G}   \Big [ \frac{\nm{\v^g}^2}{\lambdapg} + d_g ^2 \, \lambdapg \Big ]\,,
\end{equation}
where for any $\vb$, the minimum in $\lambdab'$ is uniquely attained for $\lambdab' = \lambdab(\vb)$ defined in~\eqref{eq:lambdamapping}. By definition of $\Vb(\w)$, the set of solutions of \eqref{eq:kenta} is therefore exactly the set of pairs of the form $\br{\vb,\lambdab(\vb)}$ for $\vb\in\Vb(\w)$. Let us now isolate the minimization over $\vb$ in \eqref{eq:kenta}. To incorporate the constraint $\sum_{g \in \G} \v^g=\w$ we rewrite \eqref{eq:kenta} with a Lagrangian:
$$
\Oo(\w) =  \min_{\lambdab' \in \RR^{\m}_+} \max_{\alphab' \in \RR^p} \min_{\vb \in \VV_\G} \quad \frac{1}{2} \sum_{g \in \G}   \Big [ \frac{\nm{\v^g}^2}{\lambdapg} + d_g ^2 \, \lambdapg \Big ] + {\alphab'}^\top (\w -\sum_{g \in \G} \v^g)\,.
$$
The inner minimization in $\vb$, for fixed $\lambdab'$ and $\alphab'$, yields $\v^g_i=\lambdapg\alphab'_i$. The constraint $\w =\sum_{g \in \G} \v^g$ therefore implies that, after optimization in $\vb$ and $\alphab'$, we have $\alphab'_i=\frac{w_i}{\sum_{g \ni i} \lambdapg}$, and as a consequence that $\v^g_i=\frac{\lambdapg}{\sum_{h \ni i } \lambdab'_h} \, \w_i$. A small computation now shows that, after optimization in $\vb$ and $\alphab'$ for a fixed $\lambdab'$, we have:
$$
\sum_{g \in \G} \frac{\nm{\v^g}^2}{\lambdapg}=\sum_{i=1}^p \sum_{g \ni i} \frac{\big ( \v^g_i \big )^2}{\lambdapg}=\sum_{i=1}^p \sum_{g \ni i} \frac{\lambdapg\, \w_i^2}{\Big(\sum_{h \ni i } \lambdab'_h \Big)^2}=\sum_{i=1}^p  \frac{\w_i^2}{\sum_{h \ni i } \lambdab'_h}\,.
$$
Plugging this into \eqref{eq:kenta}, we see that after optimization in $\vb$, the optimization problem in $\lambdab'$ is exactly \eqref{eq:trick}, which by definition admits $\Lambdab(\w)$ as solutions, while we showed that \eqref{eq:kenta} admits $\lambdab\br{\Vb(\w)}$ as solutions. This shows that $\lambdab\br{\Vb(\w)} = \Lambdab(\w)$, and since for any $\lambdab'\in\Lambdab(\w)$ there exists a unique $\vb\in\Vb(\w)$ that satisfies $\lambdab(\vb)=\lambdab'$, namely, $\v^g_i=\frac{\lambdapg}{\sum_{h \ni i } \lambdab'_h} \, \w_i$, $\lambdab$ is indeed a bijection from $\Vb(\w)$ to $\Lambdab(\w)$. Finally, we noted in the proof of Lemma \ref{lem:trick} that for any $\lambdab\in\Lambdab(\w)$ and $\alphab\in\A(\w)$, $\w_i = \alphab_i \sum_{h \ni i } \lambdab_h$. This shows that the unique $\vb\in\Vb(\w)$ associated to a $\lambdab\in\Lambdab(\w)$ can equivalently be written $\v_g^g = \lambdab_g \alphab_g$, which concludes the proof of Lemma \ref{lem:bijection}.
\end{proof}

\subsection{Characterization of the unit ball of $\Oo$ as a convex hull}
Figure \ref{fig:balls}(b) suggests visually that the unit ball of $\Oo$ is just the convex hull of a horizontal disk and a vertical one. This impression is correct and formalized more generally in the following lemma.
\begin{lemma} 
For any group $g\in\G$, define the hyperdisks $\mathcal{D}_g=\{\w \in \RR^p \mid \nm{\w_g} \leq d_g^{-1}, \: \w_{g^c}=\zv\}$. Then, the unit ball of $\Omega$ is the convex hull of the union of hyper-disks $\cup_{g \in \G} \: \mathcal{D}_g$.
\end{lemma}
\begin{proof}
Let $\w \in \text{ConvHull}\big ( \cup_{g \in \G} \: \mathcal{D}_g \big )$, then there exist  $\alphab^g \in \mathcal{D}_g$ and $t_g \in \RR_+$, for all $g \in \G$, such that $\sum_{g \in \G} t_g \leq 1$ and $\w = \sum_{g\in\G}t_g \alphab^g$. Letting $\vb=(t_g \, \alphab^g)_{g \in \G}$ as a suboptimal decomposition of $\w$, we easily get
$$
\Oo(\w) \leq \sum_{g\in\G} d_g \nm{t_g \, \alphab^g} \leq  \sum_{g\in\G} t_g \leq 1\,.
$$
Conversely, if $\Oo(\w) \leq 1$, then there exists $\vb \in \VV_\G$, such that $\sum_{g \in \G} d_g \nm{\v^g} \leq 1$ and we obtain $\alphab^g \in \mathcal{D}_g$ and $\mathbf{t}$ in the simplex by letting $t_g=d_g \nm{\v^g}$ and $$\alphab^g=
\begin{cases}
\zv  &\text{if} \quad t_g=0\,,\\
\frac{\v^g}{d_g \, \nm{\v^g}} & \text{else}\,.
\end{cases}$$
\end{proof}

It should be noted that this lemma shows that $\Oo$ is the gauge of the convex hull of
the disks $\mathcal{D}_g$, in other words, $\Oo$ is, in the terminology introduced by \citet{Chandrasekaran2010Convex}, the unit ball of the \emph{atomic norm} associated with the union of disks $\mathcal{D}_g$. 

\subsection{Subdifferential of $\Oo$}
\label{sec:subdiff}
The subdifferential of $\Oo$ at $\w$ is, by definition:
\[
\partial \Oo(\w)\eqdef \{\s \in \RR^p \:|\: \forall \h \in \RR^p , \: \Oo(\w+\h)-\Oo(\w) \geq \s^\top \h\}\,.
\]
It is a standard result of convex optimization \citep[resulting e.g. from characterization (${\rm b}^*$) of the subdifferential in Theorem 23.5, p. 218,][]{Rockafellar1997Convex} that for all $\w \in \RR^p, \: \partial \Oo(\w)=\A(\w)$, where $\A(\w)$ was defined in \eqref{eq:defA}. 

We can now show a simple relationship between the decomposition $(\v^g)_{g \in \G}$ of a vector $\w$ induced by $\Omega$, and the subdifferential of $\Omega$.
\begin{lemma}
\label{lem:alpha}
For any $\alphab \in \A(\w)=\partial \Oo(\w)$ and any $\vb \in\Vb(\w)$,
\[\begin{cases} \text{either} \; & \v^g \neq \zv \quad \text{and} \quad \alphab_g= d_g \frac{\v^g}{\nm{\v^g}}\,,\\
\text{or}\;  & \v^g=\zv \quad  \text{and}  \quad \nm{\alphab_g} \leq d_g\,.
\end{cases}
\]
\end{lemma}
\begin{proof}
  Let $\vb\in\Vb(\w)$ and
$\alphab\in\A(\w)$. Since $\Oo^*(\alphab) \leq 1$, we have $\nm{\alphab_g} \leq
d_g$ which implies $\alphab^\top \v^g \leq d_g\nm{\v^g}$.
On the other hand, we also have $\alphab^\top \w=\Oo(\w)$ so
that $0=\Oo(\w)-\alphab^\top \w= \sum_g \left ( d_g\nm{\v^g}-\alphab_g^\top \v^g
\right ),$ which is a sum of non-negative terms. We conclude that, for
all $g \in \G$, we have $\alphab_g^\top \v^g=d_g\nm{\v^g}$ which yields the
result.
\end{proof}
We can deduce a general property of all decompositions of given vector:
\begin{corollary}
  Let $\w\in\RR^p$. For all $\vb,\:\vb' \in\Vb(\w)$, and for all $g \in
  \G$ we have $\v^g=\zv$ or ${\v'}^g=\zv$ or there exists $\gamma \in \RR$
  such that $\v^g=\gamma {\v'}^g$.
\end{corollary}
\begin{proof}
  By Lemma \ref{lem:alpha}, if $\v^g \neq \zv$ and ${\v'}^g \neq \zv$, then
  $\alphab_g=d_g\frac{\v^g}{\nm{\v^g}}=d_g\frac{{\v'}^g}{\nm{{\v'}^g}}$ so that
  $\v^g=\frac{\nm{\v^g}}{\nm{{\v'}^g}} {\v'}^g$.
\end{proof}

\subsection{$\Oo$ as a regularizer}\label{sec:learning}
\input{learning}

\subsection{Covariate duplication}
\label{sec:duplication}
In this section we show that 
empirical risk minimization penalized by $\Oo$ is equivalent
to a regular group Lasso in a covariate space of higher dimension
obtained by duplication of the covariates belonging to several groups. 
This has implications for practical implementation of
$\Oo$ as a regularizer and for its generalization to non-linear
classification.

More precisely, let us consider the duplication operator:
\begin{equation}
\begin{split}
\RR^p &\rightarrow \RR^{\sum_{g\in\G}|g|} \\
\x & \mapsto \tilde{\x} =  \bigoplus_{g\in\G} (\x_i)_{i\in g}\,.
\end{split}
\end{equation}
In other words, $\tilde{\x}$ is obtained by stacking the restrictions of $\x$ to each group on top of each other, resulting in a $\br{\sum_{g\in\G}|g|}$-dimensional vector. Note that any coordinate of $\x$ that occurs in several groups will be duplicated as many times in $\tilde{\x}$. Similarly, for a vector $\v\in\VV_\G$, let us denote by $\tilde{\v}$ the $\br{\sum_{g\in\G}|g|}$-dimensional vector obtained by stacking the restrictions of the successive $\v^g$ on their corresponding groups $g$ on top of each other (resulting in no loss of information, since $\v^g$ is null outside of $g$). This operation is illustrated in \eqref{eq:ogdup} below. Then for any $\w\in\RR^p$ and $\v\in\VV_\G$ such that $\w=\sum_{g\in\G}\v^g$, we easily get, for any $\x\in\RR^p$:
\begin{equation}\label{eq:dupliinner}
\w^\top \x = \sum_{g\in\G} \v^{g \top} \x = \tilde{\v}^\top \tilde{\x}\,.
\end{equation}
Consider now a learning problem with training points $\x^{1},\ldots,\x^{n}\in\RR^p$ where we minimize over $\w\in\RR^p$ a penalized risk function that depends of $\w$ only through inner products with the training points, \ie, or the form
\begin{equation}\label{eq:optdupli}
\min_{\w\in\RR^p} \tilde{L}(\X\w) + \lambda \Oo(\w)\,,
\end{equation}
where $\X$ is the $n\times p$ matrix of training points and $\X\w$ is therefore the vector of inner products of $\w$ with the training points. Many problems, in particular those considered in \secref{sec:learning}, have this form. By definition of $\Oo$ we can rewrite~\eqref{eq:optdupli} as
$$
\min_{\w\in\RR^p , \v\in\VV_\G , \sum_{g} \v^g = \w} \tilde{L}(\X\w) + \lambda \sum_{g\in\G} d_g\, \nm{\v^g}\,,
$$
which by~\eqref{eq:dupliinner} is equivalent to
\begin{equation}
\label{eq:optdupli2}
\min_{\tilde{v} \in \RR^{\sum_{g\in\G}|g|}} \tilde{L}(\tilde{\X}\tilde{\v}) + \lambda \sum_{g\in\G} d_g \, \nm{\tilde{\v}_g}\,,
\end{equation}
where $\tilde{\X}$ is the $n\times(\sum_{g\in\G}|g|)$ matrix of duplicated training points, and $\tilde{\v}_g$ refers to the restriction of $\tilde{\v}$ to the coordinates of group $g$. In other words, we have eliminated $\w$ from the optimization problem and reformulated it as a simple group Lasso problem without overlap between groups in an expanded space of size $\sum_{g\in\G}|g|$.

On the example of Figure~\ref{fig:removal_decomp},with $3$ overlapping groups, this duplication trick can be rewritten as follows~:
\begin{equation}
  \X\w =\X.
  {\begin{array}{c}
      \colorcellfive{black}{lightblue}{white}{\tilde{\v}^1}\\
      \colorcellsix{black}{lightgray}{black}{0}
    \end{array}}+\X.
  {\begin{array}{c}
      \colorcellthree{black}{lightgray}{black}{0}\\
      \colorcellfive{black}{lightblue}{white}{\tilde{\v}^2}\\
      \colorcellthree{black}{lightgray}{black}{0}
    \end{array}}
  +\X.{\begin{array}{c}
      \colorcellsix{black}{lightgray}{black}{0}\\
      \colorcellfive{black}{lightblue}{white}{\tilde{\v}^3}
    \end{array}}
  = \left(\X_{g_1}, \X_{g_2}, \X_{g_3}\right).
  {\begin{array}{c}
      \colorcellfive{black}{lightblue}{white}{\tilde{\v}_1}\\
      \colorcellfive{black}{lightblue}{white}{\tilde{\v}_2}\\
      \colorcellfive{black}{lightblue}{white}{\tilde{\v}_3}\\
    \end{array}} \stackrel{\Delta}{=} \tilde{\X}\tilde{\v}.
\label{eq:ogdup}
\end{equation}

This formulation as a classical group Lasso problem in an expanded space has several implications, detailed in the next two sections. On the one hand, it allows to extend the penalty to non-linear functions by considering infinite-dimensional duplicated spaces endowed with positive definite kernels (\secref{sec:mkl}). On the other hand, it leads to straightforward implementations by borrowing classical group Lasso implementations after feature duplications (\secref{sec:implementation}). Note, however, that the theoretical results we will show in \secref{sec:consistency}, on the consistency of the estimator proposed, are not mere consequences of existing results for the classical group Lasso, because, in the case we consider, not only is the design matrix $\tilde{\X}$ rank deficient, but so are all of its restriction to sets of variables corresponding to any union of overlapping groups.

\subsection{Multiple Kernel Learning formulations}\label{sec:mkl}

Given the reformulation in a duplicated variable space presented above, we provide in this section a multiple kernel learning (MKL) interpretation to the regularization by our norm and show that it extends naturally the case with disjoint groups.

To introduce it, we return first to the concept of MKL \citep{Lanckriet2004Learning,Bach2004Multiple}  which we can present as follows. If one considers a learning problem of the form
\begin{equation}
H=\min_{\w \in \RR^p} \tilde{L}(\X\w)+\frac{\lambda}{2} \|\w\|^2,
\end{equation}
then by the representer theorem the optimal value of the objective $H$ only depends on the input
data $\X$ through the Gram matrix $\mathbf{K}=\X \X^\top$, which
therefore can be replaced by any positive definite (p.d.) kernel between the
datapoints.  Moreover $H$ can be shown to be a convex function of
$\mathbf{K}$ \citep{Lanckriet2004Learning}. Given a collection of p.d.
kernels $\mathbf{K}_1, \ldots, \mathbf{K}_k$, any convex combination
$\mathbf{K}=\sum_{i=1}^k \eta_i \mathbf{K}_i$ with $\eta_i\geq 0$ and
$\sum_i \eta_i=1$ is itself a p.d.\ kernel. The multiple kernel
learning problem consists in finding the best such combination in the
sense of minimizing $H$:
\begin{equation}
\min_{\etab \in \RR^k_+} H\big ( {\textstyle \sum_i \eta_i \mathbf{K}_i }\big) \quad \st \quad \sum_{i} \eta_i=1. 
\end{equation}
The kernels considered in the linear combination above are typically reproducing kernels associated with different reproducing kernel Hilbert spaces (RKHS).

\citet{Bach2004Multiple} showed that problems regularized by a squared $\ell_1/\ell_2$-norm and multiple kernel learning were intrinsically related.
More precisely he shows that, if $\G$ forms a partition of
$\{1,\ldots,p\}$, letting problems $(P)$ and $(P')$ be defined through
\begin{equation*}
(P) \: \min_{\w \in \RR^p} \tilde{L}(\X \w)+\frac{\lambda}{2} \big ({\textstyle  \sum_{g \in \G } d_g \|\w_g\|} \big )^2 \quad \text{and} \quad (P') \:\min_{\etab \in \RR^{\m}_+} H \big ({\textstyle \sum_{g \in \G} \eta_g \mathbf{K}_g }\big) \: \st \: \sum_{g \in \G} d_g^2 \eta_g =1,
\end{equation*}
with $\mathbf{K}_g= \X_g \X_g^\top$, then $(P)$ and $(P')$ are equivalent in the sense that the optimal values of both objectives are equal with a bijection between the optimal solutions.
Note that such an equivalence does not hold if the groups $g \in \G$ overlap.

Now turning to the norm we introduced, using the same derivation as the one leading from problem (\ref{eq:optdupli}) to problem (\ref{eq:optdupli2}), we can show that minimizing $\tilde{L}(\X \w)+\frac{\lambda}{2} \Omega(\w)^2$ w.r.t. $\w$ is equivalent to minimizing $\tilde{L}(\tilde{\X} \tilde{\v})+\frac{\lambda}{2} \big( \sum_g \|\v^g\| \big)^2$ and setting $\w=\sum_{g \in \G} \v^g$. At this point, the result of \citet{Bach2004Multiple} applied to the latter formulation in the space of duplicates shows that it is equivalent to the multiple kernel learning problem
\begin{equation}
\label{eq:flat_mkl}
\min_{\etab \in \RR^{\m}_+} H \big ( {\textstyle \sum_{g \in \G} \eta_g \mathbf{K}_g}\big ) \quad \st \quad \sum_{g \in \G} d_g^2 \eta_g=1, \quad 
\text{with} \quad \mathbf{K}_g= \X_g \X_g^\top.
\end{equation}

This shows that minimizing $\tilde{L}(\X \w)+\frac{\lambda}{2} \Omega(\w)^2$ is equivalent to the MKL problem above. Compared with the original result of \citet{Bach2004Multiple}, it should be noted now that, because of the duplication mechanism implicit in our norm, the original sets $g \in \G$ are no longer required to be disjoint. In fact this derivation shows that, in some sense, the norm we introduced is the one that corresponds to the most natural extension of multiple kernel learning to the case of overlapping groups.

Conversely, it should be noted that, while one of the application of multiple kernel learning is \emph{data fusion} and thus allows to combine kernels corresponding to functions of intrinsically different input variables, MKL can also be used to select and combine elements from different function spaces defined on the same input.
In general these function spaces are not orthogonal and are typically not even disjoint. In that case the MKL formulation corresponds implicitly to using the norm presented in this paper.  

Finally, another MKL formulation corresponding to the norm is possible.
If we denote $\mathbf{K}_i=\X_i \X_i^\top$ the rank one kernel corresponding to the $i$th feature, then we can write $\K_g=\sum_{i \in g} \K_i$. If $\B \in \RR^{p \times m}$ is the binary matrix defined by $\B_{ig}= \mathbf{1}_{\{ i \in g\}}$, and $Z=\{\B\etab \mid \etab \in \RR^{\m}_+, \, \sum_{g \in \G} \eta_g=1\}$ is the image of the canonical simplex of $\RR^m$ by the linear transformation associated with $\B$, then with $\zetab \in Z$ obtained through $\zeta_i=\sum_{g \ni i} \eta_g$, the MKL problem above can be reformulated as

\begin{equation}
\label{eq:struct_mkl}
\min_{ \zetab \in Z} \: H \big (\sum_{i=1}^p \zeta_i \mathbf{K}_i \big ).
\end{equation}

This last formulation can be viewed as the structured MKL formulation associated with the norm $\Omega$ \citep[see][sec.~1.5.4]{Bach2011Optimization}. It is clearly more interesting computationally when
$\m\gg p$. It is however restricted to a particular form of kernel
$\mathbf{K}_g$ for each group, which has to be a sum of feature kernels
$\mathbf{K}_i$. In particular, it doesn't allow for interactions among features
in the group.

In the two formulations above, it is obviously possible to replace the
linear kernel used for the derivation by a non-linear kernel. In the
case of \eqref{eq:flat_mkl} the combinatorial structure of the problem
is a priori lost in the sense that the different kernels are no longer
linear combinations of a set of ``primary" kernels, while this is still the
case for \eqref{eq:struct_mkl}.

Using non-linear kernels like RBF, or kernels on discrete structures such as sequence- or graph-kernels may prove useful in cases where the relationship between the covariates in the groups and the output is expected to be non-linear. For example if $g$ is a group of
genes and the coexpression patterns of genes within the group are
associated with the output, the group will be deemed important by a
non linear kernel while a linear one may miss it. More generally, it
allows for structured non-linear feature selection.

%% file: learning.tex
We consider in this section the situation where $\Oo$ is used as a regularizer for an empirical risk minimization problem.
Specifically, let us consider a convex differentiable loss function $\ell: \RR \times \RR \rightarrow \RR$, such as the squared error $\ell(t,y) = (t-y)^2$ for regression problems or the logistic loss $\ell(t,y) = \log(1+e^{-yt})$ for classification problems where $y=\pm1$.
Given a set of $n$ training pairs $(\x^{(i)},y^{(i)}) \in \RR^p \times \RR$, $i=1,\dots,n$, we define the empirical risk $L(\w)=\frac{1}{n} \sum_{i=1}^n \ell(\w^\top \x^{(i)},y^{(i)})$ and consider the regularized empirical risk minimization problem  
\begin{equation}
\min_{\w \in \RR^p} L(\w) + \lambda \Oo(\w)\,.
\label{eq:opt1}
\end{equation}
Its solutions are characterized by optimality conditions from subgradient calculus:

\begin{lemma}\label{lem:equivgrad}
 A vector $\w\in\RR^p$ is a solution of~\eqref{eq:opt1} if and only if
one of the following equivalent conditions is satisfied 
\begin{enumerate}
\item[(a)]  $-\nabla L(\w) / \lambda \in\A(\w)$
\item[(b)]  $\w$ can be decomposed as $ \w=\sum_{g\in\G} \v^g$ for some $\bar{\v} \in \VV_\G$ with for all $g \in \G$:
$$\text{either} \quad \v_g \neq \zv \:\:\: \text{and}\:\:\: \nabla_{\!g} L(\w) = -
  \lambda d_g\v^g / \nm{\v^g}
  \quad \text{or} \quad
  \v^g=\zv \:\:\: \text{and}\:\:\:
  d_g^{-1}\nm{\nabla_{\!g} L(\w) }\leq \lambda\,.
  $$
\end{enumerate}
\end{lemma}
\begin{proof}
(a) is immediate from subgradient calculus and the fact that $\partial \Oo(\w)=\A(\w)$ (see \secref{sec:subdiff}). (b) is immediate from Lemma~\ref{lem:alpha}.
\end{proof}

%% file: algos.tex
\label{sec:implementation}

There are several possible algorithmic approaches to solve the
optimization problem (\ref{eq:opt1}), depending on the structure of
the groups in $\G$.  The approach we chose in this paper is based on
the reformulation by \emph{covariate duplication} of
section~\ref{sec:duplication}, and applies an algorithm for the group
Lasso in the space of duplicates.  To be specific, for the experiments
presented in section \ref{sec:experiments}, we implemented the
block-coordinate descent algorithm of~\citet{Meier2008The} combined
with the working set strategy proposed by~\cite{Roth2008The}. Note
that the covariate duplication of the input matrix $X$ needs not to be
done explicitly in computer memory, since only fast access to the
corresponding entries in $X$ is required. Only the vector $\tilde{v}$
which is optimized has to be stored in the duplicated space
$\RR^{\sum_{g\in\G}|g|}$ and is potentially of large dimension
(although sparse) if $\G$ has many groups.

Alternatively, efficient algorithms which do not require working in
the space of duplicated covariates are possible. Such an algorithm was
proposed by \citet{Mosci2010Primal} who suggested to use a proximal
algorithm, and to compute the proximal operator of the norm $\Oo$ via
an approximate projection on the unit ball of the dual norm in the
input space. To avoid duplication, it would also be possible to use an
approach similar to that
of~\citep{Rakotomamonjy2008SimpleMKL}. Finally, one could also
consider algorithms from the multiple kernel learning literature.

%% file: group-support.tex
 \label{sec:group-support}
A natural question associated with the norm $\Oo$ is what sparsity pattern are elicited when the norm is used as a regularizer. This question is natural in the context of support recovery.
If the groups are disjoint, one could equivalently ask which patterns of selected group are possible, since answering the latter or the former questions are equivalent.
This suggest a view in which the support is expressed in terms of groups. We formalize this idea through the concept of 
group-support of a vector $\w$, which, put informally, is the set of groups that are non-zero in a decomposition of $\w$.
We will see that this notion is useful to characterize induced decompositions and recovery properties of the norm.

\subsection{Definitions}
More formally, we naturally call \emph{group-support} of a decomposition $\vb\in\VV_\G$, the set of groups $g$ such that $\v^g \neq \zv$. We extend this definition to a vector as follows:
\begin{definition}[\bf Strong group-support]
The {\sl strong group-support} $\Gs(\w)$ of a vector $\w\in\RR^p$ is the union of the group-supports of all its optimal decompositions, namely:
$$
\Gs(\w) \eqdef \{g \in \G \mid \exists \vb \in \Vb(\w) \:\: \text{s.t.}\:\: \v^g \neq \zv \, \}\,.
$$
\end{definition}
If $\w$ has a unique decomposition $\vb(\w)$, then $\Gs(\w)=\{g \in \G \, | \, \v^g(\w) \neq \zv\}$ is the group-support of its decomposition.
We also define a notion of {\sl weak group-support} in terms of uniqueness of the optimal dual variables.
\begin{definition}[\bf Weak group-support]
The {\sl weak group-support} of a vector $\w\in\RR^p$ is  
 $$
 \Gw(\w) \eqdef \big \{g \in \G \,|\, \exists \, \alphab_g\in\RR^p\:\, \text{s.t.}\:\, \Pi_g\mathcal{A}(\w)=\{\alphab_g\}\: \text{and}\: \nm{\alphab_g}=d_g \,\big \} \,.
 $$
\end{definition}
It follows immediately from Lemma \ref{lem:alpha} that $\Gs(\w) \subset \Gw(\w)$.
When $\Gs(\w)=\Gw(\w)$, we refer to $\Gs(\w)$ as the group-support of $\w$; otherwise we say that the group-support is ambiguous.

The definitions of {\sl strong group-support} and {\sl weak group-support} are motivated by the fact that in the variational formulation \eqref{dualvar},
the {\sl strong group-support} is the set of groups for which the constraints $\nm{\alphab_g} \leq 1$ are {\sl strongly active} whereas the {\sl weak group-support}
is the set of {\sl weakly} or {\sl strongly active} such constraints \citep[p.342]{Nocedal2006Numerical}. We illustrate these two notions on a few examples in Section \ref{sec:examples}.

\subsection{Supports induced by the group-support}
For any $\w\in\RR^p$, we denote by $\Jw(\w)$ (resp. $\Js(\w)$) the set of variables in groups of the weak group-support (resp. strong group-support):
 $$
 \Jw(\w) \eqdef \bigcup_{g \in \Gw(\w)}  g \qquad \text{and} \qquad \Js(\w) \eqdef \bigcup_{g \in \Gs(\w)} \: g \,.
 $$
Since $\Gs(\w) \subset \Gw(\w)$, it immediately follows that $\Js(\w) \subset \Jw(\w)$\footnote{It is possible to have $\Js(\w) \neq \Jw(\w)$ consider $\G=\big \{ \{1,2\},\{1,3\},\{2,3,4\} \big \}$ and $\w=\frac{1}{\sqrt{2}}(1, \mu, 1-\mu,0)$ for any $\mu \in (0,1)$. We then have $\Gs=\big \{ \{1,2\},\{1,3\} \big \}$ and $\Gw=\G$ so that $\Js=\{1,2,3\} \neq  \Jw=\{1,2,3,4\}$.}. The following two lemmas show that, on $\Jw(\w)$, any dual variables $\alphab\in\A(\w)$ are uniquely determined.
\begin{lemma}
\label{Jbar}
If  $\Jw(\w) \backslash \Js(\w) \neq \varnothing$, then for any $\alphab\in\A(\w)$, $\alphab_{\Jw(\w) \backslash \Js(\w)}=\zv$.
\end{lemma}
\begin{proof}
Note that $\w_{\Jw(\w)\backslash \Js(\w)}=\zv$ since $\v^g=\zv$ for $g \in \Gw(\w) \backslash \Gs(\w)$. Let $g \in \Gw(\w) \backslash \Gs(\w)$. If $g \backslash \Js(\w) \neq \varnothing$,
and if $\Pi_{g \backslash \Js(\w)}\mathcal{A}(\w) \neq \{\zv\}$ then, let $i \in g \backslash \Js(\w)$ such that there exists $\alphab \in \mathcal{A}(\w)$ with $\alphab_i \neq 0$. Setting $\alphab_i=0$ leads to another vector that solves the second variational formulation \eqref{eq:variational1} and such that $\nm{\alphab_g} < d_g$ which contradicts the hypothesis that $g \in \Gw(\w)$.
\end{proof}

\begin{lemma} For any $\w\in\RR^p$, $\Pi_{\Jw(\w)}\A(\w)$ is a singleton, \ie, there exists $\alphab_{\Jw(\w)}\!(\w) \in \RR^{|\Jw(\w)|}$
  such that, for all $\alphab' \in \A(\w), \;
  \alphab'_{\Jw(\w)}=\alphab_{\Jw(\w)}\!(\w)$.
\label{Aunq}
\end{lemma}
\begin{proof}
  By definition of $\Js(\w)$, for all $i \in \Js(\w)$ there exists at least one
  $\v\in\Vb(\w)$ and one group $g\ni i$, such that $(\v^g)_i \neq 0$. Now
  as a consequence of Lemma~\ref{lem:alpha}, for any two solutions
  $\alphab,\alphab'\in\A(\w)$, we have that $\alphab_g = \alphab'_{g} =
  d_g\frac{\v^g}{\nm{\v^g}}$, so in particular $\alphab_i = \alphab'_i$.
  For $i \in \Jw(\w) \backslash \Js(\w)$, Lemma \ref{Jbar} shows that $\alphab_i=0$.
\end{proof}

%% file: examples.tex
In this section, we consider a few examples that illustrate some of
the properties of $\Oo$, namely situations where weak and strong group
support differ, or where there is an entire set of optimal
decompositions. We will abuse notations and write $\v_g$ for $\v^g_g$
when writing explicit decompositions. We will denote by ${\Sign}$ the correspondence (or set-valued function) defined by $\Sign(x)={1}$ if $x>0$, $\Sign(x)={-1}$ if $x<0$ and $\Sign(0)=[-1,1]$.

\subsection{Two overlapping groups}
We first consider the case $p=3$ and $\G=\{\{1,2\},\{2,3\}\}$.
\begin{lemma} 
We have $\Oo(\w)=\nm{(w_2,|w_1|+|w_3|)^\top}$. If $(w_1,w_3)\neq \zv$, the optimal decomposition is unique
with 
\begin{equation}
\v_{\{12\}}=
\Big ( w_1\:,\:
\frac{ |w_1|}{|w_1|+|w_3|}\,w_2 \Big )^\top
 \quad \text{and} \quad
 \v_{\{23\}}=
\Big ( 
\frac{ |w_3|}{|w_1|+|w_3|} \, w_2\:,\:
w_3 \Big )^\top,
\end{equation}
\begin{equation*}
\mathcal{A}(\w)=\big\{\, \big ( (|w_1|\!+\!|w_3|) \, \gamma_1,w_2, (|w_1|\!+\!|w_3|)\, \gamma_3 \big ) /{\Oo(\w)} \mid \gamma_i \in \Sign(w_i), i \in \{1,3\} \big\},
\end{equation*}
$\Jw=\Js=\supp{\w}$ and $\Gw=\Gs$ includes $\{w_1,w_2\}$ if $w_1\neq 0$ and $\{w_2,w_3\}$ if $w_3\neq 0$.\\
If $(w_1,w_3)=\mathbf{0}$, then $\v_{\{12\}}=
( 0\:,
\gamma \, w_2 )^\top
 \quad \text{and} \quad
 \v_{\{23\}}=
( \,(1-\gamma) \, w_2\:,\:
0  )^\top
$ is an optimal decomposition for any $\gamma \in [0,1]$, $\mathcal{A}(\w)=\{(0,\sign(w_2),0)\}$, $\Jw=\Js=\{1,2,3\}$ and $\Gw=\Gs=\G$.
\end{lemma}
We prove this lemma in section \ref{sec:at_most_two} (as a special
case of the ``cycle of length three'' which we consider next). Here,
the case where the decomposition is not unique seems to be a
relatively pathological case where the true support is included in the
intersection of two groups. However, note that the weak group-support
and strong-group support coincide, even in the latter case.

\subsection{Cycle of length 3}
\label{sec:tricyle}
We now turn to the case $p=3$ and $\G=\{\{1,2\},\{2,3\},\{1,3\}\}$.
Note that if at least one of the groups is not part of the weak-group support, we fall back on the case of two overlapping groups. We therefore have the following lemma:
\begin{lemma} 
Define  $\mathcal{W}_{\text{bal}} \eqdef \big \{ \w \, \in \RR^3 |\:
 |w_i| \leq \nm{\w_{\{i\}^c}}_1,\: i \in \{1,2,3\} \big \}$. We have
 \begin{equation*}
\Om{\w}=
\begin{cases}
\frac{1}{\sqrt{2}}\, \nm{\w}_1 & \text{if}\; \w \in \mathcal{W}_{\text{bal}} \\
\min_{i \in \{1,2,3\}}
 \big \|\, \big (w_i,\|\w_{\{i\}^c}\|_1 \big )\, \big \|
& \text{else.}
\end{cases}
\end{equation*}
If $|\supp{\w}| \neq 1$ the optimal decomposition is unique.
If in addition, $\w \in \mathcal{W}_{\text{bal}}$ we have for $(i,j,k) \in \{(1,2,3),(2,3,1),(3,1,2)\}$:
$$\v_{\{ij\}}=\frac{1}{2} \, (|w_i| +|w_j|-|w_k|) 
\begin{pmatrix}
\text{sign}(w_i)\\
\text{sign}(w_j)\\
\end{pmatrix}
\quad \text{and} \quad \mathcal{A}(\w)= \big \{\frac{1}{\sqrt{2}} \big (\sign(w_1),\sign(w_2),\sign(w_3) \big ) \big \}.$$
Moreover, we have $\Jw=\Js=\{1,2,3\}$, $\Gw=\G$ and for $\w \in \mathring{\mathcal{W}}_{\text{bal}}$, $\Gw=\Gs=\G$.
\end{lemma} 
We prove this lemma in appendix \ref{sec:cycle_three}, and illustrate it on Figure~\ref{fig:tortoise} with the unit ball of the obtained norm. In this case it is interesting to note that the group-support (weak or strong) is not necessarily a \textit{minimal cover}, where we say that a set of groups provides a minimal cover if it is impossible to remove a group while still covering the support. For instance, for $\w$ in the interior of $\mathcal{W}_{\text{bal}}$, the group-support contains all three groups, while the support is covered by any two groups. This is clearly a consequence of the convexity of the formulation.
The cycle of length 3 is also interesting because, for any $\w$ on the boundary of $\mathcal{W}_{\text{bal}}$, the weak and strong group-support do not coincide, as illustrated on Figure~\ref{fig:tortoise} (right).   Indeed if for example $|w_3|=|w_1|+|w_2|$, then $\v_{\{1,2\}}=(0,0)^\top, \v_{\{1,3\}}=|w_1| (\sign(w_1),\sign(w_3))^\top$ and $\v_{\{2,3\}}=|w_2|(\sign(w_2),\sign(w_3))^\top$ so that by lemma ~\ref{lem:alpha} the dual variable satisfies $\nm{\alphab_{\{1,2\}}}=1$, which means that $\{1,2\}$ is in the weak but not in the strong group-support. 

\begin{figure}
\centering
\includegraphics[width=6cm]{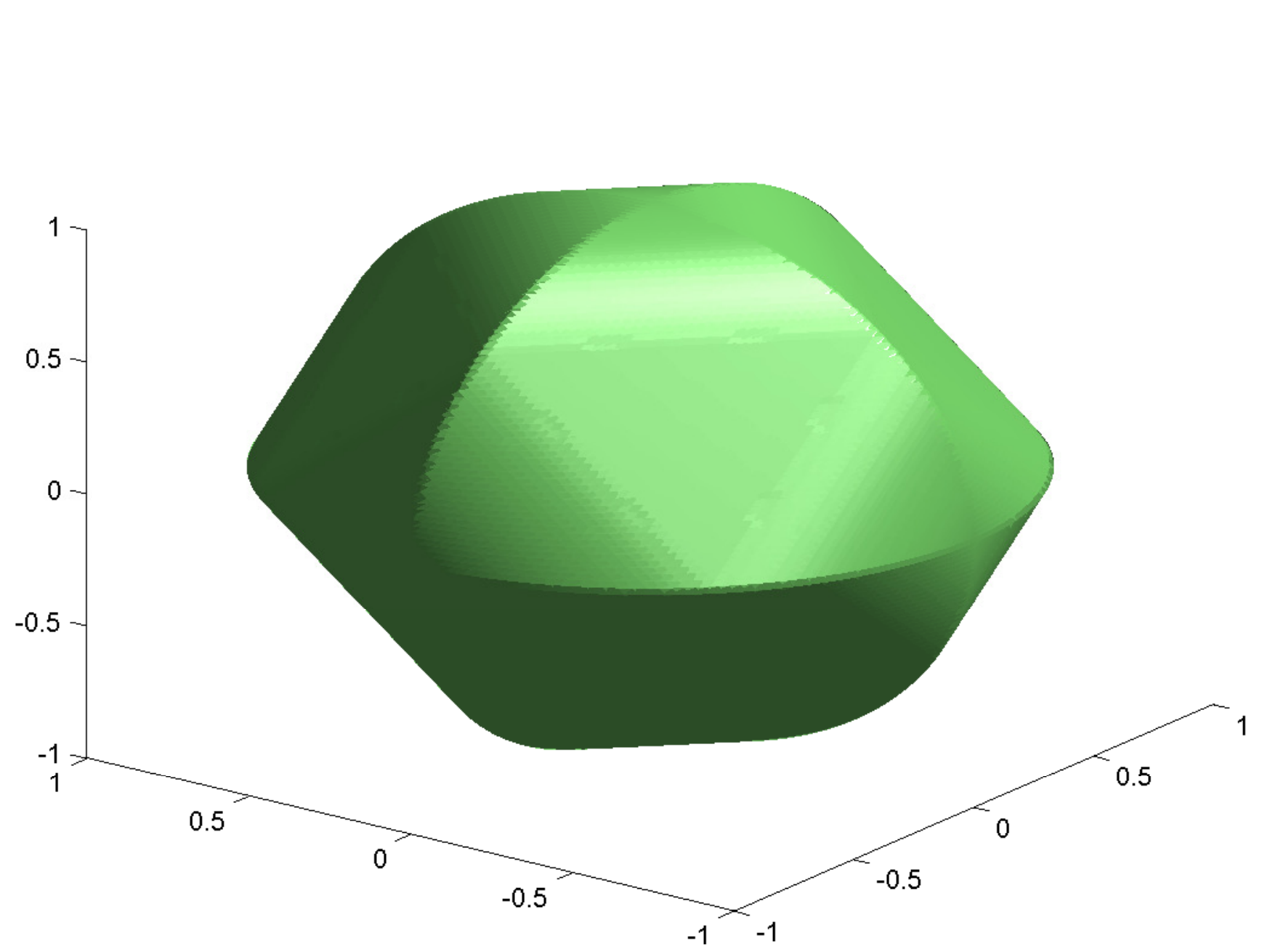} \hspace{1cm}
\includegraphics[width=6cm]{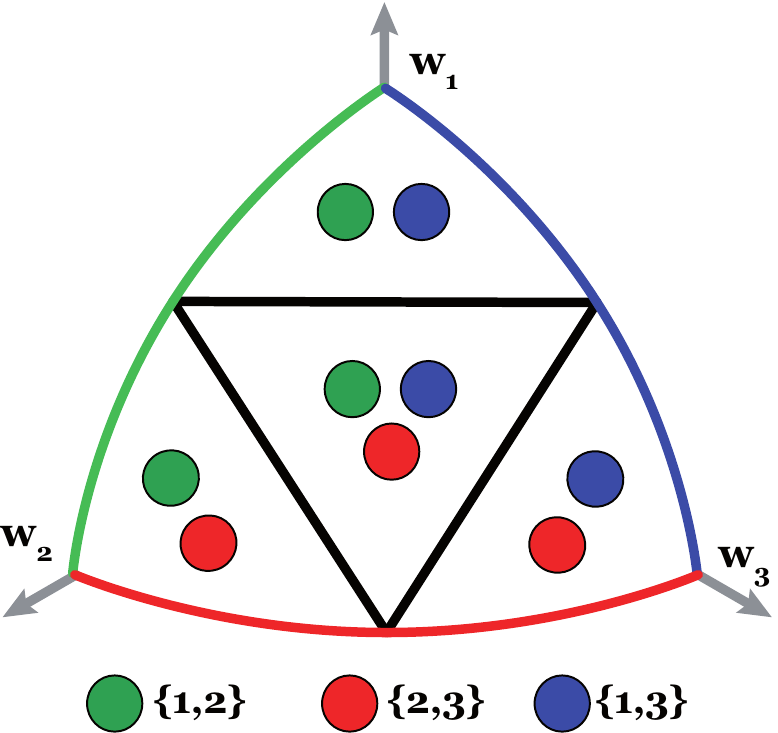}
\caption{(Left) The unit ball of $\Oo$ for the groups $\{1,2\},
  \{1,3\}, \{2,3\}$ in $\RR^3$. (Right) a diagram that represents the
  restriction of the unit ball to the positive orthant. The black
  lines separate the surface in four regions. The triangular central
  region is $\mathcal{W}_{\text{bal}}$.  On the interior of each
  region and on the colored outer boundaries, the group-support is
  constant, non-ambiguous (i.e., the weak and strong group-supports
  coincide) and represented by color bullets or the color of the edge,
  with one color associated to each group. On the boundary of
  $\mathcal{W}_{\text{bal}}$, the black lines indicate the
  group-support is ambiguous, the weak group-support containing all
  three groups, and the strong group-support being equal to that of
  the outer adjacent region for each black segment.  }
\label{fig:tortoise}
\end{figure}

\subsection{Cycle of length 4}
We consider the case $p=4$ and show the following result in appendix \ref{sec:cycle_four}.
\begin{lemma}
For $\G=\big \{ \{1,2\},\{1,3\},\{2,4\},\{3,4\} \big \}$. $\Oo$ has the closed form
$$\Oo(\w)=\nm{(|w_1|+|w_4|,|w_2|+|w_3|)}_2.$$ However, if $|\supp{\w}|=4$, the optimal decomposition is never unique.
\end{lemma}
This suggests that for a general $\G$, unique solutions are the exception rather than the rule. This motivates a posteriori definitions of group-support that are meaningful in the case where the decomposition is not unique. We consider a necessary and sufficient condition for uniqueness in lemma \ref{lem:uniqueness}.

%% file: consistency.tex
\label{sec:supp_rec}
In this section we consider the estimator $\wh$ obtained as a solution of the learning problem 
(\ref{eq:opt1}) in the context of a well-specified model. Specifically, we consider the linear regression model:
\begin{equation}\label{eq:linreg}
\y = \X \ws + \epsb\,,
\end{equation}
where $\X\in\RR^{n\times p}$ is a design matrix, $\y\in\RR^p$ is the response vector and $\epsb\in\RR^p$ is a vector of i.i.d. random variables with mean $0$ and finite variance. We denote by $\ws$ the true regression function, and by $\wh$ the one we estimate as the solution of the following optimization problem, which is a particular case of~\eqref{eq:opt1}:
\begin{equation}\label{eq:linregopt}
  \min_{\w\in\RR^p} \frac{1}{2n}\nm{\y-\X \w}^2 +
  \lambda_n \, \Oo(\w)\,.
\end{equation}

Several types of consistency results are of interest when using a sparsity-inducing norm as a regularizer. One typically distinguishes \textit{classical consistency} where $\nm{\wh-\ws}_p$ converges in probability  to zero, \textit{prediction consistency} where  $|L(\wh)-L(\ws)|$ converges to zero in probability, and \textit{model selection consistency} or \textit{support recovery} where the support of $\wh$ coincides with the support of $\ws$ with high probability. We are interested in the discussion of the last type of result, support recovery, for solutions of~\eqref{eq:linregopt}. 

As compared with the Lasso and the group Lasso in the case of disjoint
supports, the discussion of support recovery is complicated by several
factors here. First, supports that can be recovered are not exactly
the ones that can be expressed as unions of groups in $\G$: as the
reader might expect, the appropriate notion of support is $\Jw(\ws)$ (or $\Js(\ws)$),
the one induced by the concept of group-support introduced in section
\ref{sec:group-support}. 
Second, by contrast with the
situation of the group Lasso with disjoint groups, the identification
of the support $\Jw(\ws)$ (or $\Js(\ws)$) is not equivalent to the identification of the
group-support $\Gw(\ws)$ (or $\Gs(\ws)$), the latter being now a harder problem.
As a consequence one should distinguish \text{support recovery} from
\text{group-support recovery}, and, depending on the context, the
appropriate notion to consider for \text{model selection consistency}
might be one or the other. Third, the group-support is characterized
by properties of the set $\Vb(\wh)$ whose convergence is less trivial
to study than that of a vector. For these reasons, we consider only in
this paper the classical asymptotic regime in which the model
generating the data is of fixed finite dimension $p$ while $n
\rightarrow \infty$. However we focus on the harder problem of
\textit{group-support recovery}, which will then imply \textit{support
  recovery} results.  

The proof of consistency we present below follows a classical proof scheme \citep{Bach2008Consistencya}.
However the originality of our work reside in that we characterize the \text{group-support consistency} here, which requires in particular to study the convergence of the set-valued map $\Vb(\wh)$. We therefore start in the next section by introducing appropriate notions of continuity for set-valued functions.

\subsection{Correspondence theory to the rescue}
We appeal to the theory of {\sl correspondences} developed by Claude Berge at the end of the 1950's \citep{Berge1959Espaces}. In particular, we follow closely its presentation by \citet{Border1985Fixed}.
\begin{definition}[\bf correspondence]
A {\sl correspondence} $\phi$ from a set $X$ to a set $Y$, denoted $\phi: X \twoheadrightarrow Y$, is a set-valued mapping which to each element $x \in X$ associates a set $\phi(x) \subset Y$.
\end{definition}
When $X$ and $Y$ are metric spaces, the usual notion of continuity of a function is replaced for correspondences by the following notions:
\begin{definition}[\bf hemicontinuity and continuity]
Given two metric spaces $(X,d)$ and $(Y,\rho)$, a correspondence $\phi: X \twoheadrightarrow Y$ is said to be {\sl upper hemicontinuous} or u.h.c. (resp. {\sl lower hemicontinuous} or l.h.c.) if for any point $x\in X$ and any open set $U\subset Y$ such that $\phi(x) \subset U$ (resp. $\phi(x) \cap U \neq \varnothing$) there exists a neighborhood $V$ of $x$ such that, for all $x' \in V$, $\phi(x') \subset U$ (resp. $\phi(x') \cap U \neq \varnothing$). A correspondence is said to be {\sl continuous} if it is both upper and lower hemicontinuous.
 \end{definition}
Note that a singleton valued correspondence $\phi$ can be identified with the function $f$ taking this unique value, and that $f$ is continuous if and only if $\phi$ is lower or upper hemicontinuous, both notions being equivalent in that case. The following results, which we prove in appendix \ref{sec:two_lem}, are key to study the consistency of our method in the next section.
\begin{lemma}
\label{Auhc}
$\w \mapsto \mathcal{A}(\w)$ is an upper hemicontinuous correspondence.
\end{lemma}
\begin{lemma}
\label{Vlhc}
If $\supp{\w}=\Jw$, then, on the domain $\mathcal{D}=\{\u \in \RR^p \mid \supp{\u}=\Jw\}$, $\u \mapsto \Vb(\w+\u)$ is a continuous correspondence at $\u=\zv$.
\end{lemma}

\subsection{Group-support recovery}
In this section, we state and prove our main consistency results for group-support and support recovery in the least-square linear regression framework (\ref{eq:linreg}). We consider two main hypotheses:
$$
\text{(H1)}  \quad \XX\eqdef\frac{1}{n} \X^\top \X \succ 0\,, \qquad \qquad 
\text{(H2)} \quad  \supp{\ws}=\Jw(\ws)\,.
$$

We denote $\Gwc(\ws) \eqdef\G \backslash \Gw(\ws)$ and
$\Jwc(\ws)\eqdef [1,p\,]\backslash \Jw(\ws)$.  For convenience, for
any group of covariates $g$ we note $\X_g$ the $n\times\abs{g}$ design
matrix restricted to the covariates in $g$, and for any two groups $g,
g'$ we note $\XX_{gg'}=\frac{1}{n}\X_g^\top \X_{g'}$.

Consider the following two conditions, where we denote $\Jw(\ws)$ simply by $\Jw$ for sake of clarity:
\begin{equation}
\label{eq:necmutinc}
\forall g \in \Gwc(\ws)\,, \quad \nm{\XX_{g \Jw} \XX_{\Jw \Jw}^{-1} \alphab_{\Jw}(\ws)} \leq d_g\,,
\tag{C1}
\end{equation}
\begin{equation}
\label{eq:suffmutinc}
\forall g \in \Gwc(\ws)\,, \quad \nm{\XX_{g \Jw} \XX_{\Jw \Jw}^{-1} \alphab_{\Jw}(\ws)} < d_g\,.
\tag{C2}
\end{equation}

\begin{theorem}
\label{theo:part_cons}
 Under assumption (H1), for $\lambda_n \rightarrow 0$
  and $\lambda_n n^{1/2} \rightarrow \infty$, conditions
  (\ref{eq:necmutinc}) and (\ref{eq:suffmutinc}) are respectively
  necessary and sufficient for the strong group-support of the solution of (\ref{eq:opt1}), $\Gs(\wh)$ to
  satisfy with probability tending to $1$ as $n\rightarrow +\infty$:
  $$\Gs(\wh) \subset \Gw(\ws)\,.$$
\end{theorem}
\begin{proof}
  We follow the line of proof of \cite{Bach2008Consistencya} but
  consider a fixed design for simplicity of notations. Let us first
  consider the subproblem of estimating a vector only on the support
  of $\ws$ by using only the groups in $\Gw(\ws)$ in the penalty, \emph{i.e.},
  consider $\w_1\in\RR^{\Jw}$ a solution of
$$
  \min_{\w_{\Jw}\in\RR^{\Jw}} \frac{1}{2n}\nm{\y-\X_{\Jw} \w_{\Jw}}^2 +
  \lambda_n \, \Omone{\w_{\Jw}}\,.
$$
By standard arguments, we can prove that $\w_1$ converges in Euclidean
norm to $\ws$ restricted to $\Jw$ as $n$ tends to infinity
\citep{Knight2000Asymptotics}.  In the rest of the proof we show how
to construct a vector $\w\in\RR^p$ from $\w_1$ which under condition
(\ref{eq:suffmutinc}) is with high probability a solution to
\eqref{eq:linregopt}.  By adding null components to $\w_1$, we obtain a
vector $\w\in\RR^p$ whose support is also $\Jw$, and $\u=\w-\ws$
therefore satisfies $\supp{\u}\subset \Jw$. A direct computation of the
gradient of the loss $L(\w)= \frac{1}{2 n}\nm{\y - \X\w}^2$ gives $ \nabla L(\w) =
\XX \u - \qn\,$, where $\qn=\frac{1}{n}\X^\top\epsb$.  From this we deduce
that $\u_{\Jw}=\XX_{\Jw \Jw}^{-1} \br{ \nabla_{\Jw} L(\w) + \qn_{\Jw}}$, and
since, by Lemma \ref{lem:equivgrad}, $-\nabla_{\Jw} L(\w) \in \lambda_n \Pi_{\Jw}\mathcal{A}(\w)$, there exists  $\alphab_{\Jw} \in \Pi_{\Jw}\mathcal{A}(\w)$ such that we have
$$
\nabla_{\Jwc} L(\w) = \XX_{\Jwc \Jw}\u - \qn_{\Jwc} =
\XX_{\Jwc \Jw}\XX_{\Jw \Jw}^{-1} \br{ -\lambda_n \alphab_{\Jw} + \qn_{\Jw}} - \qn_{\Jwc}\,.
$$
To show that $\w$ is a feasible solution to (\ref{eq:linregopt}) it is enough to show that $\forall g \in \Gwc(\ws),\: \nm{\nabla_{\!g} L(\w)} \leq \lambda_n \, d_g $.
But since the noise has bounded variance, 
$$ 
\XX_{\Jwc \Jw}\XX_{\Jw \Jw}^{-1}\qn_{\Jw} - \qn_{\Jwc} =\frac{1}{n} \X_{\Jwc}^\top \left [ \frac{1}{n} \X_{\Jw} \XX_{\Jw \Jw}^{-1} \X_{\Jw}^\top - I \right ] \epsb
$$
is $\sqrt{n}$-consistent, and by the union bound we get
$ \PP(\forall g \in \Gwc(\ws), \nm{\nabla_{\!g} L(\w)} \leq \lambda_n\,  d_g)
\geq 1- \sum_{g \in \Gwc(\ws)} \PP(\nm{\nabla_{\!g} L(\w)} > \lambda_n\,  d_g)$.
We therefore deduce that, for any $g\in\Gwc(\ws)$,
$$
\frac{1}{\lambda_n}\nm{\nabla_{\!g} L(\w)} \leq \nm{\XX_{g \Jw} \XX_{\Jw
  \Jw}^{-1} \alphab_{\Jw}}+\Op(\lambda_n^{-1} n^{-1/2}) \,.
$$
By Lemma~\ref{Auhc}, we have that $\Pi_{\Jw}\mathcal{A}(\w)$ is an upper hemicontinuous correspondence so that
$\w_{\Jw} \overset{\PPP}{\rightarrow} \ws_{\Jw}$ implies that $$\max_{\alphab' \in \mathcal{A}(\w)}\nm{\alphab'_{\Jw}-\alphab_{\Jw}(\ws)} \overset{\PPP}{\rightarrow}0\,.$$
Since we chose $\lambda$ such that $\lambda_n^{-1} n^{-1/2} \rightarrow 0$, we have
$$\frac{1}{\lambda_n}\nm{\nabla_{\!g} L(\w)} \leq \nm{\XX_{g \Jw} \XX_{\Jw \Jw}^{-1} \alphab_{\Jw}(\ws)}+o_p(1)\,.$$
This shows that, under (C2), $\w$ is a feasible solution to (\ref{eq:linregopt}) whose group-support is contained in $\Gw(\ws)$, i.e., we have shown $\Gs(\wh) \subset \Gw(\ws)$.

For the
necessary condition, by contradiction, consider a solution supported on $\Jw$. Then, reusing the previous argument we have
$$\frac{1}{\lambda_n}\nm{\nabla_{\!g} L(\w)} \geq \nm{\XX_{g \Jw} \XX_{\Jw \Jw}^{-1} \alphab_{\Jw}(\ws)}-o_p(1)\,,$$
which shows that for the optimality conditions of Lemma \ref{lem:equivgrad}(b) to hold, condition (C1) is necessary.
\end{proof}
The previous theorem shows some partial consistency result in the sense that it guarantees that no group outside of the group-support will be selected. Since $\wh$ also converges with high probability in Euclidean norm to $\ws$, this implies for the support that with high probability
$$\supp{\ws} \subset \supp{\wh} \subset \Jw(\ws)\,.$$

However, the theorem does not guarantee that all groups in $\Gs(\ws)$ will be selected. This is not a shortcoming of the theorem: we provide an example in \appref{sec:part_cons} which shows that it is possible that $\Gs(\wh) \subsetneq \Gs(\ws)$ with probability 1. Nonetheless, we also show in the same appendix that with high probability there exists $\vbs \in \Vb(\ws)$ whose group-support is included in $\Gs(\wh)$.

\begin{theorem}
\label{theo:ambig_cons}
 With assumptions (H1,H2) and for $\lambda_n \rightarrow 0$
  and $\lambda_n n^{1/2} \rightarrow \infty$, condition
  (\ref{eq:necmutinc}) is sufficient for the strong group-support of the solution of (\ref{eq:linregopt}), $\Gs(\wh)$, to
  satisfy with high probability:
  $$
  \Gs(\ws) \subset \Gs(\wh) \subset \Gw(\ws)\,.
  $$
\end{theorem}
\begin{proof}
The previous theorem shows that (C1) implies, with high probability, $ \Gs(\wh) \subset \Gw(\ws)$.
However, by Lemma \ref{Vlhc},
we have that hypothesis (H2) guarantees that $\w \mapsto \Vb(\w)$ is continuous at $\ws$ for $\w$ with $\supp{\w} \subset \Jw(\ws)$.
Combined with the fact that $\wh$ converges in probability with $\ws$, this implies that $\forall \epsilon >0, \exists n_0, \forall n>n_0$, with probability larger than $1-\epsilon$, $\forall \vbs \in \Vb(\ws)$, 
there exists $\vb \in \Vb(\wh)$ such that $\nm{\vb-\vbs}<\epsilon$. For each $g \in \Gs(\ws)$, for $\vbs \in \Vb(\ws)$ such that ${\vs}^g \neq 0$, there thus exists $\epsilon>0$ such that the previous convergence results implies that $g \in \Gs(\wh)$ with high probability.
Finally, since $|\Gs(\ws)|$ is finite, for $n$ large enough, the union bound ensures that, with high probability, $\Gs(\ws) \subset \Gs(\wh)$.
\end{proof}

The previous theorem shows the best result possible for the situation where $\Gs(\ws)\neq \Gw(\ws)$, as, in the example of the cycle of length 3 of section~\ref{sec:tricyle}, the case of $\ws=(2,1,1)$. 
If $\Gs(\ws)=\Gw(\ws)$, then we have the obvious corollary:

\begin{corollary}
\label{cor:consistency}
 With assumptions (H1,H2), and assuming $\Gs(\ws)=\Gw(\ws)$, for $\lambda_n \rightarrow 0$
  and $\lambda_n n^{1/2} \rightarrow \infty$, conditions
  (\ref{eq:necmutinc}) and (\ref{eq:suffmutinc}) are respectively
  necessary and sufficient for the solution of (\ref{eq:opt1}) to
  estimate consistently the correct group-support $\Gw(\ws)$.
\end{corollary}

{\bf Remarks}: For the Lasso and the usual group Lasso with disjoint groups, the most favorable case w.r.t. to condition (C2) is the case where the empirical covariance of the design is the identity (the same analysis can be done in the random design case), \ie, the case where there is no correlations between groups.
In that case, we have $\XX_{\Jwc \Jw} \XX_{\Jw \Jw}^{-1}=0$ and the mutual incoherence condition is 0. However, in the case of overlap, for $g \in \G$ such that $g \cap \Jw \neq \emptyset$, then $\XX_{g \Jw} \XX_{\Jw \Jw}^{-1} \neq 0$ and we have $\nm{\XX_{g \Jw} \XX_{\Jw \Jw}^{-1} \alphab_{\Jw}}=\nm{\alphab_{g \cap \Jw}}$. First, this gives yet another motivation to consider the weak-group support, since those groups in the weak-group support are exactly the ones for which $\nm{\alphab_{g \cap \Jw}}=1$ (see Lemma \ref{Jbar}). Second this show that if $g_1 \in \Gs(\ws)$ and $g_2 \notin \Gw(\ws)$ have a large overlap then $\nm{\alphab_{g_1 \cap g_2}}$ can be fairly close to $1$ even for a design with identity covariance. This means that it might be very difficult in practice to identify $g_2$ correctly as being outside of the support unless large amounts of data are available.

\subsection{Related theoretical results}
\label{sec:related_theo}
Two papers proposed recently some theoretical results on the estimator
via regularization by $\Oo$ in the high-dimensional setting.
\citet{Percival2011Theoretical} shows two types of results. First, he
proposes a generalization of the restricted eigenvalue condition of
\citet{Bickel2009Simultaneous} and generalize their proof to obtain
fast-rate type of concentration results for the prediction error and
convergence in $\ell_2$-norm. The bounds obtained scales as $\sqrt{B}
\log(M)$, where $M$ is the total number of groups and $B$ is the
largest group size.  Then he considers an adaptive version of the
regularization (in the sense of the adaptive Lasso) and shows for the
resulting estimator a central limit theorem under high-dimensional
scaling, under the conditions that the support is exactly a union of
groups and that the decomposition of any point in a neighborhood of
the optimum is unique.  These results do not focus on support or
group-support recovery. Also, it was one of our concerns to relax the
assumption that the decomposition was unique or that the support was
exactly a union of groups.

\citet{Maurer2011Structured} give a bound on the Rademacher complexity of 
linear functions whose parameter vector lies in the unit ball of the norm $\Omn$, hence bounding the generalization error of such function. They consider as well
extensions of this norm where each of the latent variables in the latent group Lasso are penalized by the norm of their image by some operator.

Our paper and these two papers have thus considered complementary
aspects of estimation and recovery in statistical and compressed
sensing based on $\Omn$ settings which should all contribute to
understanding the high-dimensional learning setting.

%% file: weights.tex
The choice of the weights $d_g$ associated to each group
has been discussed in the literature on the classical group Lasso, when groups do not overlap. The main motivation for the introduction of these weights is to take into account the discrepancies of size existing between different groups.
\cite{Yuan2006Model} used $d_g =
\sqrt{|g|}$, which yields solutions similar to the ANOVA test under a
certain design. \cite{Bach2004Multiple} in the context of
\emph{multiple kernel learning} used $d_g \propto \sqrt{\text{tr}
    K_g}$, where $\{K_g\}_{g\in\G}$ are positive definite
kernels, with $K_g =\X_g\X_g^\top$ in our context; for normalized features such as $\X\X^\top = I$,
this yields $d_g =\sqrt{|g|}$ as well. 

In the context of our latent group Lasso with overlapping groups, the choice of the weights 
is significantly more important than in the case of disjoint groups, and, arguably, than
in the case of other formulations considering overlapping groups: indeed, the notions
of group-support $\Gw(\w)$ and $\Gs(\w)$ and of support $\Jw(\w)$ and $\Js(\w)$ associated to a vector $\w$ through the norm $\Oo(\w)$ themselves change according to the choice of the weights.

In this section we propose two types of arguments to study the effect of and guide the choice of weights:
\begin{itemize}
\item On the one hand we consider
a vector $\w$ and ask, independently of a learning problem, which groups participate in its group support: there is no point in introducing a group in $\G$ if the weights are such that it can never be included in the group support. We show in \secref{sec:redundant} that, for all groups to be useful, weights should increase with the size of the groups, but not too quickly; in \secref{sec:dominating} we attempt to characterize when large groups are preferred over unions of smaller ones.

\item On the other hand, we consider in \secref{sec:fdr} a simple regression scenario, and discuss the impact of the weights on the probability to correctly identify relevant groups, and simultaneously control the rate of false positives.
\end{itemize}

\subsection{Redundant groups}
\label{sec:redundant}
Informally, we are concerned in this section with the fact that, if a group $g$ contains a group $h$ and ${d_g}/{d_h}$ is too small, $h$ will never enter the group support, and, conversely,
if $g$ is covered by a certain number of groups and $d_g$ is too large, then $g$ will never enter the group-support.

Formally, we say that a group $g\in\G$ is \textit{redundant} for a certain set of weights $(d_g)_{g \in \G}$ 
if it can be removed without changing the value of the norm $\Oo$ for any $\w$; this is equivalent to asking that the dual norm $\Oo^*$ is unchanged.

We first show that if there exists another group $g'\in\G$ such that $g \subset g'$, $g$ is redundant unless we require that $d_g < d_g'$: 

\begin{lemma}
\label{lem:weights_incr}
If $g,g'\in\G$ satisfy $g \subset g'$ and $d_g \geq d_{g'}$, then for any $\w$, $(g \in \Gw(\w)) \Rightarrow (g'\in \Gw(\w))$.
\end{lemma}

\begin{proof} If $d_g \geq d_{g'}$, and if $g \in \Gw(\w)$ then $1=\frac{\nm{\alphab_g(\w)}}{d_g}\leq \frac{\nm{\alphab_{g'}(\w)}}{d_{g'}}$, which implies $g' \in \Gw(\w)$. 
\end{proof}

It would be very natural to try and require that the weights are
chosen so that, if $g=\supp{\w}$, its group-support is exactly
$g$. Unfortunately, this is in general not possible: we show a
negative result, which arises as a consequence of the previous lemma.
\begin{lemma}
\label{lem:impossible}
For some group sets $\G$, it is impossible to choose the weights $d_g$ independently of $\w$
so that $\Jw(\w)=\supp{\w}$ (or $\Js(\w)=\supp{\w}$) if the latter is a union of groups.
\end{lemma}
\begin{proof}
Consider the groups $A=\{1,2,3\}$, $B=\{3,4\}, C=\{2,3,4\}$~:
\begin{itemize}
\item To have that $\Js(\w)=\supp{\w}$ for all $\w$
 Lemma \ref{lem:weights_incr} imposes that $d_B < d_C$ so that $B$ is
 not redundant; this is necessary to have $\Js(\w)=\supp{\w}=B$ for
 $\w=(0,0,w,w)$.
\item Then consider $\w=(0,w,\epsilon,\epsilon)$. $\Js(\w) =
  \supp{\w}$ requires that $\Gs(\w) = \{C\}$. But then $\v^C=\w$ so that
  $\alphab=d_C\, \w/\|\w\|$. In particular $\|\alphab_A\|^2=d_C^2 \,(w^2+\epsilon^2)/(w^2+2\epsilon^2)$ and $\|\alphab_B\|^2=d_C^2 \,2\epsilon^2/(w^2+2 \epsilon^2)$. For the inequality
 $\|\alphab_A\|\leq d_A$ to hold for all $\epsilon>0$, we need $d_A \geq d_C$.
 
\item Finally consider $\w=(\epsilon,\epsilon,w,0)$. Following the
  same line as for the previous case, $\Js(\w) =
  \supp{\w}$ requires that $\Gs(\w) = \{A\}$, which implies that $\v^A=\w$ so that
  $\alphab=d_A \, \w/\|\w\|$. In particular $\|\alpha_B\|^2=d_A^2 \,w^2/(w^2+2\epsilon^2)$ and $\|\alphab_C\|^2=d_A^2 \,(w^2+\epsilon^2)/(w^2+2\epsilon^2)$. For the inequalities,
  $\|\alphab_B\| \leq d_B$ and $\|\alphab_C\|\leq d_C$ to hold for all $\epsilon>0$, we need to have $d_A \leq d_B$.
\end{itemize}
These three inequalities are clearly incompatible and $\Js(\w) \subset \Jw(\w)$ which proves the result.
\end{proof}

We now characterize more technically redundancy. The intuition behind the next lemma is the following geometric interpretation of the dual norm: the definition of $\Oo^*$ implies that its unit ball is the intersection of cylinders of the form $\{\alphab \mid \|\alphab_g\| \leq d_g\}$. This means that a group $g$ is redundant if its associated cylinder contains the unit ball of the norm induced by the remaining groups. This can be formally stated as follows:

\begin{lemma}
\label{lem:char_redundant}
A group $g\in\G$ is not redundant if and only if there exists $\alphab\in\RR^p$ such that $\nm{\alpha_g}>d_g$ and $\forall h \in \G \backslash \{g\}, \, \nm{\alpha_h} \leq d_h$.
\end{lemma}
\begin{proof}
Define the unit balls:
$\mathcal{U} =\{\alphab \in \RR^p \mid \forall h \in \G, \, \nm{\alphab_h} \leq d_h\}$
and $\mathcal{U}_g =\{\alphab \in \RR^p \mid \forall h \in \G \backslash \{g\}, \, \nm{\alphab_h} \leq d_h\}$. We have that $g$ is redundant for $\Oo$ if and only if it is redundant for $\Oo^*$, and the latter is true if and only if $\mathcal{U}=\mathcal{U}_g$. Since $\mathcal{U} \subset \mathcal{U}_g$, $g$ is not redundant if and only if there exists $\alphab \in \mathcal{U}_g \backslash \mathcal{U}$.
\end{proof}

\begin{corollary}
Let $g\in\G$ and $\H \subset \G$ such that $g$ is covered by groups in $\H$, \ie, $g \subset \cup_{h \in \H} \, h$.
Then $g$ is redundant if
$\displaystyle d_g^2 > \sum_{h \in \H} d_h^2.$
\end{corollary}
\begin{proof}
The fact that $g$ is covered by groups in $\H$ implies that, for any $\alphab\in\RR^p, \nm{\alphab_g}^2 \leq \sum_{h \in \H} \nm{\alphab_h}^2$. If $g$ is part of the group-support, then necessarily $d_g^2=\nm{\alphab_g}^2 \leq \sum_{h \in \H} \nm{\alphab_h}^2 \leq \sum_{h \in \H} d_h^2.$ 
\end{proof}
In particular, if all singletons are part of $\G$ with $d_{\{i\}}=1, \: i \in [1,p]$, this imposes $d_g \leq \sqrt{|g|}$.\\

In the case where the weights depend only on the cardinality of the  $g$, i.e., $d_g=d_{k}$ for $|g|=k$, we consider the following condition: 

\begin{equation}
\label{eq:suff_cond}
\tag{C}
\forall k>1\,, \quad d_{k-1} < d_k <\sqrt{\frac{k}{k-1}} \: d_{k-1}\,.
\end{equation}

\begin{lemma}
Condition (\ref{eq:suff_cond}) is sufficient to  guarantee that no group is redundant.
\end{lemma}
\begin{proof}
Assume that $(d_i)_{1\leq i \leq m}$ satisfy condition (\ref{eq:suff_cond}), and let $g\in\G$ a group of cardinality $k$. Consider the vector $\alphab=\frac{d_k}{\sqrt{k}} \mathbf{1_g}$ with $\mathbf{1_g} \in \RR^p$ the vector with entry $i$ equal to $1$ for $i \in g$ and $0$ else. Since $|g|=k$ we have $\nm{\alphab_g}=d_k$.
Note that  (\ref{eq:suff_cond}) implies $\frac{d_k}{\sqrt{k}} < \frac{d_{k-1}}{\sqrt{k-1}}$, which more generally implies by induction $\frac{d_k}{\sqrt{k}} < \frac{d_j}{\sqrt{j}}$ for any $j<k$. Now, for any group $g' \in\G$ of cardinality $j<k$, we have $\nm{\alphab_{g'}} \leq \frac{d_k}{\sqrt{k}}\sqrt{j} < d_j$. Similarly, if $\abs{g'}=j>k$ then $\nm{\alphab_{g'}} \leq \nm{\alphab_g} = d_k < d_j$, and if $g'\neq g$ but $\abs{g'}=\abs{g}$, then $\nm{\alphab_{g'}} < \nm{\alphab_g} = d_k = d_{g'}$. Since $\nm{\alphab_g} = d_g$ and $\nm{\alphab_{g'}} < d_{g'}$ for $g'\neq g$, it is possible to choose $\epsilon>0$ sufficiently small such that the vector $\alphab' = \alphab+\varepsilon \mathbf{1_g}$ satisfies $\nm{\alphab_g'} > d_g$ and $\nm{\alphab_{g'}'} < d_{g'}$ for any $g'\neq g$. \lemref{lem:char_redundant} then shows that $g$ is not redundant.
\end{proof}

We would like insist that condition \eqref{eq:suff_cond} is sufficient to guarantee non-redundancy but might be unnecessary for many restricted families of groups, for example as soon as each group contains an element which belongs to no other group. 
However, without any condition on the set of groups, the previous condition is the weakest possible if the weights depend only on the group sizes, since it becomes necessary in the following special case:
\begin{lemma}
Assume that group $g$ with cardinality $|g|=k$ contains all $k$ groups of size $k-1$,
then (\ref{eq:suff_cond}) is necessary for $g$ to be non-redundant.
\end{lemma}
\begin{proof}
If $g\in\G$ is not redundant, by \lemref{lem:char_redundant} we can find $\alphab\in\RR^p$ such that $\nm{\alphab_g} > d_g$ and $\nm{\alphab_h}\leq d_h$ for $h\in\G \backslash \{g\}$. In particular, for all $i \in g$, $\nm{\alphab_{g\backslash \{i\}}}^2 \leq d_{k-1}^2$ so that
$(k-1) d_k^2 < (k-1)\nm{\alphab_g}^2=\sum_{i \in g}\nm{\alphab_{g\backslash \{i\}}}^2 \leq k \, d_{k-1}^2$ which shows the result.
\end{proof}

Condition \eqref{eq:suff_cond} allows scalings of the weights which go
from quasi uniform weights, in which case the larger groups dominate
the smaller groups in the sense that they are preferably selected, to
weights that scale like $\sqrt{k}$, in which case the smaller group
dominate (and in particular if the singletons are included the norm
approaches the $\ell_1$-norm). Condition (\ref{eq:suff_cond}) suggests
to consider weights of the form $d_k=k^\gamma, \:\: \gamma
\in (0,\frac{1}{2})$. 
We illustrate on \figref{fig:gamma_balls} the trade-offs obtained with the groups $\G=\{\{1\},\{2\},\{3\},\{1,2,3\}\}$ and different $\gamma$. The first ball for $\gamma = 0$ is the
ball we would have without considering the singletons since only the
largest group is active. At the other extreme for $\gamma=\frac{1}{2}$
the ball is the one we would have without the $\{1,2,3\}$ group since
only the singletons are active. In intermediate regimes, all the
groups are active in some region. More specifically, the second ball
for $\gamma=\frac{1}{4}$ corresponds to a limit case that we present
in~\secref{sec:fdr}, while the third one for
$\gamma=\frac{\log(2)}{2\log(3)}$ illustrate another problem that we
now introduce~: the possibility that a group \emph{dominates} other
groups. Intuitively for $\gamma \geq \frac{\log(2)}{2\log(3)}$,
\ie, if the sphere gets any smaller than on the third ball, it
becomes impossible to select a support of exactly two covariates even
though (i) such a support would be a union of groups and (ii) no group
is redundant. We detail this notion in the next section.
\begin{figure}
  \begin{center}
      \includegraphics[width=.4\linewidth]{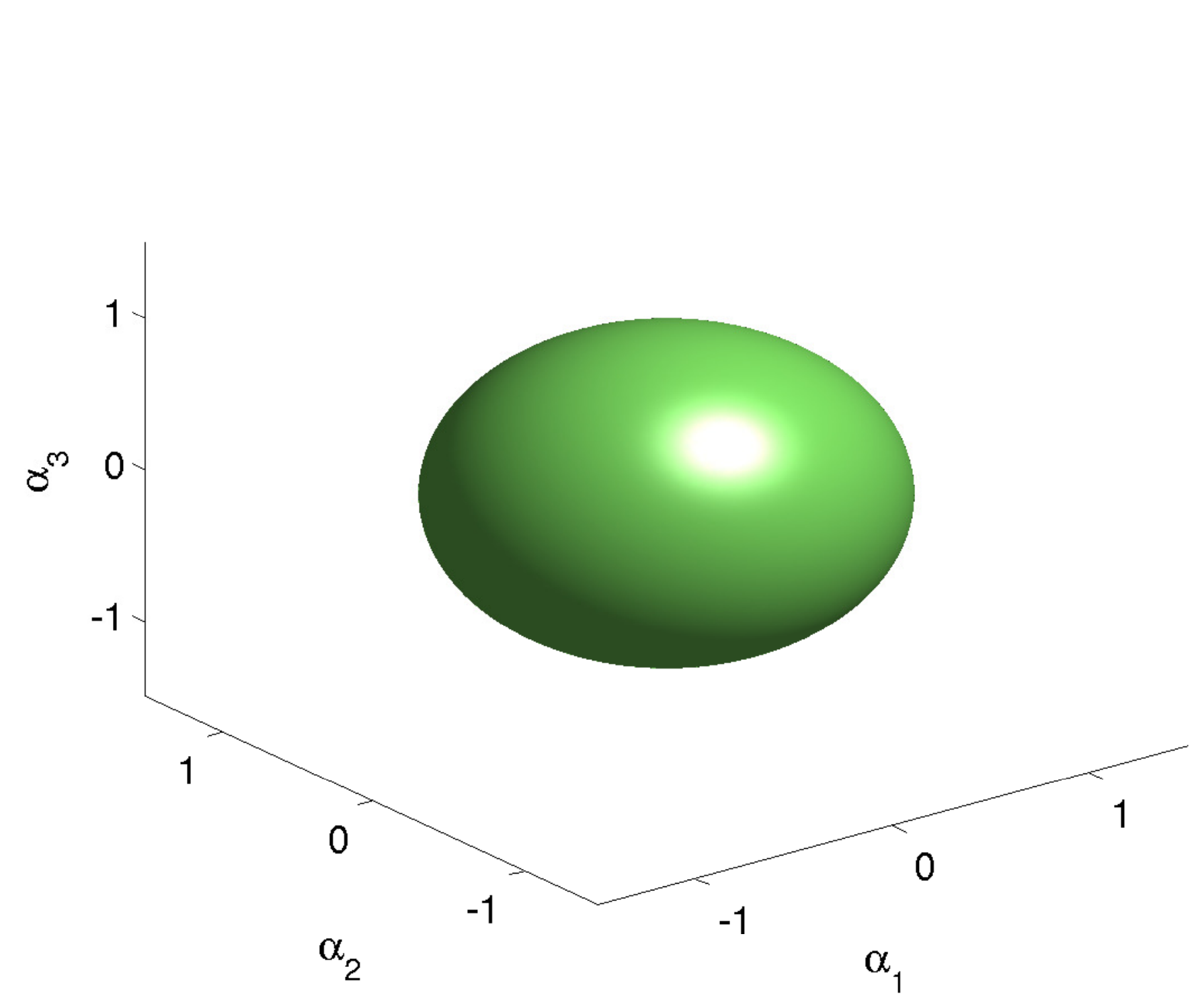}
      \includegraphics[width=.4\linewidth]{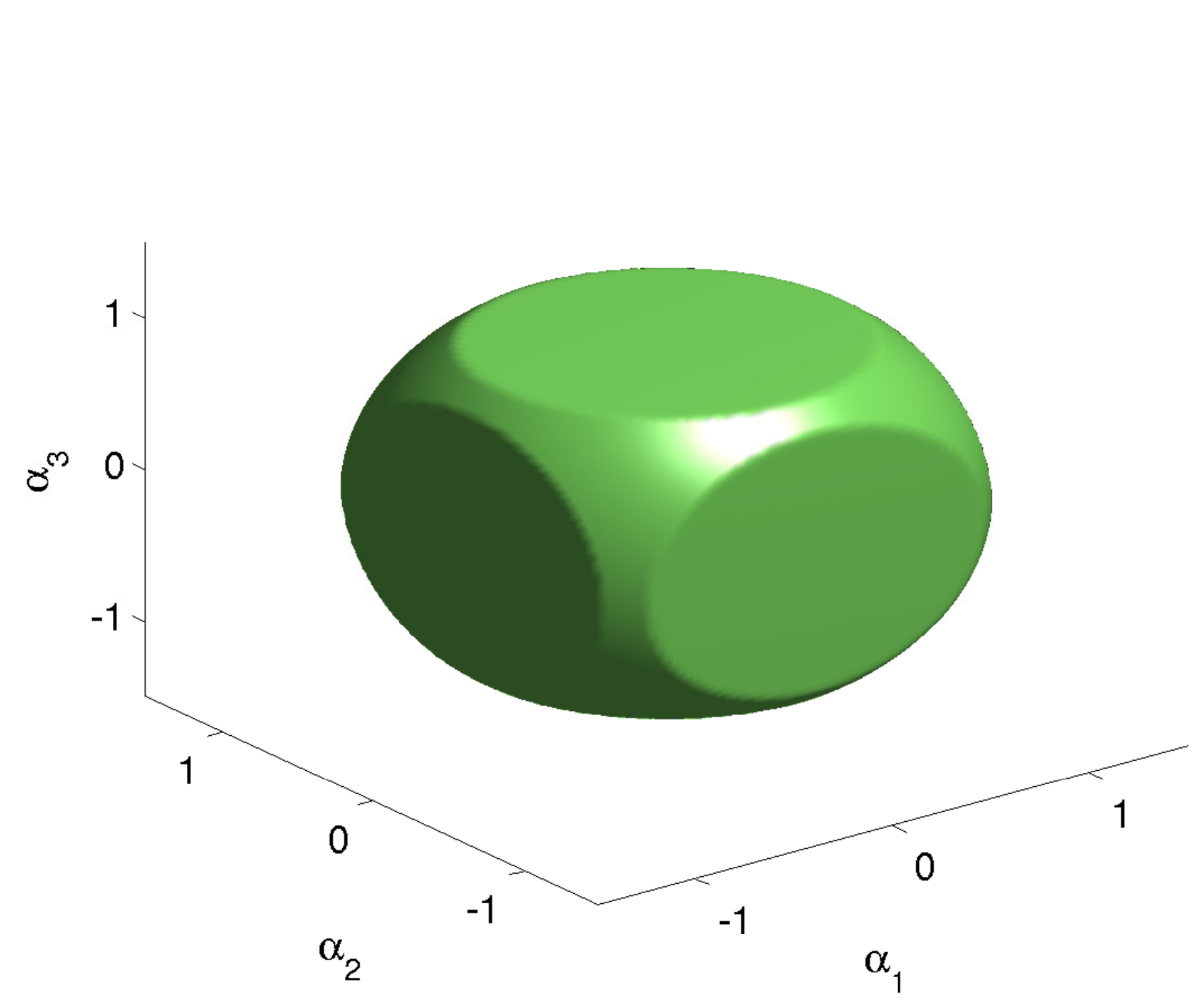}
      \includegraphics[width=.4\linewidth]{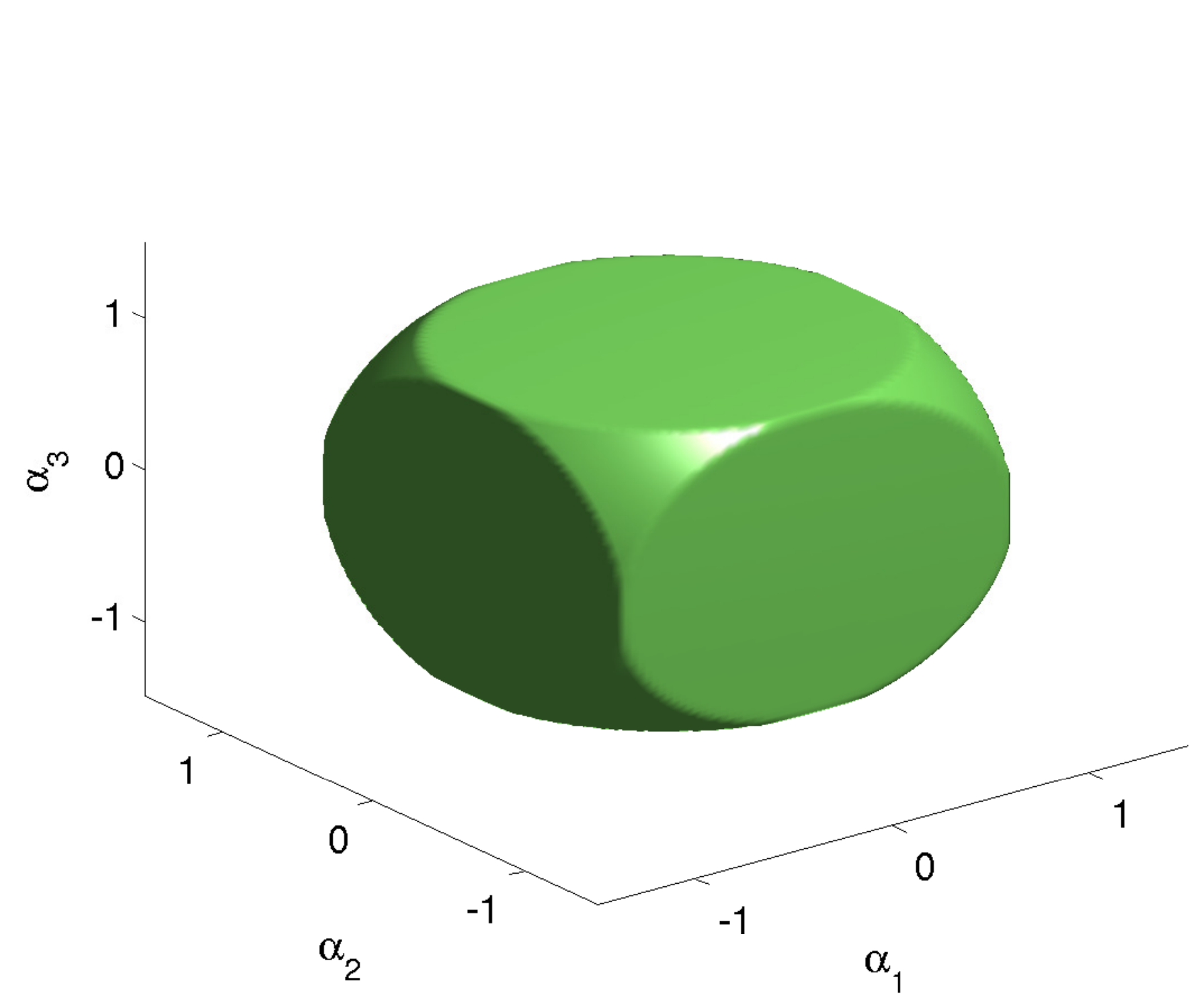}
      \includegraphics[width=.4\linewidth]{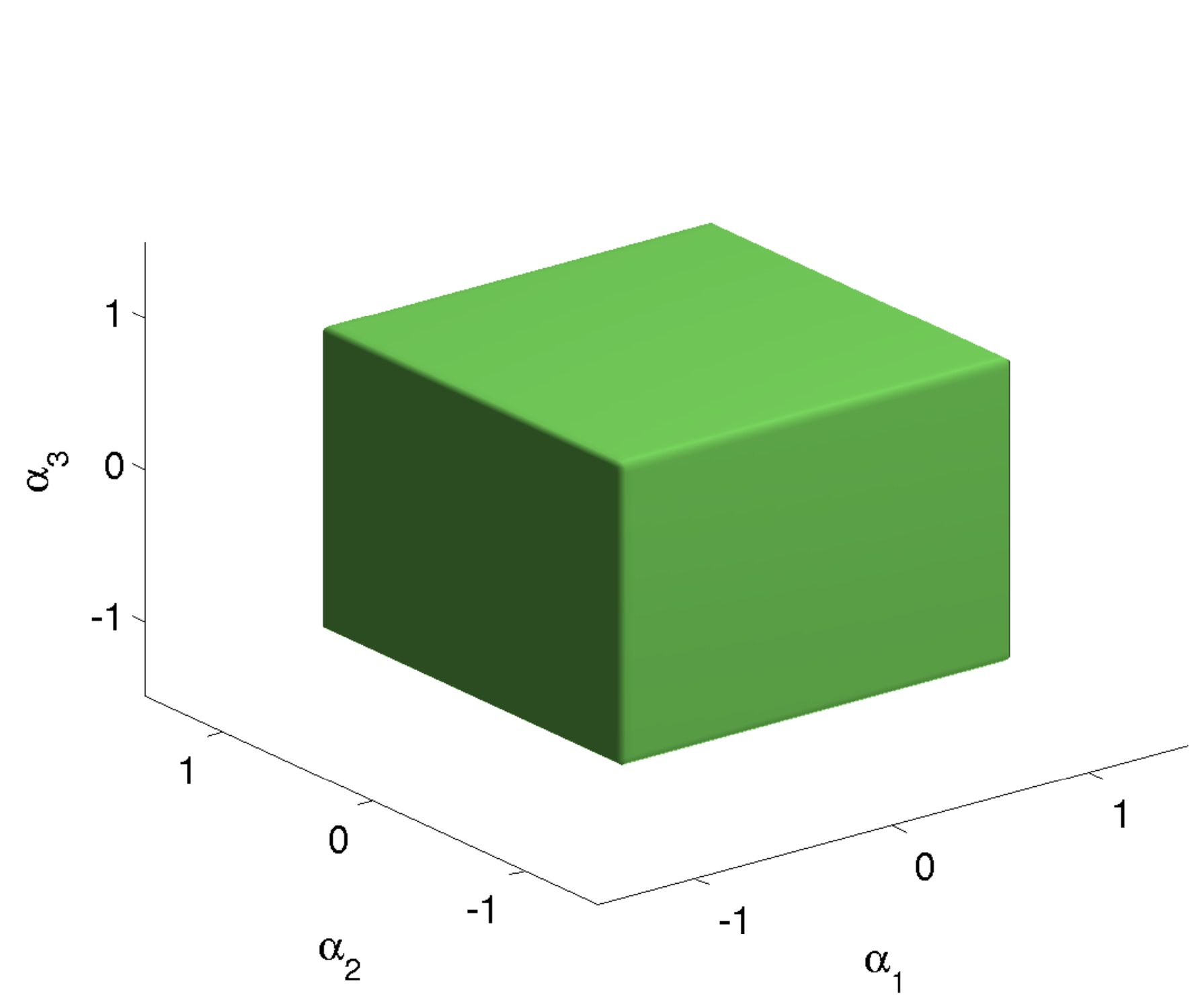}
  \end{center}
  \caption{Balls for $\Oo^*$ for the groups
    $\G=\{\{1\},\{2\},\{3\},\{1,2,3\}\}$ with $\gamma =
    0,\frac{1}{4},\frac{\log(2)}{2\log(3)},\frac{1}{2}$ from top to
    bottom, left to right.}
\label{fig:gamma_balls}
\end{figure}

\subsection{Dominating group}
\label{sec:dominating}
Let us first formalize the notion of group domination.
\begin{definition}\label{def:domination}
Let $g\in\G$  and $\H\subset\G$ a set of subgroups satisfying $\forall h \in\H, h\subset g$. We say that $g$ \textit{dominates} $\H$ if $\H$ could be the weak group-support for some $\w$ if $g$ was removed from $\G$, but is the weak group support of no $\w$ in the presence of $g$.
\end{definition}

We can characterize the presence of domination in terms of weights as follows:
\begin{lemma}
\label{lem:domin}
A group $g$ dominates a \textit{set of subgroups} $\H$ if and only if, on the one hand, $\H$ is a possible group-support when $g$ is removed from $\G$, and, on the other,  
$$
d_g < P(g,\H) \eqdef \min\cbr{ \nm{\alphab_g} \,|\, \alphab \in \RR^p \quad and \quad  \nm{\alphab_h} =d_h, \: \forall h \in \H}.$$ 
\end{lemma}
\begin{proof} 
  First note that the set of constraints $\nm{\alphab_h} =d_h, \:
  \forall h \in \H$ is feasible since $\H$ is assumed to be a possible
  group support without $g$. Then note that the condition is
  equivalent to saying that the ball $\{ \alphab_g \in \RR^{|g|} \mid
  \nm{\alphab_g} \leq d_g \}$ does not intersect the previous feasible
  set, which characterizes the set of possible dual variables for
  which the weak group-support is $\H$.
\end{proof}

As discussed previously, one natural property to require would be that
if $\w$ is exactly supported by a group $g$, its group-support should
be $g$. As argued in Lemma~\ref{lem:impossible}, we can not have this
property in general.
We can however show that if the support of $\w$ is a single group in
$\G$, then this group is always in the group support of $\w$.

The following result shows that, under some conditions on the weights, we can ensure that a group $g$ does not dominate any set of subgroups that do not cover it entirely.
\begin{lemma}\label{lem:nodomination}
Let a group $g\in\G$ and a set of subgroups $\H\subset\G$ such that $\forall h \in\H, h \subset g$ and $\cup_{h\in\h}h \subsetneq g$. Assuming that $\H$ could be in the group support of some $\w$ if $g$ was removed from $\G$, then $g$ does not dominate $\H$ if, for some constant $d_1>0$, weights satisfy $d_h \leq \sqrt{|h|} d_1$ for all $h \in \H$ and $d_g \geq \sqrt{|g|-1} \, d_1$.
\end{lemma}
\begin{proof}
By \lemref{lem:domin}, $g$ does not dominate $\H$ if and only if $d_g \geq P(g,\H)$. To prove this, let us rewrite $P(g,\H)$ as the solution of the following optimization problem:
$$\min_{\x \in \RR^p_+} \x^\top \mathbf{1}_g \quad \st \quad \forall h \in \H, \:\x^\top \mathbf{1}_h=d^2_h\,.$$
By strong duality of linear programs $P(g,\H)$ is also the solution of the dual problem:
$$
\max_{\mathbf{u} \in \RR^{|\H|}} \sum_{h \in \H} u_h d_h^2 \quad \st \quad \forall i \in [1,p],\: \sum_{h \in \H} u_h 1_{\{i \in h\}} \leq 1_{\cbr{i\in g}}\,.
$$
But if $\bar{h}\eqdef\cup_{h \in \H} \, h$, under the conditions on the weights in \lemref{lem:nodomination}, we can upper bound the optimal value as follows:
$$ \sum_{h \in \H} u_h d_h^2 \leq d_1^2 \sum_{h \in \H} u_h |h|= d_1^2 \sum_{i \in g} \sum_{h \in \H} u_h 1_{\{i \in h\}} \leq d_1^2 |g \cap \bar{h}| \leq d_1^2 \br{|g|-1}\,,
$$
where the second inequality results from the constraints of the dual program and the fact that for $i \in g \backslash \bar{h}$, the corresponding terms in the sum are equal to $0$.
This shows that if $d_g^2\geq(|g|-1) \, d_1^2$, then $d_g \geq P(g,\H)$. 
\end{proof}
Note that \lemref{lem:domin} is tight in the following case:
\begin{lemma}
\label{lem:nec_dom_sing}
For any group $g \in \G$, if $\H$ is a set of $|g|-1$ singletons of $g$, each with weight $d_1$, that could be in a group support if $g$ was removed, then $g$ dominates $\H$ if and only if $d_{g} < d_1 \sqrt{|g|-1}$.
\end{lemma}
\begin{proof}
This is a direct consequence of \lemref{lem:domin}, where the value of $P(g,\H)$ is trivially equal to $d_1\sqrt{|g|-1}$.
\end{proof}

What the two previous lemmata indicate is that, if there are large
gaps in size between a group of size $k$ and many much smaller
subgroups contained in it, it is necessary to choose a value for the
weight which is possibly unreasonably large, to allow all combinations
of subgroups to be selected (even \textit{non-covering}
ones). \lemref{lem:nec_dom_sing} is illustrated on
Figure~\ref{fig:gamma_balls}, with the the group
$\G=\{\{1\},\{2\},\{3\},\{1,2,3\}\}$. Giving singletons the weight
$d_1=1$, the critical weight for $g=\cbr{1,2,3}$ to dominate or not
pairs of singletons is $d_g = \sqrt{|g|-1} = \sqrt{2}$. We represent
it equivalently as $d_g=|g|^\gamma$ with $\gamma=\frac{\log(2)}{2
  \log(3)}$ on \figref{fig:gamma_balls}. This corresponds to the
critical value, below which it is not possible to select two singletons
only .  The trade-off we are facing here is not surprising when the
weights are thought to correspond to \textit{code lengths}. Indeed, in
light of the interpretation of the norm $\Omega$ as a relaxation of a
\textit{block coding} penalization, it is clear that allowing groups
with quite large weights (i.e., code lengths) increases the
expressiveness of the code at the expense of compressibility and
reduces the strength of the prior on support, since large weight
allows for a greater diversity of supports.  Put more simply, there is
a trade-off between how coarsely the supports are encoded and how
informative the prior on the supports is.  The trade-off can also be
interpreted as a bias-variance trade-off, where biasing the estimate
of the support with a coarser set of patterns reduces the variance in
its estimation. 

It should be noted that, as an important consequence of domination, the
set of possible sparsity patterns (although consisting of unions of sets of $\G$) is in general \emph{not} stable by union. 

\subsection{Importance of weights for support consistency, FDR and FWER control}
\label{sec:fdr}

In this section we consider the following regression setting:
\begin{equation}
\label{eq:proxop}
\min_{\w\in\RR^p} \frac{1}{2}\nm{\w-\w^*+\epsilon}^2 + \lambda \Oo(\w)\,,
\end{equation}
where the design matrix is taken to be the identity and the noise to be Gaussian, bearing in mind that the analysis we propose here could be extended easily to the case of a design satisfying properties such as RIP with noise that could be taken more generally subgaussian. The mapping to the solution of this optimization problem is often called  the soft-thresholding operator, shrinkage operator or proximal operator associated with the norm $\Oo$. We denote this mapping $\w \mapsto \text{ST}(\w)$. In terms of support recovery and group-support consistency, a reasonable minimal requirement is that for sufficiently large values of the coefficients and for small levels of noise, assuming that the distribution of the noise is absolutely continuous with respect to the Lebesgue measure, the solution to problem (\ref{eq:proxop}) should retrieve the correct support, provided the latter can be expressed as a union of groups.

We first show that redundant groups may never be selected by \eqref{eq:proxop}.
\begin{lemma}
\label{lem:weights_incr_regression}
Take $\G=\cbr{g,g'}$ with $g \subsetneq g'$ and $d_g \geq d_{g'}$. Then for any $\w$, $g \notin \Gw(\wh)$ a.s. where $\wh=\prox(\w)$.
\end{lemma}
\begin{proof} We first note that the optimality condition for \eqref{eq:proxop} is
\begin{equation}\label{eq:kktproximal}
\wh - \w^* + \epsilon = -\lambda\alphab\,,
\end{equation}
where $\alphab\in\A(\wh)$. We then reason by contradiction and assume $g \in \Gw(\wh)$ so that $\nm{\alphab_g(\wh)}=d_g$. Then, because $g \subsetneq g'$, $\nm{\alphab_{g'}(\wh)}^2=\nm{\alphab_g(\wh)}^2+\nm{\alphab_{g'\backslash g}}^2 \leq d_{g'}$, which implies $\alphab_{g'\backslash g}= 0 =\w_{g'\backslash g}+\epsilon_{g'\backslash g}-\wh_{g'\backslash g}$. But $\w_{g'\backslash g}+\epsilon_{g'\backslash g} \neq \zv$ a.s., this implies $\wh_{g'\backslash g} \neq \zv$, and therefore that $\v^{g'}\neq \zv$. But $\v^{g'}$ restricted to $g'\backslash g$ should then both be equal to $\zv$ by optimality condition, and be equal to $\wh_{g'\backslash g}$, which is a contradiction.
\end{proof}
\lemref{lem:weights_incr_regression} should be compared to \lemref{lem:weights_incr}. While the later one shows that $g$ can not be selected without $g'$, \lemref{lem:weights_incr_regression} shows that in the regression setting it may simply not be selected a.s. This shows in particular that $d_g \geq d_{g'}$ can pose a problem of support consistency because it implies that,
if the only way to write the support as a union of elements of $\G$ is $\supp{\w}=g$, the support is a.s. never correctly estimated by solving problem (\ref{eq:proxop}).

We now discuss in more details the influence of the weights on the probability to select false positives (\secref{sec:falsepositives}) and to have false negatives (\secref{sec:falsenegatives})

\subsubsection{False positives}\label{sec:falsepositives}
Let us consider a group $g\in\G$ of size $|g|=k$ which is outside of the support (i.e. $\w^*_g=\zv$), and such that not other group intersecting it is selected. From the optimality condition \eqref{eq:kktproximal}we see that $\wh_g = \zv$ if and only if $\displaystyle \nm{\epsilon_g}^2 \leq \lambda^2 d_k^2.$

If we assume that $\lambda=\sigma$, then setting 
\begin{equation}\label{eq:dk}
d_k=\sqrt{k+c\sqrt{k}}
\end{equation}
is an interesting choice because this is, at second order, the smallest possible rate that ensures that each group has a vanishingly small probability of being selected by chance.
Indeed, on the one hand, $\nm{\epsilon_g}^2 \sim \sigma^2\chi^2_k$ so the usual Chernoff bound yields:
 $$
 \PPP(\nm{\epsilon_g}^2 \geq t k \sigma^2) \leq e^{-\frac{k}{2}(t-\log(t)-1)}\,,
 $$
 and it is easy to verify that for $t=1+\frac{c'}{\sqrt{k}}$, with $c$
 sufficiently large, the above probability can be made arbitrarily
 small uniformly in $k$. This implies that if $d_k$ is fixed according
 to \eqref{eq:dk}, then with $c$ large enough we can make the
 probability that $g$ is selected as small as possible. On the other
 hand, choosing $d_k$ smaller, \ie, $d_k^2-k=o(\sqrt{k})$, would fail
 to guarantee $\PPP(\nm{\epsilon}^2 \leq \sigma^2 d_k^2)>1-\eta$ for
 $k$ large because the central limit theorem implies that
 $\frac{X-k}{\sqrt{2k}} \overset{d}{\rightarrow} \mathcal{N}(0,1)$. In
 summary, \eqref{eq:dk} is the smallest rate which ensures that we can
 control the probability of selecting a wrong group uniformly in
 $k$. Finally, note that for $d_k=\sqrt{k+c\sqrt{k}}$, condition
 (\ref{eq:suff_cond}) is satisfied; furthermore, we have $c
 k^{\frac{1}{4}} \leq d_k \leq (1+c) k^{\frac{1}{2}}$.  In particular
 if we consider the case of $c\rightarrow \infty$, we retrieve a
 scaling of the form $d_k=k^{\frac{1}{4}}$.

 Note that if we want to control the expected number of incorrectly
 selected variables instead of the number of incorrectly selected
 groups, then, using the same reasoning, but based on a bound on the
 expected number of false positive of the form $\sum_{g \in G} |g| \,
 \PPP(X_g \geq t |g|)$ we would show similarly that an appropriate
 choice for $d_k$ is $d_k=\sqrt{k+c\sqrt{k \log k}}$. Obtaining a
 control of the type FWER instead of FDR is possible by choosing $c
 \propto \sqrt{\log \m}$.  The reader probably noticed that the
 analysis in this section is ignoring the overlaps between groups, and
 for groups that have a quite significant overlap with a group of the
 support, the probability of being incorrectly selected is much
 larger. This issue can however be addressed by choosing $c$
 sufficiently large. Besides this point, the weights derived
 nonetheless satisfy constraints from the previous sections in which
 issues arising from overlaps were considered.

\subsubsection{False negatives}\label{sec:falsenegatives}
These choices for $d_k$ allow to control for false positives, but it is interesting as well to ask which groups containing
true non-zero elements will be selected, and which ones could be false negatives. 
For simplicity we assume that $\w^*_i \in \{0,1\}$ and that the noise is Gaussian as previously. If the fraction of non-zero elements in $\w^*_i$ is $p$ and one assumes a null model $H_0$ under which group $g$ is unrelated to the nonzero pattern of $\w^*$ then it is reasonable to model the number of non-zero elements in $g$ as a binomial random variable $\mathcal{B}in(k,p)$ with $k=|g|$.
Using again the KKT conditions, if none of the groups intersecting $g$ is selected, we will have $\v_g=\zv$ if and only if $\nm{\w^*_g+\epsilon_g}^2\leq \lambda^2 d_k^2$.

Since $\nm{\w^*_g}^2\sim\mathcal{B}in(k,p)$ and $\nm{\epsilon_g}^2 \sim \sigma^2 \chi^2_k$, we have $\EE[\nm{\w^*_g+\epsilon_g}^2]=kp+k\sigma^2$ and 
$$\quad \text{Var}(\nm{\w^*_g+\epsilon_g}^2)=\text{Var}(\nm{\w^*_g}^2)+\EE[(\epsilon_g^\top\w^*_g)^2]+\text{Var}(\nm{\epsilon_g}^2)=kp(1-p)+4 k p \sigma^2+2k\sigma^4.$$ 

If $\lambda^2=p+\sigma^2$ and if $d_k$ is chosen of the previous form
 $d_k=\sqrt{k+c\sqrt{k}}$, then, for an appropriate choice of $c$, namely $c=c'\,\frac{\sqrt{p(1-p)+4p\sigma^2+2\sigma^4}}{p+\sigma^2}$, classical Chernoff bounds together with an analysis similar to that of the previous section shows that we have $\nm{\w^*_g+\epsilon_g}^2 > \lambda^2 d_k^2$ with probability decreasing exponentially in $c'$. Therefore in this model,
groups selected can be interpreted as groups that are ``enriched" in non-zero coefficients, where we call a group enriched if the number of non-zero coefficients in that group is significantly larger than for a random group of the same size.
To put things differently the false negatives correspond to groups that do not have a significant number of non-zero elements.

This property is certainly a feature that can be desirable, especially in the applications in genomics that we have in mind where it is common to test for biological processes (or other groups of genes) that are enriched in ``active genes". 

Note that if a group $g$ has elements in common with another selected group $g'$, the elements that are in $g'$ are explained in part by $g'$ and are therefore ``discounted" for group $g$, in the sense that we only need
$$\big\| \w^*_g-\sum_{g' \cap g \neq \varnothing}\v^{g'}_g+\epsilon_g \big\|^2\leq \lambda^2 d_k^2.$$ A group is therefore selected if it contains enough non zero components that it itself explains.

It should be stressed that the previous analysis depends on the assumption that the components of $\w^*$ are of the same order of magnitude and fails if the distribution of the entries of $\w^*$ has a long tail. 

Finally, the analysis presented in these last two sections is heuristic is nature. It is by no means aimed at proving that a specific weighting scheme can be chosen universally for all possible collections of groups $\G$, but rather solely motivated by the need for an initial set of criteria to guide this choice.  It is likely that finer analyses, namely under high-dimensional scaling and dedicated to specific collections of groups are required to make more definite recommendations for the choice of the weights.
It should be noted that a different view on the weights can be adopted by considering them as defined through a set function; this is the point of view adopted in \citet{Obozinski2011Convex} which relates the behavior of $\Oo$ to the set-function.

%% file: graph_lasso.tex
\label{sec:graphlasso}

We now consider the situation where we have a simple undirected graph
$(I,E)$, where the set of vertices $I=[1,k]$ is the set of covariates
and $E\subset I\times I$ is a set of edges that connect covariates. We
suppose that we wish to estimate a sparse model such that selected
covariates tend to be connected to each other, \emph{i.e.}, form a
limited number of connected components on the graph. An obvious
approach is to use the norm $\Omn$ where $\G$ is a set that
generates connected components by union. For example, we may consider
for $\G$ the set of edges, cliques, or small linear subgraphs. As an
example, considering all edges, \emph{i.e.}, $\G=E$ leads to~:
\[
\Omega_{\text{graph}}(\w)=\min_{v \in \VV_E} \sum_{e \in E} d_e\|v_e\|
\quad \text{s.t.}\: \sum_{e \in E} v_e=w, \: \supp{v_e}=e\,.
\]

\begin{figure}
  \centering
  \includegraphics[width=.3\linewidth]{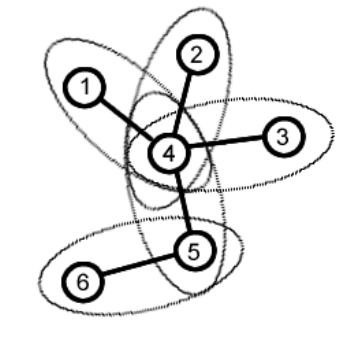}
  \caption{Graph-Lasso~: if the penalty leads to the selection of
    connected sets of covariates like the edges, the resulting pattern
    should be more connected on the graph.}
  \label{fig:graph-lasso}
\end{figure}

Alternatively, we will consider in the experiments the set of all
linear subgraphs of length $k\geq 1$. Although we have no formal
statement on how to chose $k$, it intuitively controls the size of the
groups of connected variables which are selected, and should therefore
be typically chosen to be slightly smaller than the size of the
minimal connected component expected in the support of the model.

%% file: exps.tex
\label{sec:experiments}

To assess the performance of our method when either overlapping groups or a graph are
provided as a priori information, and subsequently, to assess the influence of the weights $d_g$, we considered several synthetic examples of regression model in which the structure of the model generating the data matches the prior on supports induced by the norm.

\subsection{Synthetic data: given overlapping groups}
In this experiment, we simulated data with $p=82$ variables, covered
by $10$ groups of $10$ variables with $2$ variables of overlap between
two successive groups: 
$$\G=\big \{ \{1,\ldots,10\}, \{9,\ldots,18\}, \ldots, \{73,\ldots,82\} \big \}.$$ 
We chose the support of $\w$
to be the union of groups $4$ and $5$ and sampled both the
coefficients on the support and the offset from i.i.d. Gaussian
variables. Note that in this setting, the support can be expressed as
a union of groups, but not as the complement of a union. Therefore,
our latent group Lasso penalty $\Omn$ could recover the right
support.

The model is learned from $n$ data points $(\x_i,y_i)$, with $y_i =
\w^\top \x_i + \E$, $\E \sim \mathcal{N}(0,\sigma^2)$,
$\sigma=|\mathbb{E}(\X\w + b)|$.  Using an $\ltwo$ loss $L(\w) = \|\y -
\X\w - b\|^2$, we learn models from $100$ such training sets. 
\begin{figure}[ht]

\begin{center}
\begin{tabular}{rl}
  \includegraphics[width=.5\linewidth]{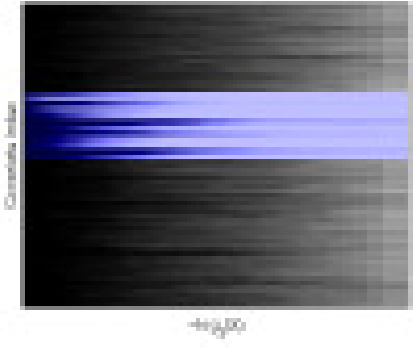} \hspace*{-3mm}&\hspace*{-3mm}
  \includegraphics[width=.5\linewidth]{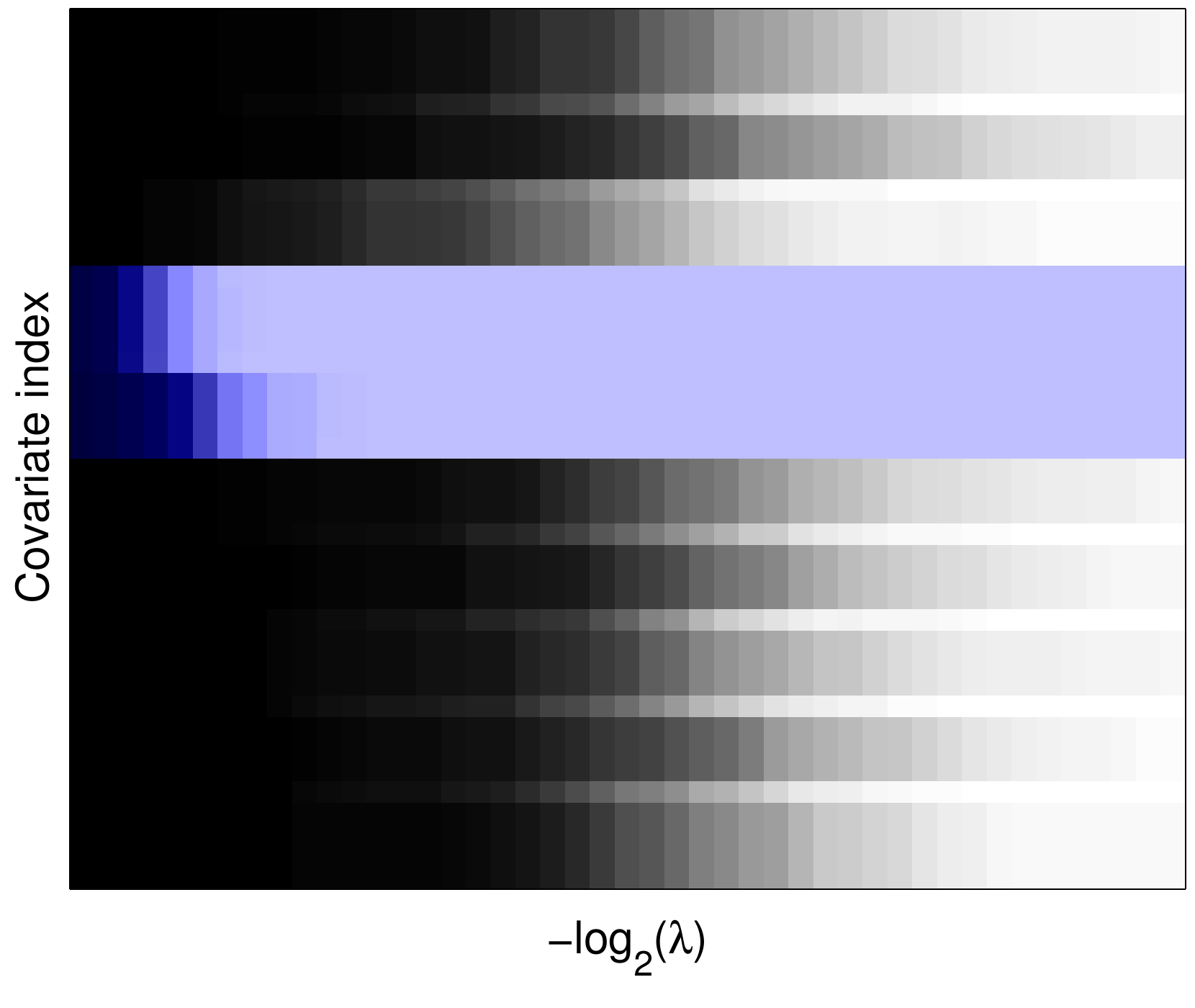} \\
  \includegraphics[width=.5\linewidth]{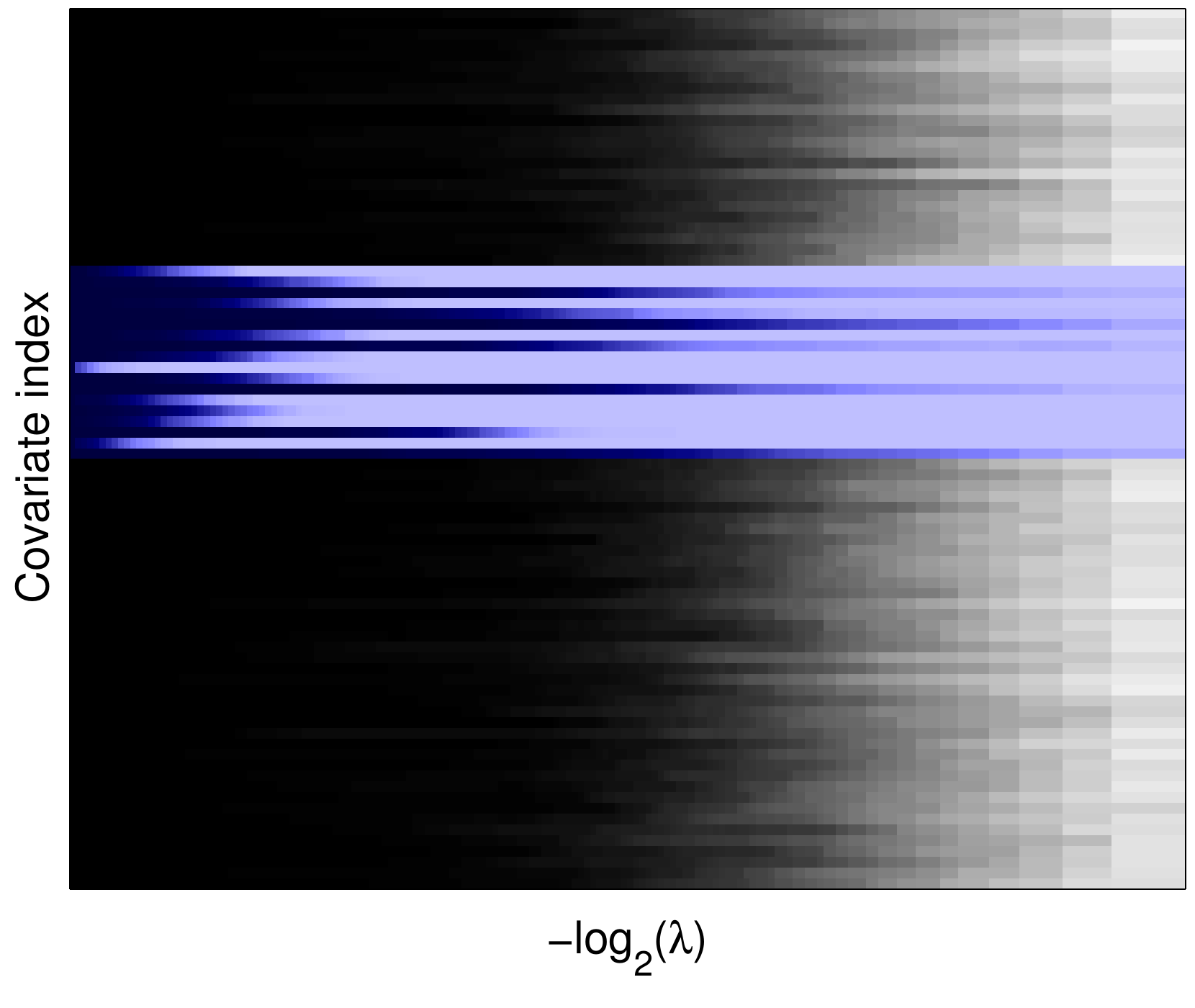} \hspace*{-3mm}&\hspace*{-3mm}
  \includegraphics[width=.5\linewidth]{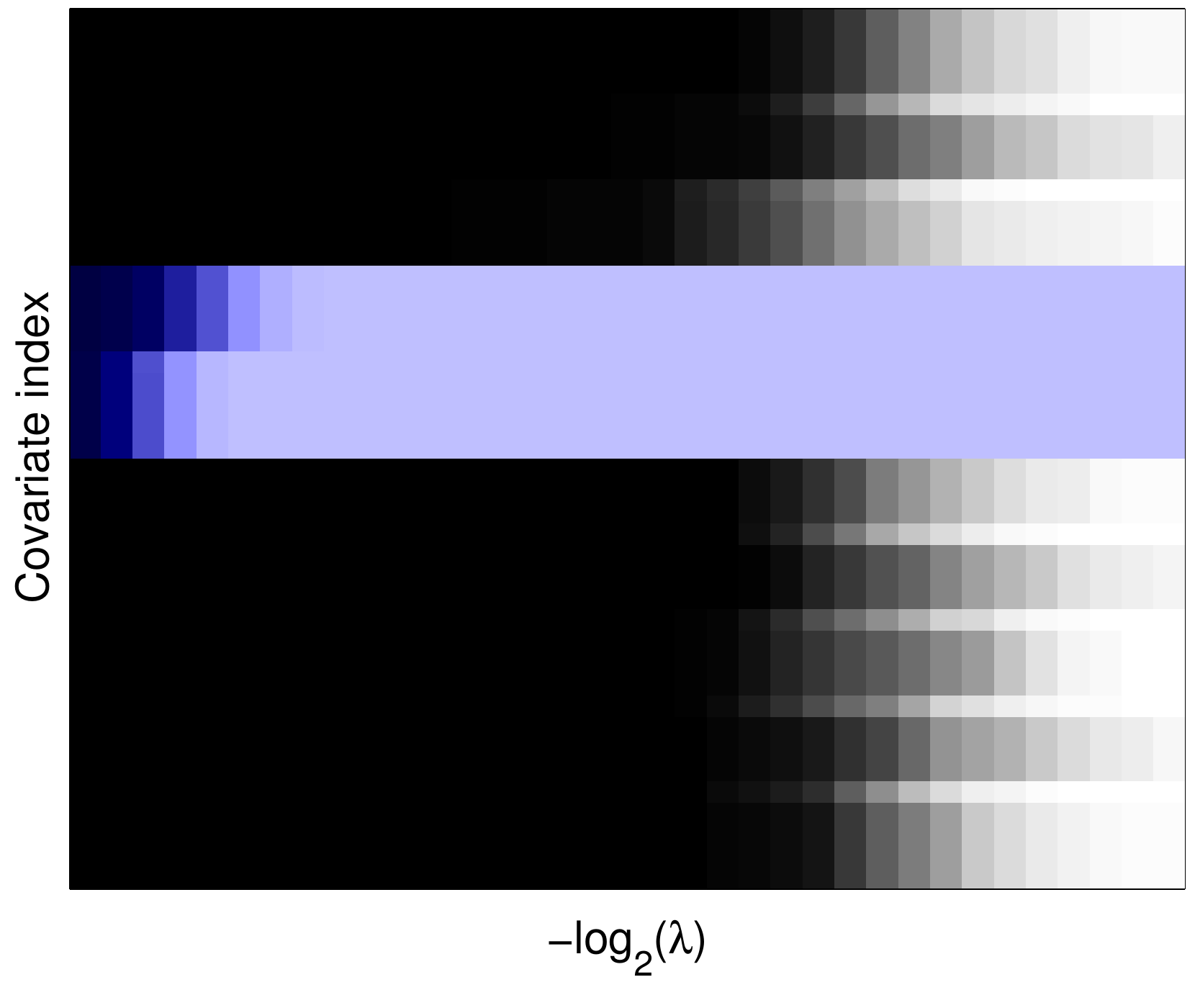} \\
\end{tabular}
\end{center}

\caption{Frequency of selection of each variable with the Lasso (left)
  and $\Omn$ (right) for $n=50$ (top) and $100$ (bottom).
For each variable index (on the y-axis),
its frequency of selection is represented in levels of gray as a function of the
regularization parameter $\lambda$ (on the x-axis), both for the Lasso penalty and
$\Omn$. The transparent blue band superimposed indicates the set of covariates that belong to the support.}
\label{fig:Ws}

\end{figure}

We report the empirical frequencies of the selection of each variable on Figure~\ref{fig:Ws}.
For any choice of $\lambda$, the Lasso frequently misses some variables
from the support, while $\Omn$ does not miss any variable from the
support on a large part of the regularization path. Besides, we
observe that over the replicates, the Lasso never selects the exact
correct pattern for $n<100$. For $n=100$, the right pattern is
selected with low frequency on a small part of the regularization
path. $\Omn$ on the other hand selects it up to $92\%$ of the times
for $n=50$ and more than $99\%$ on more than one third of the path for
$n=100$.  

\begin{figure}[ht]
\begin{center}
  \includegraphics[width=.65\columnwidth]{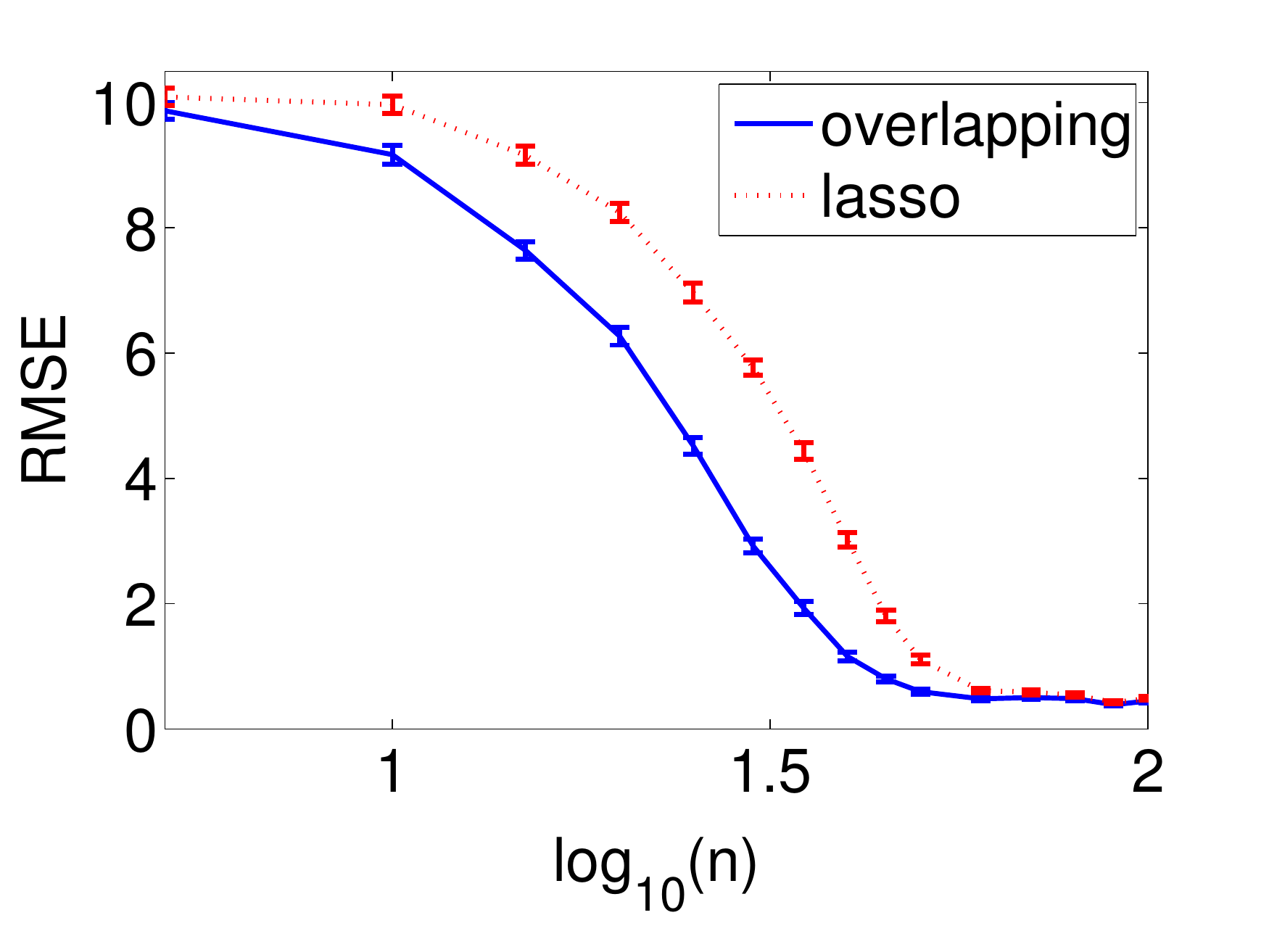}
\end{center}

\caption{Root mean squared error of overlapped group lasso and Lasso
  as a function of the number of training points.}
\label{fig:perf}

\end{figure}

Figure~\ref{fig:perf} shows the root mean squared error for both
methods and several values of $n$. For both methods, the full regularization
path is computed and tested on three replicates of $n$ training and
$100$ testing points. We selected the best parameter in average and used it
to train and test a model on a fourth replicate.  For a large range of
$n$, 
$\Omn$ not only helps to recover the right pattern, but also 
decreases the MSE compared to the classical Lasso.

\subsection{Synthetic data: given linear graph structure}
\label{sec:expLin}

We now consider the case where the prior given on the variables is a graph
structure and where we are interested by solutions which are highly connected
components on this graph. As a first simple illustration, we consider
a chain in which variables with successive indices are connected. 
We use $\w\in\RR^p$, $p=100$, $\supp{\w} = [20,40]$. The nodes of
the graph correspond to the parameters $w_i$ and the edges to the pairs
$(w_i,w_{i+1}), i=1,\ldots,n$. The parameters of the model and the
$50$ training examples $(\x_i,y_i)$ are drawn using the same protocol as in
the previous experiment.  We use for the groups all the sub-chains of
length $k$. Results are reported for various choices of $k$ and
compared to the Lasso ($k = 1$).
\begin{figure}[ht]
\begin{center}
\begin{tabular}{cc}
  \includegraphics[width=.48\linewidth]{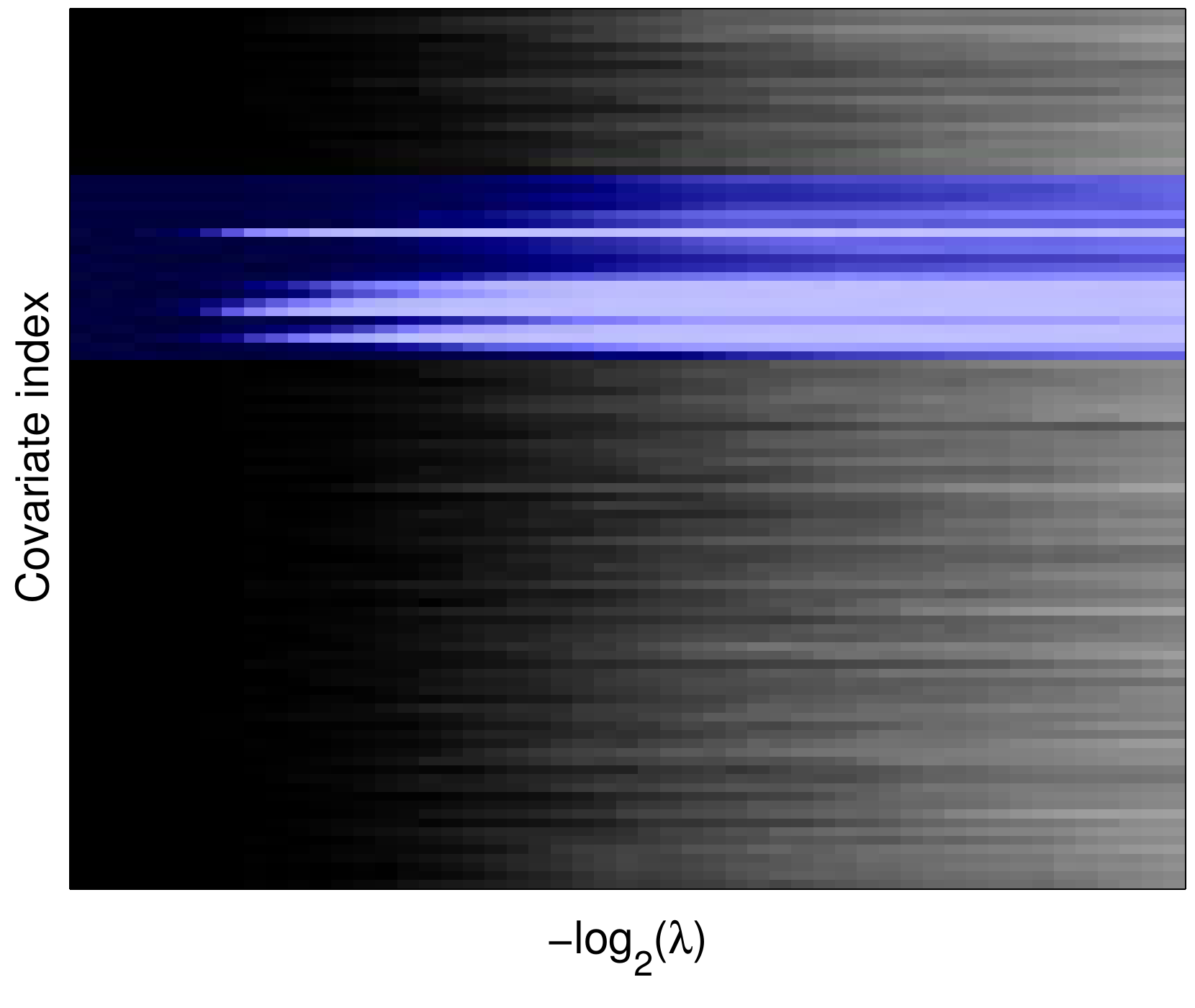} \hspace*{-3mm}&\hspace*{-3mm}
  \includegraphics[width=.48\linewidth]{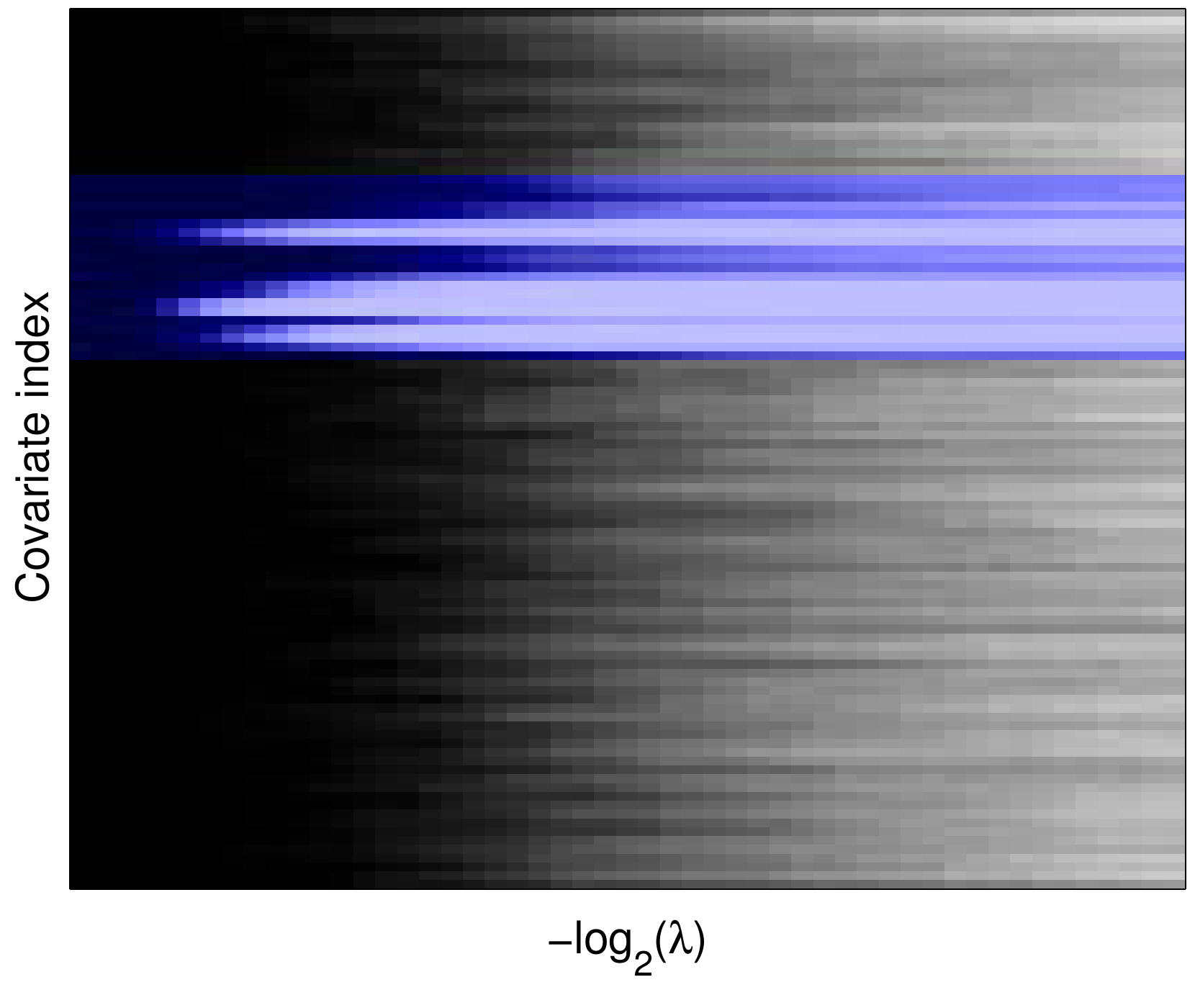}\\
  \includegraphics[width=.48\linewidth]{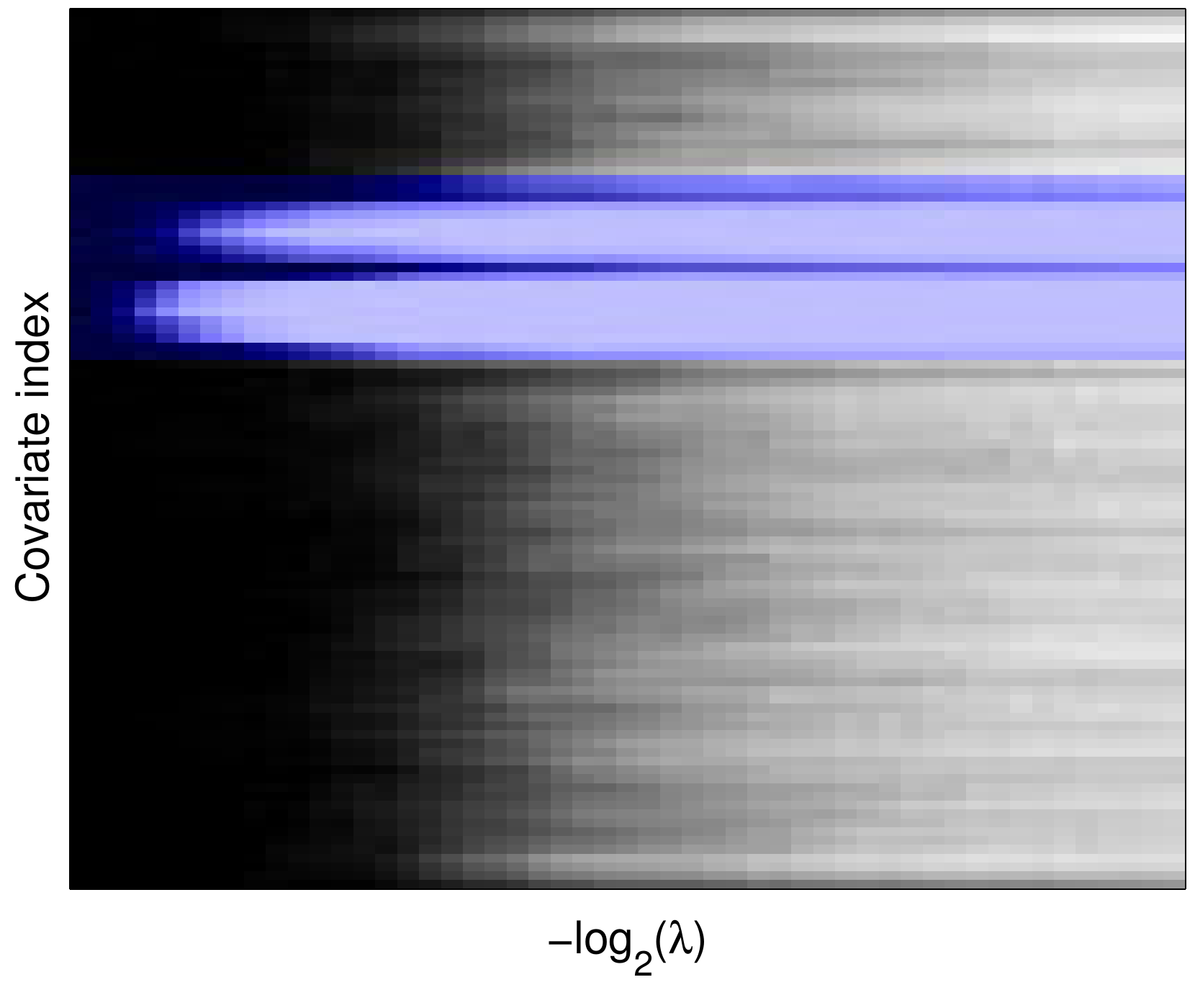} \hspace*{-3mm}&\hspace*{-3mm}
  \includegraphics[width=.48\linewidth]{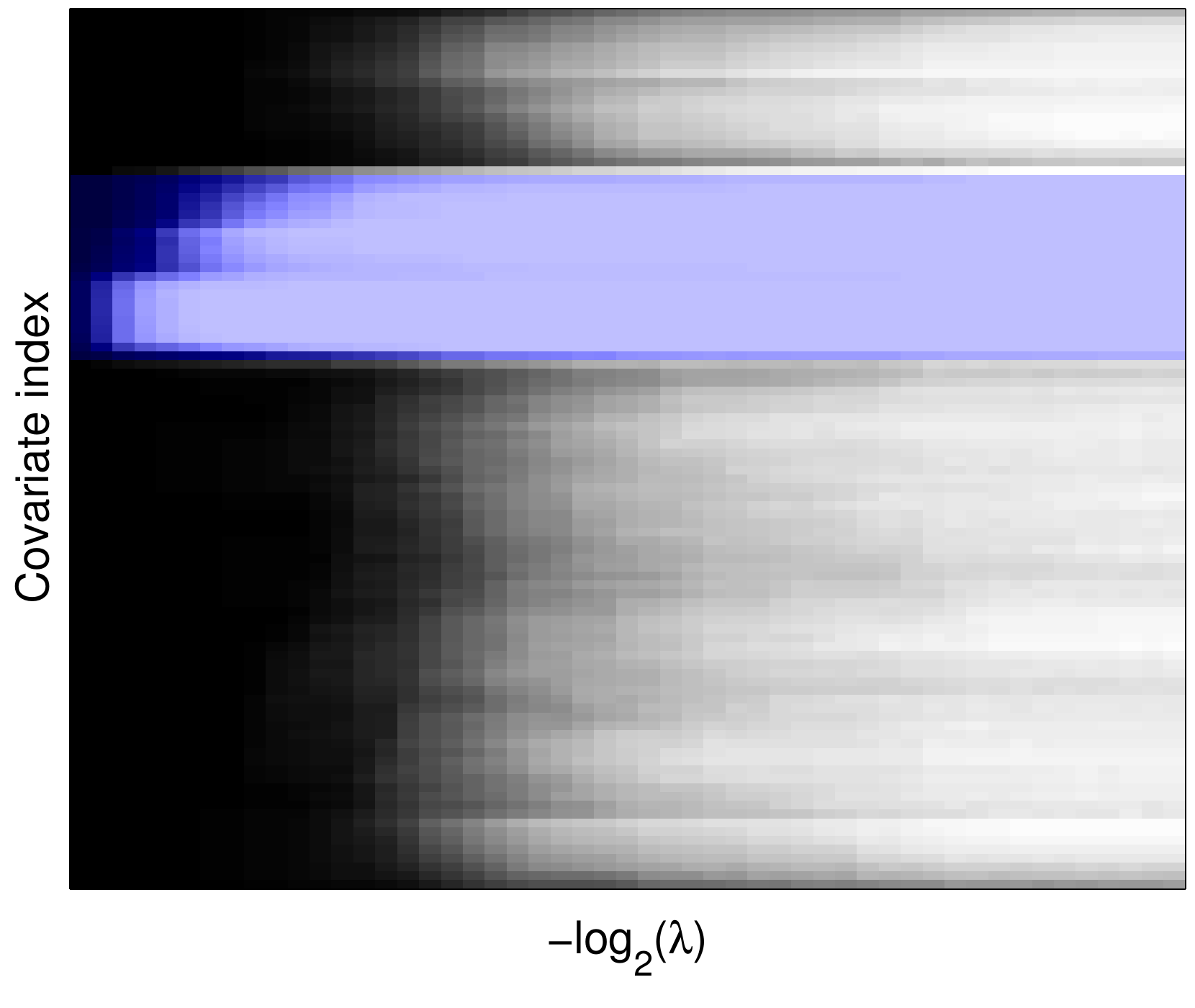}\\
\end{tabular}
\end{center}
\caption{Variable selection frequency with $\Omn$ using the chains
  of length $k$ (left) as groups, for $k = 1$ (Lasso), $2,4,8$. For each variable index (on the y-axis),
its frequency of selection is represented in levels of gray as a function of the
regularization parameter $\lambda$ (on the x-axis), both for the Lasso penalty and
$\Omn$.
The transparent blue band superimposed
  indicates the covariates that belong to the support.}
\label{fig:WChainSel50}
\end{figure}

Figure~\ref{fig:WChainSel50} shows the frequency of each variable
selection over $20$ replications. Here again, using a group prior
improves pattern recovery, with better results as $k$ increases. 
However, for larger groups, two consecutive groups are very correlated, which makes it more difficult to identify the exact
boundaries of the support.

\subsection{Synthetic data: effect of the weights}

As discussed in Section~\ref{sec:weights}, the choice of a set of
weights $\{d_g\}_{g\in\G}$ influences the variable selection
behavior of the learning algorithm penalized by $\Oo{}$. At one extreme,
if the weights are uniform, only groups that are included in no other can be selected.
At the other extreme, for weights
growing as the square root of the group size, the group-support selected will be composed (almost surely) of the smallest groups possible covering the support.

To illustrate the effect of the weighting scheme on covariate
selection, we run  three experiments with respectively $p=100,200,300$
covariates and $n=100,50,30$ training points. In each
setting, the groups are all the sets of size from $1$ to $20$ formed by sequences of 
consecutive covariates, much like in \ref{sec:expLin} but with more
groups. Note that this creates a lot of nested groups. The support is
formed by covariates with indices from $5$ to $24$ and from $90$ to $92$,
\emph{i.e.}, $23$ covariates. The noise level $\sigma^2$ is $0.1$. For
each of the three settings, we compare $6$ weighting schemes over $50$
replications. The first $4$ schemes follow \eqref{eq:dk} and assign $d_s = \sqrt{s+c\sqrt{s}}$
to each group of size $s$, with $c=0,1,4,6$. We also try $d_s =
\sqrt[4]{s}$ (the limit when $c$ grows) and $d_s = 1$. Note that
$d_s=1$ and $c = 0$ ($d_s = \sqrt{s}$) correspond to the two extreme
regimes in condition~(\ref{eq:suff_cond}).

We evaluate the performance of the regularization in two different ways.
First, we select
by cross-validation the value of $\lambda$ that yields the smallest MSE
and return the corresponding value. 
Second, we return the best possible recovery error attainable on the entire regularization path. 
We consider these two criteria since it is known that the regularization regime corresponding to optimal support recovery and best MSE are not the same~\citep{Leng2004note,Bach2008Bolasso}.

Ideally, for support recovery, we would have to either use a theoretical value for $\lambda$ or to use the OLS-hybrid two-step procedure~\citep{Efron2004Least} in which the models obtained in sequence along the regularization path are refitted with OLS and tested on a held out set to select the best model. This would obviously lead to a much heavier experimental setting, which is why we simply return the best performance along the path.

The results are shown in Table~\ref{tab:syn-100-100},
\ref{tab:syn-50-200} and \ref{tab:syn-30-300}. In each case, the best
average MSE across the $50$ runs and along the regularization path is
given along with the corresponding point on the regularization path
($\lambda^*$), average number of selected variables in the corresponding model
(Model size$^*$), pattern recovery error of the selected model (Rec
$\textrm{err}^*$) and lowest pattern recovery error along the
regularization path (Rec err min). The pattern recovery error is the
average of the proportion of covariates that were in the support and
were not selected, and the proportion of covariates that were not in
the support and were selected. The standard deviation is given for
each measured quantity as well. The regularization path was
approximated by a grid of $51$ values of $\lambda$ between $2^{-7}$
and $2^{3}$. For Table~\ref{tab:syn-50-200}, a longer grid of $76$
values starting at $2^{-12}$ was used to make sure that the end of the
regularization path was reached.

\begin{figure}[ht]
\begin{center}
  \includegraphics[width=.28\linewidth]{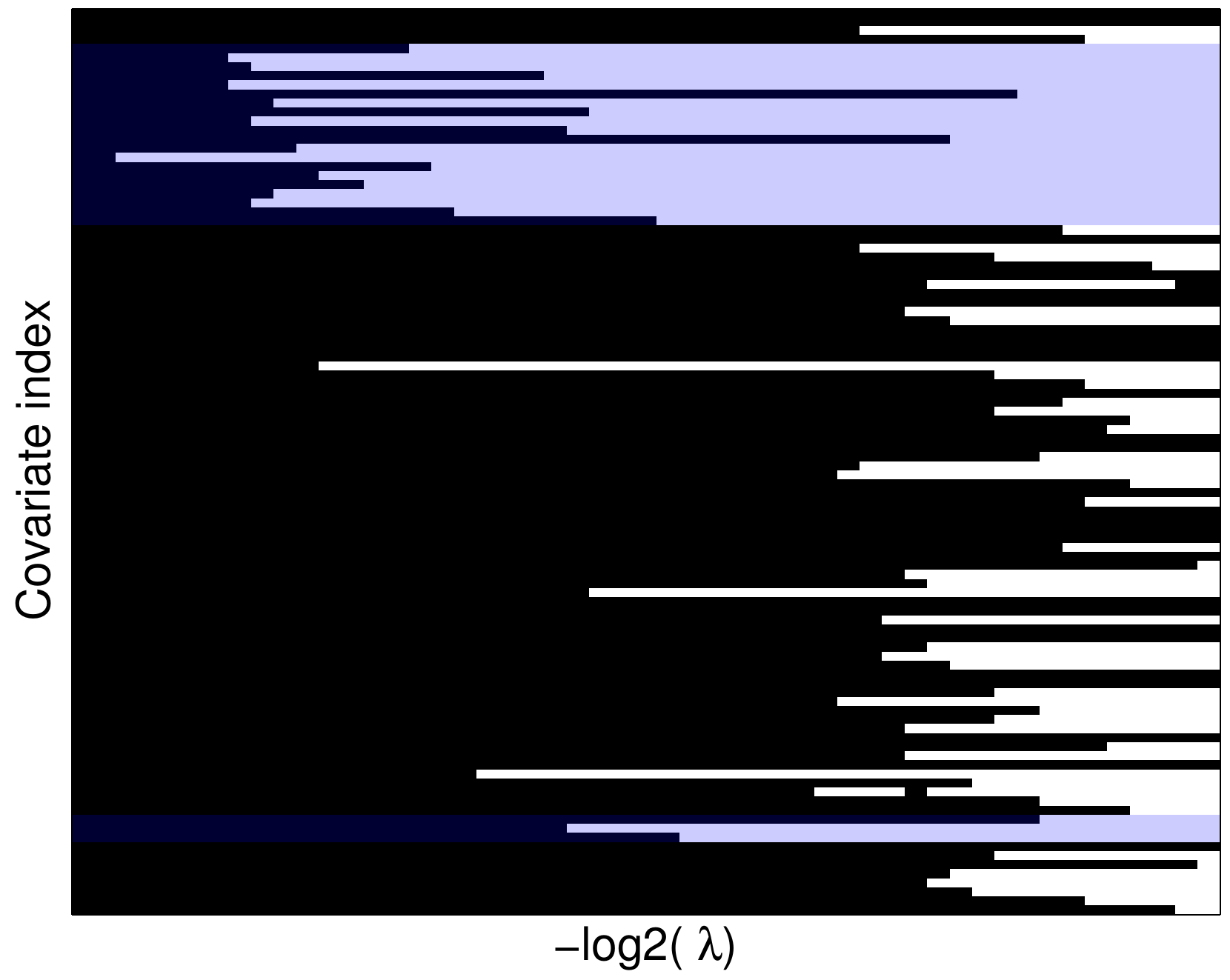}
  \includegraphics[width=.28\linewidth]{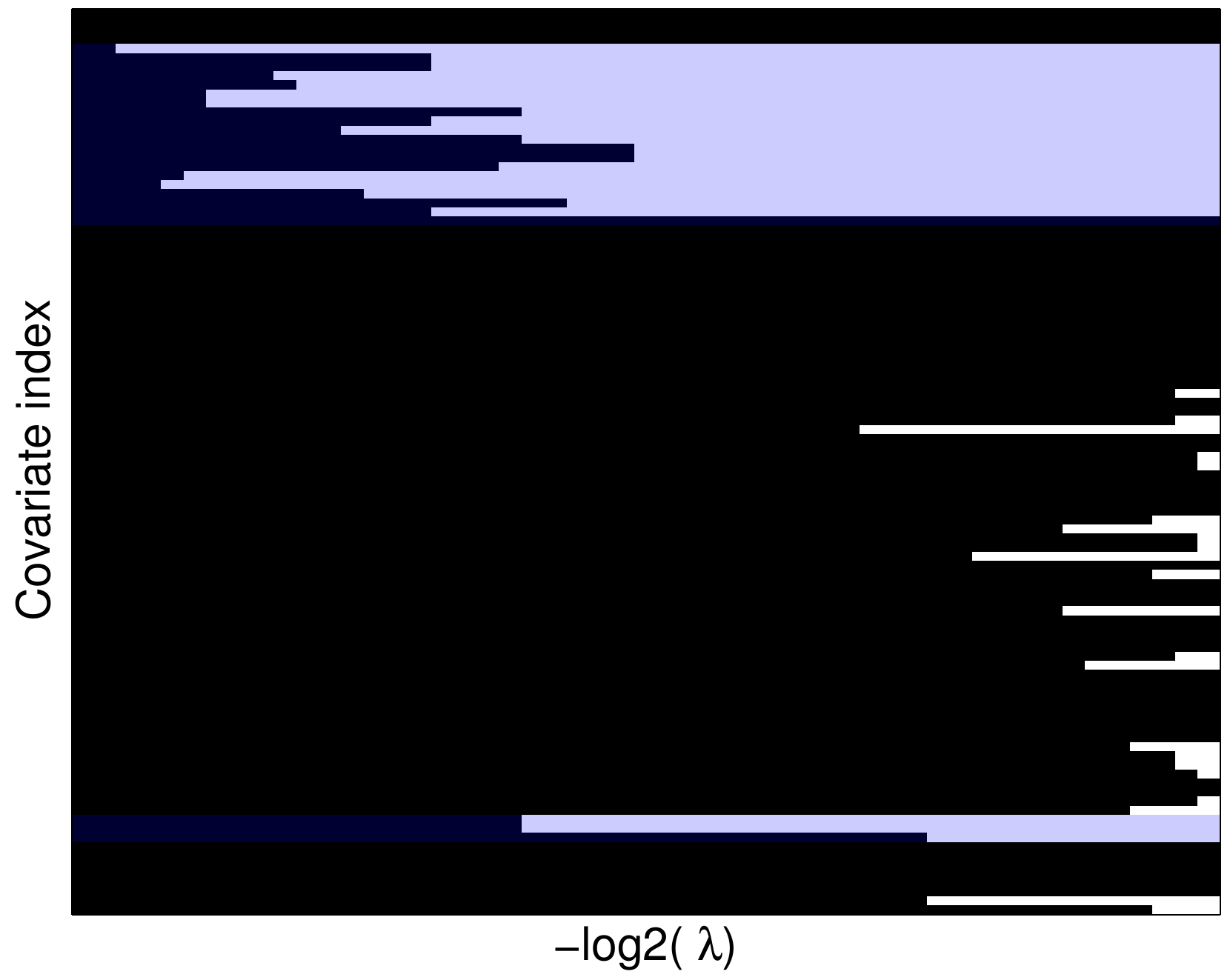}
  \includegraphics[width=.28\linewidth]{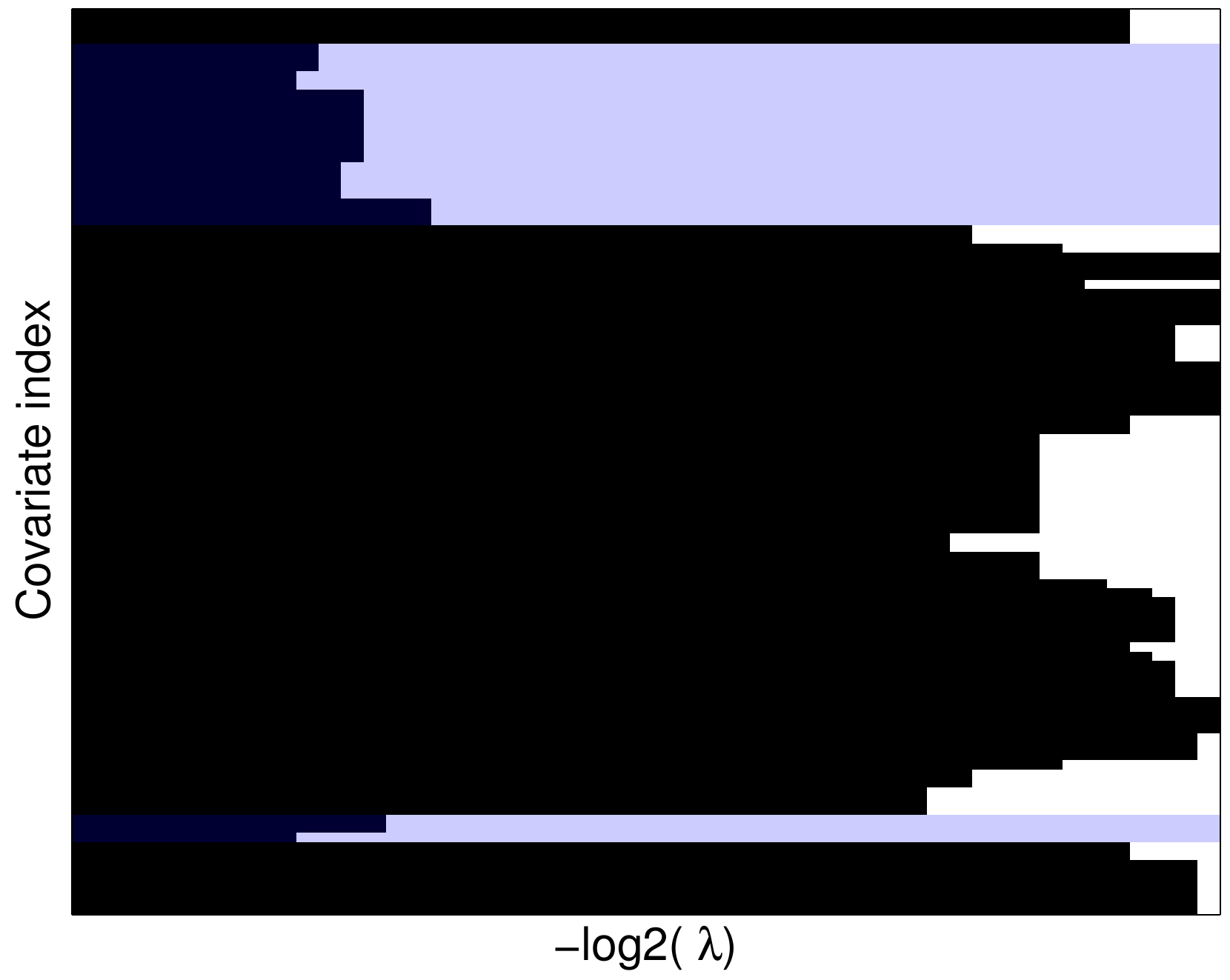}
  \includegraphics[width=.28\linewidth]{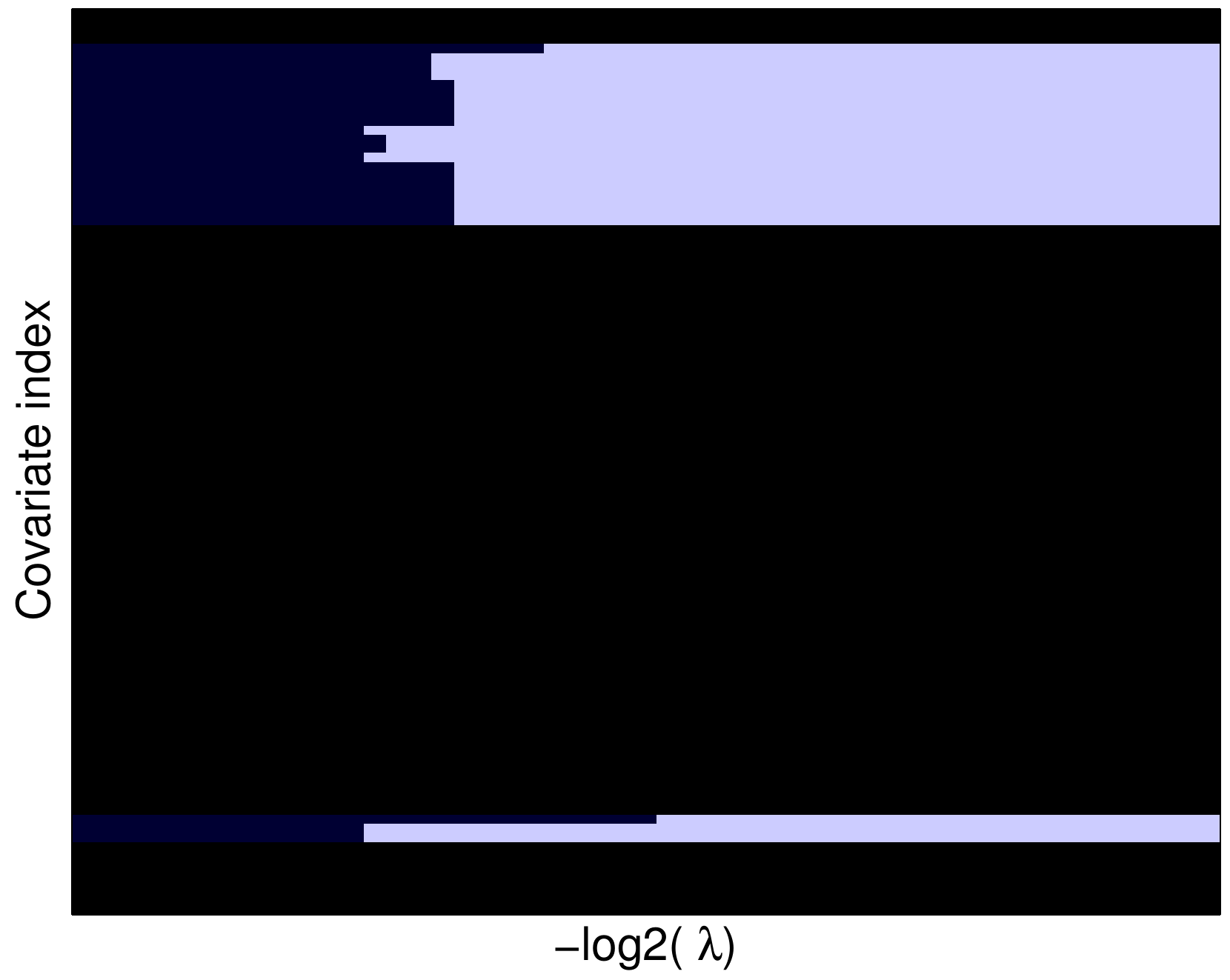}
  \includegraphics[width=.28\linewidth]{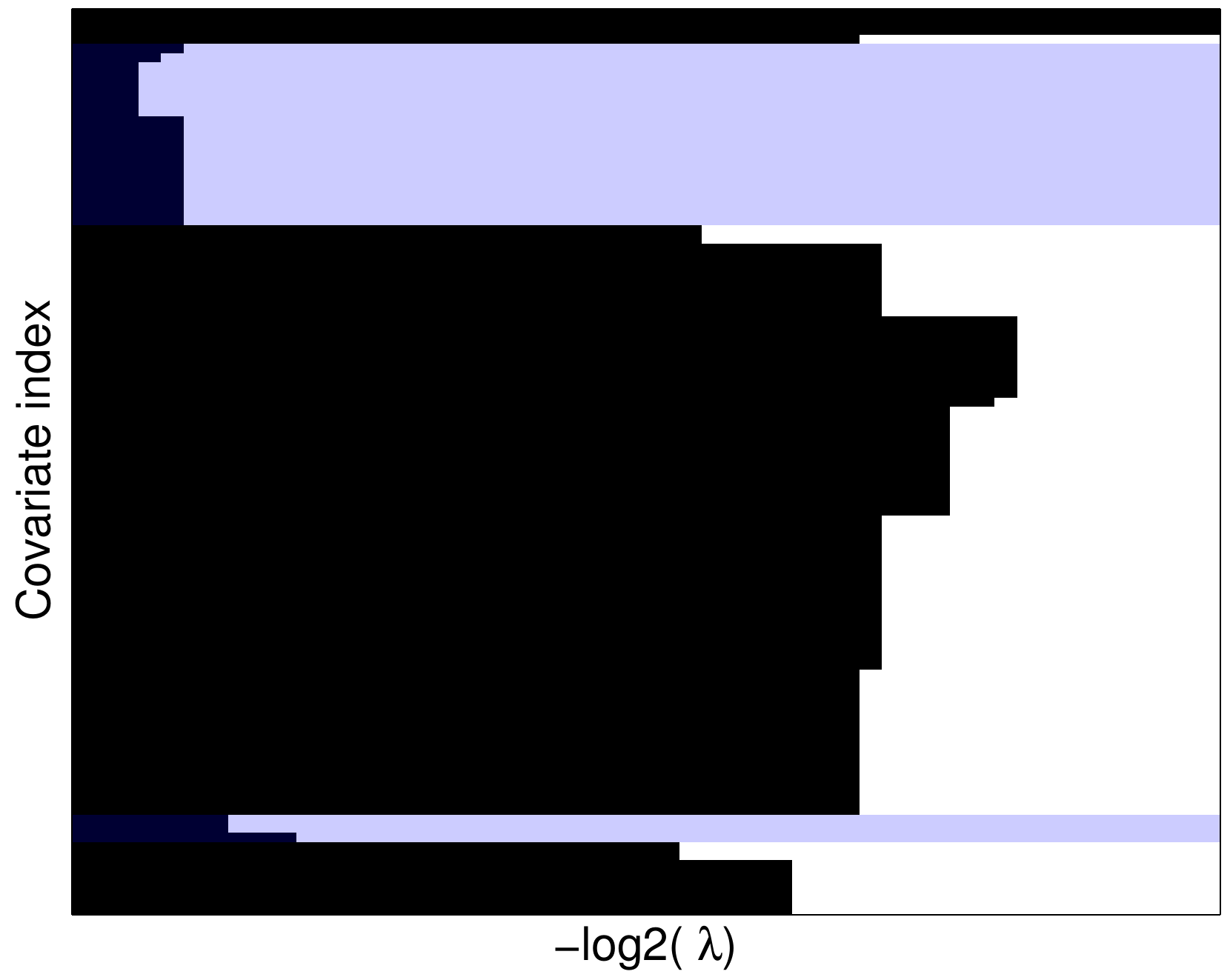}
  \includegraphics[width=.28\linewidth]{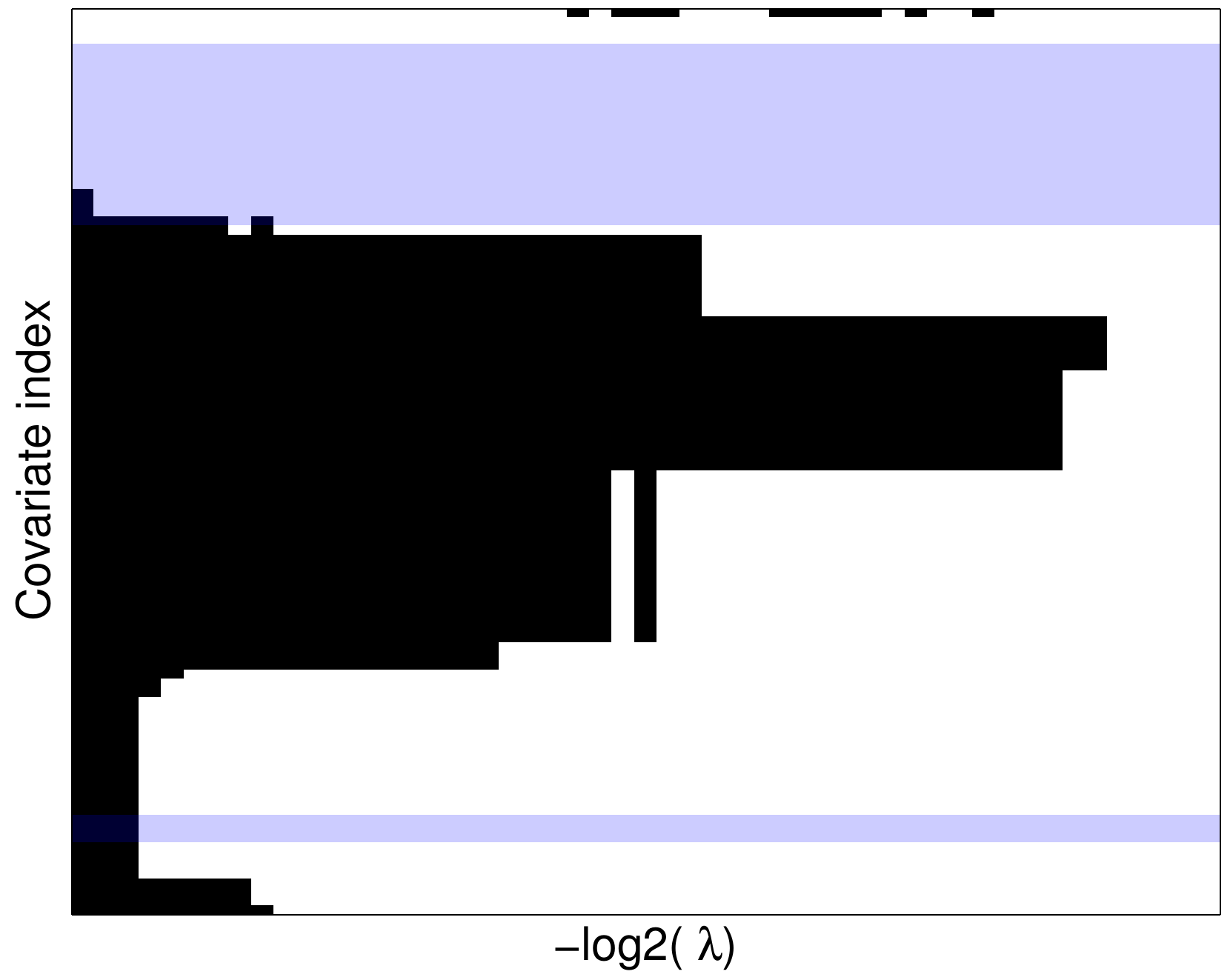}
\end{center}
\caption{Variable selection for one of the $50$ runs with $\Omn$
  using the chains up to length $20$ as groups and weights of the form $d_k=\sqrt{k+c\sqrt{k}}$, $c=0,1,4,6,\infty$ and uniform
  weights (from left to right, top to bottom). A transparent blue band is superimposed to indicate
  the covariates that belong to the support.}
\label{fig:WChainWeights}

\end{figure}

The last column of Table~\ref{tab:syn-100-100} illustrates the effect
of the weighting scheme on pattern recovery. 

The results of Table 1 correspond to $n=100$, $p=100$ so that if $s=23$ is the size of the
support, we have $n/(2s\log(p))\approx 0.47$ which means that the sample size is slightly too
small for the Lasso to recover the support exactly.
Note that as expected from the
theory, the fifth column shows that the model
selected based on the MSE is not optimal in term of variable
selection. The fourth column shows that 
 more uniform weights encourage the selection of more
variables, which is expected given that they favor the selection of larger groups. Lastly,
the values of the MSE suggest that in this regime of sparsity, dimension
and number of training points, the performances in pattern recovery
have little influence on the MSE, because there are enough training
points to deal with the noise created by the selection of spurious
covariates. Here again however, the two extreme regimes lead to higher
MSE.

Figure~\ref{fig:WChainWeights} illustrates the influence of the
weights on the selection behavior. As expected from theory, uniform
weights ($d_k = 1$) only allow selection of the largest groups
\emph{i.e.}, chains of size $20$ while at the other extreme, for $d_k
= \sqrt{k}$, only the small groups (singletons) are active. In
intermediate regimes, all groups are active and allow to recover the
correct support at some point on the regularization path, except $c=1$
which on this particular run doesn't yield perfect recovery. More
adequate choices of $c$ lead to correct recovery on a larger portion
of the regularization path.

Table~\ref{tab:syn-50-200} corresponds to a harder regime, with
fewer training points and in higher dimension. As in the first regime,
the fourth and last columns shows that the weighting scheme has a
significant influence on the variable selection behavior, with more uniform
schemes leading to more variables selected, and a better pattern
recovery being achieved for an intermediate scheme ($c=6$). The reason
for the optimal $c$ to be higher than in the previous regime may be
that in higher dimension with less training points, it is not possible
anymore to recover the fine structure of the true pattern and a better
alternative is to select a less precise but more stable selection of
larger groups. In terms of MSE, the minimum is reached for $d_s =
\sqrt[4]{s}$, and for all the other weightings the optimum $\lambda$
is the last one in the grid, for which a large fraction of the covariates
have entered the model. 

In the last regime ($30$ training points, $300$ dimensions),
Table~\ref{tab:syn-30-300} shows that the best pattern recovery is
performed with uniform weights, which suggests that at this level of
noise, using the fine structure of the groups is more harmful than
helpful, and that the best choice is to only use the largest groups. The
same reasoning applies to the MSE.

\begin{table}[t]
  \caption{Effect of $c$ on the MSE, the $\lambda$ giving the best
    average MSE, the pattern recovery error at the optimal MSE, and
    the best pattern recovery error possible. $100$ training points,
    $100$ dimensions, $50$ replications.}
\label{tab:syn-100-100}
\begin{center}
\begin{small}
\begin{sc}
\begin{tabular}{lccccc}
  \hline
  c & MSE & $\lambda^*$ & Model size$^*$ & Rec $\textrm{err}^*$ & Rec
  err min\\ 
  \hline
  $0$ & $0.06709 \pm 0.1814$ & $0.02368$ & $ 37.08 \pm   12.8$ & $0.1068 \pm 0.07444$ & $0.07148 \pm 0.03768$ \\ 
  $1$ & $0.02891 \pm 0.09583$ & $0.01031$ & $  41.8 \pm   18.4$ & $0.1245 \pm   0.12$ & $0.02951 \pm 0.02057$ \\ 
  $4$ & $0.04513 \pm 0.07202$ & $0.0136$ & $ 49.72 \pm  27.21$ & $0.1759 \pm 0.1759$ & $0.01468 \pm 0.01599$ \\ 
  $6$ & $0.03877 \pm 0.1116$ & $0.01031$ & $ 45.78 \pm  26.63$ & $0.1506 \pm 0.1741$ & $0.01804 \pm 0.01579$ \\ 
  $d_s = \sqrt[4]{s}$ & $0.04318 \pm 0.08945$ & $0.0359$ & $ 51.72 \pm  27.11$ & $0.1878 \pm 0.1757$ & $0.02461 \pm 0.02585$ \\ 
  $d_s = 1$ & $0.09263 \pm 0.2278$ & $0.04737$ & $ 81.22 \pm  17.16$ & $0.3764 \pm 0.1129$ & $0.09788 \pm 0.03598$ \\ 
  \hline
\end{tabular}
\end{sc}
\end{small}
\end{center}
\end{table}

\begin{table}[t]
  \caption{Effect of $c$ on the MSE, the $\lambda$ giving the best
    average MSE, the pattern recovery error at the optimal MSE, and
    the best pattern recovery error possible. $50$ training points,
    $200$ dimensions, $50$ replications.}
\label{tab:syn-50-200}
\begin{center}
\begin{small}
\begin{sc}
\begin{tabular}{lccccc}
\hline
c & MSE & $\lambda^*$ & Model size$^*$ & Rec $\textrm{err}^*$ & Rec
err min\\ 
\hline
$0$ & $ 8.264 \pm  5.187$ & $0.04123$ & $ 47.54 \pm  7.149$ & $0.2706
\pm 0.06144$ & $0.2661 \pm 0.06096$ \\ 
$1$ & $ 6.317 \pm  4.809$ & $0.0002441$ & $  61.3 \pm  3.824$ & $0.1957 \pm 0.07468$ & $0.1823 \pm 0.08499$ \\ 
$4$ & $ 2.428 \pm  2.401$ & $0.0002441$ & $ 101.4 \pm  13.74$ & $0.2301 \pm 0.04765$ & $0.08716 \pm 0.05194$ \\ 
$6$ & $   2.2 \pm  2.404$ & $0.0002441$ & $ 111.9 \pm  17.29$ &
$0.2572 \pm 0.05094$ & $0.06944 \pm 0.03839$ \\ 
d(s) = $\sqrt[4]{s}$ & $  1.66 \pm  1.593$ & $0.0007401$ & $ 141.2 \pm  15.52$ & $0.3366 \pm 0.04511$ & $0.0823 \pm 0.05281$ \\ 
d(s) = $1$ & $ 3.707 \pm  2.836$ & $0.0002441$ & $ 155.4 \pm  14.44$ & $0.3757 \pm 0.0409$ & $0.08228 \pm 0.02283$ \\
\hline
\end{tabular}
\end{sc}
\end{small}
\end{center}
\end{table}

\begin{table}[t]
  \caption{Effect of $c$ on the MSE, the $\lambda$ giving the best
    average MSE, the pattern recovery error at the optimal MSE, and
    the best pattern recovery error possible. $30$ training points,
    $300$ dimensions, $50$ replications.}
\label{tab:syn-30-300}
\begin{center}
\begin{small}
\begin{sc}
\begin{tabular}{lccccc}
\hline
c & MSE & $\lambda^*$ & Model size$^*$ & Rec $\textrm{err}^*$ & Rec err
min\\ 
\hline
$0$ & $ 18.78 \pm  7.021$ & $  1.32$ & $ 15.74 \pm  3.451$ & $0.4059 \pm 0.07167$ & $ 0.396 \pm 0.07169$ \\ 
$1$ & $ 17.21 \pm  6.763$ & $0.5743$ & $ 23.22 \pm  3.501$ & $0.3841 \pm 0.06413$ & $0.3693 \pm 0.07547$ \\ 
$4$ & $ 17.21 \pm  8.195$ & $ 0.125$ & $  51.5 \pm  10.74$ & $0.2281 \pm 0.1294$ & $0.2181 \pm 0.1285$ \\ 
$6$ & $ 14.74 \pm  7.398$ & $ 0.125$ & $ 66.86 \pm  17.36$ & $0.2037
\pm 0.1122$ & $0.1996 \pm 0.1198$ \\ 
d(s) = $\sqrt[4]{s}$ & $ 11.81 \pm  5.307$ & $0.007812$ & $ 119.8 \pm  23.15$ & $0.2259 \pm 0.08258$ & $0.1546 \pm 0.1197$ \\ 
d(s) = $1$ & $ 11.82 \pm   5.31$ & $0.007812$ & $ 159.2 \pm  24.22$ & $ 0.268 \pm 0.0401$ & $0.1284 \pm 0.05387$ \\  
\hline
\end{tabular}
\end{sc}
\end{small}
\end{center}
\end{table}

\subsection{Breast cancer data: pathway analysis}

An important motivation for our method is the possibility to perform
gene selection from microarray data using priors which are overlapping
groups. Genes are known to modify each other's expression through
various regulation mechanisms. More generally, some genes are known to
be involved in the same biological function, so the presence of a
particular gene in a predictive models can be indicative of the
presence of related genes. In other words, when
we select one gene in our predictive model, we can expect that genes
which are known to either regulate or to be regulated by this gene, or more
generally to be involved in the same biological function should also be
selected.  
Since an increasing amount of
information on gene interaction is being gathered from empirical
biological knowledge and organized in
databases~\citep{Subramanian2005Gene}, our hope is to use this
information to~:
\begin{description}
\item[Improve prediction accuracy~:] 
  Functions involving a small number of pre-defined
  gene sets, form a smaller hypothesis sets in which we can hope 
  to better estimate. Since genes present in the same biological function are
  likely to be either all involved in the studied phenomenon (disease
  outcome, subtype, response to a treatment) or all not involved, we
  can expect to find a function predicting the phenomenon correctly in
  this class. 
\item[Build accurate sparse prediction functions~:] Building sparse
  estimators has practical implications in this context because it is
  technically easier to measure the expression level of a small number
  of genes in a patient than a whole transcriptome. Selecting a small
  number of gene sets is a more robust procedure than selecting a
  small number of genes, because it is easy to spuriously select a
  gene from a noisy training set while the evidences add up for a set
  of genes. In addition, selecting a few genes that belong to the same
  functional groups could lead to increased interpretability of the
  signature.
\end{description}

To reach this goal we use our $\Omn$ penalty with an (overlapping)
predefined gene sets as groups. Several groupings of
genes into gene sets are available in various databases. We use the
canonical pathways from MSigDB~\citep{Subramanian2005Gene} containing
$639$ groups of genes, $637$ of which involve genes from our
study. Among these, we restricted ourselves to the $589$ groups that
contained less than $50$ genes. Indeed we observed empirically that
keeping very large pathways in the penalty lead to poor
regularization, which makes sense because the presence of very large
groups allows the penalty to select a very large number of covariates
at a low cost, partially breaking the purpose of regularization. As
discussed in Section~\ref{sec:weights}, it is possible to penalize
large groups more heavily, but weighting cannot correct extreme size
discrepancies such as combinations of groups of size two and groups of
size $100$. In addition, we are interested in identifying a small
number of well defined biological functions that predict the
outcome. Selecting a large pathway which contains one third of the
genes would not be very informative.

We use the breast cancer dataset compiled by
\citet{Vijver2002gene-expression}, which consists of gene expression
data for $8,141$ genes in $295$ breast cancer tumors ($78$ metastatic
and $217$ non-metastatic). We restrict the analysis to the $2465$
genes which are in at least one pathway. Since the dataset is very
unbalanced, we use a balanced logistic loss, weighting each positive
example by the proportion of negative examples and each negative
example by the proportion of positive examples.

We estimate by $5$-fold cross validation the balanced accuracy
(average of specificity and sensitivity) of the balanced logistic
regression with $\lone$ and $\Omn$ penalties, using the pathways as
groups. As a pre-processing, we keep the $500$ genes most correlated
with the output (on each training set). This type of prefiltering is
common practice with microarray data, and all the results are quite
robust to changes in the number of genes kept. $\lambda$ is selected
by internal cross validation on each training set.

In our experiments on this very noisy dataset, we noticed that results
changed a lot with the choice of the split, often more than between
methods. In order to make sure that observed differences were actually
caused by algorithms and not by particular choices of the $5$
foldings, we repeated each experiment on $5$ choices of the $5$
foldings, and show the result for each of these choices separately.

Table~\ref{tab:bcp} gives the balanced accuracies using $\Omn$ with
and without weights, and using $\ell_1$. We observe a consistent
improvement in the performances when using $\Omn$ against $\ell_1$
(between $2\%$ and $12\%$ depending on the fold). The weighted
version of $\Omn$ using $c=4$ also leads to consistent improvement
over $\ell_1$ but is outperformed by the unweighted version of the
penalty. Table~\ref{tab:bcpNPath} shows that the unweighted version of
the penalty tends to select groups that are larger than average, since the
average size of the initial set of pathways (after the preprocessing step that keeps
only $500$ genes) is $5$ genes with a standard deviation of
slightly above $5$. The weighted penalty allows to correct this bias:
it leads to the selection of groups of average size $5$ 
but typically selects a much larger number of groups. 

Table~\ref{tab:bcpNGenes} shows the average number of genes involved
in the model learned by each of the methods. As expected, $\Oo{}$
selects more genes, since it enforces sparsity at the gene set level
but doesn't enforce sparsity at the gene level. Note however that the
number of involved genes remains reasonable. As expected given the
numbers of Table~\ref{tab:bcpNPath} the number of genes selected in
the model learned by the weighted version of $\Omn$ is even larger.

Finally, we should mention, as a caveat, that the regularization
coefficient was chosen here to minimize the classification error,
i.e., in a regime which typically overestimates the support.  A more
tedious two-stage approach allowing to remove the bias of the
estimator, would probably lead to smaller supports, as suggested by
the comparison of Rec Err and Rec Err Min in Tables 1,2 and 3.

\begin{table}[t]
  \caption{Balanced classification error for the $\ell_1$ and $\Omn$
    (with and without weights) on
    average over $5$ folds, for $5$ different folding choices.}
\label{tab:bcp}
\begin{center}
\begin{small}
\begin{sc}
\begin{tabular}{lccc}
\hline
Method & $\Omn$ & Weighted $\Omn$ & $\ell_1$ \\
\hline
Error folding $1$ & $0.29\pm 0.05$ & $0.35\pm 0.05$ & $0.36\pm 0.04$\\
Error folding $2$ & $0.30\pm 0.08$ & $0.39\pm 0.05$ & $0.42\pm 0.04$\\
Error folding $3$ & $0.34\pm 0.14$ & $0.34\pm 0.1$ & $0.37\pm 0.10$\\
Error folding $4$ & $0.31\pm 0.11$ & $0.33\pm 0.07$ & $0.37\pm 0.08$\\
Error folding $5$ & $0.35\pm 0.05$ & $0.35\pm 0.05$ & $0.37\pm 0.05$\\
\hline
\end{tabular}
\end{sc}
\end{small}
\end{center}
\end{table}

\begin{table}[t]
  \caption{Number (and size) of involved pathways in the $\Omn$
    (with and without weights) signatures on
    average over $5$ folds, for $5$ different folding choices.}
\label{tab:bcpNPath}
\begin{center}
\begin{small}
\begin{sc}
\begin{tabular}{lcc}
\hline
Method & $\Omn$ & Weighted $\Omn$ \\
\hline
Folding $1$ & $     6 \pm  1.225 ( 16.73 \pm  2.378)$ & $  45.8 \pm  21.11 (  5.35 \pm 0.6635)$\\ 
Folding $2$ & $  12.6 \pm  7.765 ( 13.86 \pm  3.589)$ & $  48.8 \pm  23.13 ( 5.092 \pm 0.4939)$\\ 
Folding $3$ & $   7.6 \pm  3.209 ( 14.86 \pm  2.584)$ & $  43.8 \pm  12.13 ( 5.147 \pm 0.7176)$\\ 
Folding $4$ & $   8.6 \pm  7.266 (  16.7 \pm  4.477)$ & $  30.6 \pm   17.3 ( 5.045 \pm 0.7267)$\\ 
Folding $5$ & $     8 \pm      1 ( 14.82 \pm  1.191)$ & $  48.4 \pm  10.62 ( 5.347 \pm 0.2867)$\\
\hline
\end{tabular}
\end{sc}
\end{small}
\end{center}
\end{table}

\begin{table}[t]
  \caption{Number of involved genes in the $\ell_1$ and $\Omn$
    (with and without weights) signatures on
    average over $5$ folds, for $5$ different folding choices.}
\label{tab:bcpNGenes}
\begin{center}
\begin{small}
\begin{sc}
\begin{tabular}{lccc}
\hline
Method & $\Omn$ & Weighted $\Omn$ & $\ell_1$ \\
\hline
Folding $1$ & $98\pm 18$ & $159.4\pm 60.1$ & $41.2\pm 20.6$\\
Folding $2$ & $86.4\pm 18$ & $143.4\pm 32$ & $59.4\pm 22.5$\\
Folding $3$ & $125\pm 37.7$ & $156.4\pm 36.7$ & $59.4\pm 21.4$\\
Folding $4$ & $91.6\pm 25$ & $115.2\pm 57.9$ & $45.6\pm 28.4$\\
Folding $5$ & $98\pm 36$ & $178.4\pm 33.9$ & $56\pm 97$\\
\hline
\end{tabular}
\end{sc}
\end{small}
\end{center}
\end{table}

\subsection{Breast cancer data: graph analysis}

Another important application of microarray data analysis is the
search for potential drug targets. In order to identify genes which
are related to a disease, one would like to find groups of genes
forming densely connected components on a graph carrying biological
information such as regulation, involvement in the same chain of
metabolic reactions, or protein-protein interaction. Similarly to what
is done in pathway analysis,~\citet{Chuang2007Network-based} built a
network by compiling several biological networks and performed such
a graph analysis by identifying discriminant subnetworks in one step and
using these subnetworks to learn a classifier in a separate step. We
use this network and the approach described in
section~\ref{sec:graphlasso}, treating all the edges on the network as
groups of size two, on the breast cancer dataset. Here again, we restrict the
data to the $7910$ genes which are present in the network, and use the
same correlation-based pre-processing as for the pathway analysis to reduce the set
to 500 genes.

Table~\ref{tab:bcgr} shows the prediction accuracy of the balanced
logistic regression with $\lone$ and $\Omn$. Both methods yield
almost exactly the same performance in average, suggesting that this
particular network is not a particularly informative prior for this learning problem.

\begin{table}[t]
  \caption{Balanced classification error of the $\lone$ and $\Omn$
    (using the edges as the groups) on
    the $5$ folds.}
\label{tab:bcgr}
\begin{center}
\begin{small}
\begin{sc}
\begin{tabular}{lcc}
\hline
Method & $\Omn$  & $\lone$\\
\hline
Folding $1$ & $0.3625 \pm 0.04538$ & $0.3367 \pm 0.03788$\\ 
Folding $2$ &$0.4142 \pm 0.05885$ & $0.4042 \pm 0.06035$\\ 
Folding $3$ &$0.3681 \pm 0.04773$ & $0.3782 \pm 0.07497$\\ 
Folding $4$ &$0.3749 \pm 0.06476$ & $0.3834 \pm 0.06449$\\ 
Folding $5$ &$0.3317 \pm 0.04318$ & $0.3443 \pm 0.04414$\\ 
\hline
\end{tabular}
\end{sc}
\end{small}
\end{center}
\end{table}

Nonetheless, while $\lone$ mostly selects isolated
variables on the graph, $\Omn$ tends to select variables which are
clustered into larger connected components. Table~\ref{tab:bcgrCC} shows, for each of the $5$ foldings, the
size of the largest connected component of the network restricted to
the selected genes (the average and standard deviations are computed
over the $5$ folds of each folding). 
The average size of the largest connected component in the network
after preprocessing (\emph{i.e.}, keeping only $500$ genes in each
training set) is $68$. One might suspect that the increase of
connectivity is merely caused by the fact that overall the $\Omn$
selects more genes. While it is clear that selecting more genes makes
it more likely to select larger connected components, the last two
columns of Table~\ref{tab:bcgrCC} suggest that the increased
connectivity is not simply caused by the selection of a larger number
of genes. For example in folding $5$, $\Omn$ selects many more genes
than $\lone$ but leads to the most modest increase in connectivity,
while in folding $4$ the number of selected genes is practically the
same, although the $\Omn$ estimate is still much more connected than
that of $\lone$.

This gain of connectivity without loss of prediction accuracy could potentially
make the interpretation of the classifier and the search for new drug
targets easier in practice.

\begin{table}[t]
  \caption{Average size of the largest connected
    components and average number of genes selected by the $\lone$ and
    $\Omn$ (using the edges as the groups) on
    the $5$ folding.}
\label{tab:bcgrCC}
\begin{center}
\begin{small}
\begin{sc}
\begin{tabular}{lcccc}
\hline
Method & $\Omn$ largest cc & $\lone$ largest cc & $\Omn$ $\sharp$
genes & $\lone$ $\sharp$ genes \\
\hline
Folding $1$ & $  10.2 \pm  5.586$ & $   1.8 \pm 0.4472$ & $  75.4 \pm  47.54$ & $  37.2 \pm  17.68$\\ 
Folding $2$ & $   6.2 \pm  3.633$ & $     2 \pm      0$ & $  58.4 \pm  30.81$ & $    50 \pm  9.301$\\ 
Folding $3$ & $   8.6 \pm  4.278$ & $     2 \pm 0.7071$ & $  53.2 \pm  8.012$ & $  43.2 \pm  5.357$\\ 
Folding $4$ & $     8 \pm  6.205$ & $   2.2 \pm 0.4472$ & $  48.6 \pm  30.25$ & $  45.6 \pm  20.63$\\ 
Folding $5$ & $     6 \pm  3.082$ & $   1.8 \pm 0.4472$ & $    69 \pm   31.2$ & $  37.2 \pm   12.3$\\
\hline
\end{tabular}
\end{sc}
\end{small}
\end{center}
\end{table}

%% file: discussion.tex
\label{sec:discussion}
We have presented the latent group Lasso, a generalization of the group lasso penalty which
leads to sparse models with sparsity patterns that are unions of
pre-defined groups of covariates, or, given a graph of covariates,
groups of connected covariates in the graph. We studied various properties of the penalty function, and gave both sufficient and
necessary conditions for \emph{group-support recovery}, \ie, the correct recovery 
of the same union of groups as in the decomposition induced by the penalty on the true
optimal parameter vector. We have highlighted the importance of setting weights correctly, and obtained promising empirical results on both simulated and real data.

In future work it would be interesting to characterize further for which collections of groups the latent group Lasso penalty and the estimators obtained by regularizing with it are computable
efficiently; which form of structures can be encoded via such collections; and what are the appropriate choice of weights in those cases, which will have to be determined based on specific analyses of the consistency of these estimators under high-dimensional scaling. Finally, more systematic comparisons with other group Lasso formulations, such as that proposed by~\citet{Jenatton2009Structured}, would be important.

%% file: appendices2.tex
\appendix

\section{Proofs of Lemmata \ref{Auhc} and \ref{Vlhc}}
\label{sec:two_lem}
Lemmata~\ref{Auhc} and \ref{Vlhc} are about the continuity of the correspondences $\w \mapsto \A(\w)$ and $\w \mapsto \Vb(\w)$. In order to prove them, we start by reviewing general results in correspondence theory (\secref{sec:tech_corr_theo}), notably  Berge's maximum theorem which is the main ingredient to prove the to lemmas. We prove \lemref{Auhc} directly in \secref{sec:proof_lemma_Auhc}. We then prove several continuity properties of auxiliary correspondences in \secref{sec:continuity1} and \ref{sec:continuity2} in order to finally prove \lemref{Vlhc} in \secref{sec:proof_lemma_Vlhc}.

\subsection{Elements of correspondence theory}
\label{sec:tech_corr_theo}
We start with a couple of useful technical lemmas from correspondence theory.

\begin{lemma}
\label{comp_w_cont}
If $f$ is a continuous function at $p$ and $\phi$ is a correspondence u.h.c. (resp. l.h.c.) at $f(p)$, then $\phi \circ f$ is  a
correspondence u.h.c. (resp. l.h.c.) at $p$.\\
If  $\phi: P \rightarrow X$ is a correspondence u.h.c. (resp. l.h.c.) at $p$ and $f$ is a continuous function on $X$
then  $f \circ \phi$ is  a
correspondence u.h.c. (resp. l.h.c.) at $p$.
\end{lemma}
\begin{proof}
The proofs are straightforward from the definitions.
\end{proof}

\begin{lemma}
\label{product}
An elementwise product of u.h.c. (resp. l.h.c.) correspondences is itself u.h.c. (resp. l.h.c.).
\end{lemma}
\begin{proof}
It is easy to check that a cartesian product of l.h.c. (resp. u.h.c.) correspondences has itself the same property. Moreover,
the product is a continuous application, so the result is proved by \lemref{comp_w_cont}.
\end{proof}

We now state without proof the celebrated maximum theorem \citep{Berge1959Espaces}.
\begin{theorem}[\bf Berge maximum theorem]
\label{Berge}
Let $\phi: P \twoheadrightarrow X$ be a compact-valued correspondence. Let $f: X \times P \rightarrow \RR$ be a continuous real valued function.
Define the ``argmax" correspondence $\mu: P \twoheadrightarrow X$ by
$\mu(p)= \big\{x \in \phi(p) \, \big | \,  f(x,p)=\max_{x' \in \phi(p)} f(x',p) \big \}.$
If $\phi$ is continuous at $p$, then $\mu$ is non-empty, compact-valued and u.h.c. at $p$.
\end{theorem}

\subsection{Proof of \lemref{Auhc}}\label{sec:proof_lemma_Auhc}

\lemref{Auhc} is a simple consequence of \thmref{Berge}. Indeed, remember that, by definition, $\mathcal{A}(\w)=\text{argmax}_{\alphab} \alphab^\top \w \: \text{s.t.} \: \Oo^*(\alphab) \leq 1$. Since $(\alphab,\w) \mapsto \alphab^\top \w$ is continuous and
since the correspondence  $\w \mapsto \{\alphab \in \RR^p \,|\, \Oo^*(\alphab) \leq 1\}$ is compact-valued and continuous (it is constant), \thmref{Berge} applies and shows that the correspondence $\w\mapsto\A(\w)$ is u.h.c. (For more general results on the continuity of the subdifferential viewed as a multi-function see \citet[][chap.~VI.6.2 p.~282]{Hiriart1994Convex}).
\hfill\BlackBox\\[2mm]

\subsection{Continuity properties of $\Vb(\w)$, $\Lambdab(\w)$ and $\Zb(\w)$}\label{sec:continuity1}
The fact that $\w\mapsto\Vb(\w)$ is u.h.c. is also a direct consequence of Berge's maximum theorem. We show this in the following two lemmata.
\begin{lemma}
The correspondence $\phi$ defined by 
\begin{equation}
\label{eq:phi}
\phi(\w)=\big \{\vb \in \VV^\G \mid \w=\sum \v ^g, \:  \sign(\v^g_i)=\sign(\w_i), \: 1\leq i \leq p \big \}
\end{equation} is a continuous correspondence.
\end{lemma}
\begin{proof}
We have $\phi(\w)=\prod_{i=1}^p \phi_i(\w_i)$ with $$\phi_i(\w_i)=\big \{ (\v^g_i)_{g \in \G} \in \RR^{\m} \mid \w_i=\sum_{g \in \G} \v^g_i, \: \forall i \in g ,\: \sign(\v^g_i)=\sign(\w_i), \text{and} \: \v^g_i=0, i \notin g \big \}.$$
 It is easy to verify that a Cartesian product of compact-valued continuous correspondences is also continuous, so that we only need to show that $\phi_i$ is compact-valued and continuous.
We therefore focus on $\phi_i(\w_i) \subset \RR^{\m}$.
 First note that $\phi_i$ is compact valued because the sign constraints in the definition of $\phi_i$ imply that for all $\v_i=(\v_i^g)_{g \in \G} \in \phi_i(\w_i)$ we have $\|\v_i\|_1 \leq |\w_i|$.
We first show that $\phi_i$ is  u.h.c.. Let $U$ be an open set containing $\phi_i(\w_i)$.
For two sets $A,B\subset \RR^{\m}$, we define $d_\infty(A,B)\eqdef \inf_{a \in A, b \in B} \|a-b\|_{\infty}$. Let $u_0 \in U^c$, $d_0\eqdef d_{\infty}(\{u_0\},\phi_i(\w_i))$ and define $K \eqdef \big \{u \in \RR^{\m} \mid d_{\infty}(\{u\},\phi_i(\w_i))\leq d_0 \big \}$. By construction $K \cap U^c \neq \varnothing$, and we have $d_\infty(U^c,\phi_i(\w_i)) = d_\infty(U^c\cap K,\phi_i(\w_i))$. Moreover, it is classical to show that the compactness of $\phi_i(\w_i)$ implies that $K$ is compact as well.
Since $U^c\cap K$ and $\phi_i(\w_i)$ are compact sets the infimum in the definition of $d_\infty$ is attained, which means that there are $\u^*\in U^c\cap K$ and $\v^* \in \phi_i(\w_i)$ such that  $d_\infty(U^c\cap K,\phi_i(\w_i))=\|\u^*-\v^*\|_\infty$. But we must have $\|\u^*-\v^*\|_\infty>0$ otherwise $\u^*=\v^*\in U^c \cap \phi(\w_i)$ which would contradict the hypothesis that $\phi_i(\w_i) \subset U$. If $\varepsilon\eqdef \|\u^*-\v^*\|_\infty/2$, we just showed that for all $\delta \in \RR^{\m}$ such that $\|\delta\|_\infty \leq \varepsilon, \: \phi_i(\w_i)+\delta \subset U$.

 If $\w_i=0$, then any decomposition of $\w_i \pm \varepsilon$, say $\check{\v}_i$ is such that $\|\check{\v}_i\|_\infty \leq \varepsilon$, and $\phi(\w_i \pm \varepsilon) \subset U$. If $\w_i\neq 0$, w.l.o.g. assume that $\w_i>0$; consider a decomposition $\check{\v}_i \in \RR^m$ of $\w_i+\varepsilon'$ with $|\varepsilon'|\leq \min(\varepsilon,|\w_i|/2)$; if $\varepsilon' < 0$ then $\v_i \eqdef \check{\v}_i+\varepsilon' \mathbf{e}_1$ is a decomposition
of $\w_i$ and $\|\v_i-\check{\v}_i\|_\infty \leq \varepsilon'$; if $\varepsilon'>0$ then it is easy to show that the projection $\v_i$ of $\check{\v}_i$ on the simplex $\phi(\w_i)$ satisfies 
$\|\v_i-\check{\v}_i\|_\infty<\varepsilon'$. In all cases $\phi(\w_i+\varepsilon') \subset U$ for some $\varepsilon>0$, which shows that $\phi$ is u.h.c.. 

We can show similarly that $\phi$ is l.h.c.\ : if $\v_i \in U \cap \phi(\w_i)$, then for some $\varepsilon>0$, $U$ contains a closed $\ell_\infty$ ball of radius $\varepsilon$ centered at $\v_i$, which contains a decomposition of $\w_i \pm \varepsilon$ so that  
$U \cap \phi(\w_i \pm \varepsilon) \neq \varnothing$.
\end{proof}

\begin{lemma}
\label{Vuhc}
The correspondence $\w\mapsto\Vb(\w)$ is compact-valued and u.h.c.
\end{lemma}
\begin{proof}
Define $f(\vb, \w)=\sum_{g \in \G} \|\v^g\|$ and $\phi$ as in \eqref{eq:phi}.

We have that $\Vb(\w)=\text{Argmin}_{\vb \in \phi(\w)} f(\vb,\w)$ since it can be shown easily that any optimal decomposition satisfies $\sign(\v^g_i)=\sign(\w_i)$.

Since the previous lemma shows that $\phi$ is a compact-valued continuous correspondence,  theorem \ref{Berge} applies and proves the result.
\end{proof}

Remember that $\Lambdab(\w) \subset \RR^{\m}$ is the set of solutions to (\ref{eq:trick}). For a vector $\lambdab \in \RR^{\m}$ we consider the vector $\zetab(\lambdab) \in \RR^p$ defined by
$\zeta_i(\lambdab)=\sum_{g \ni i} \lambdag$, and denote 
$\Zb(\w)=\{\zetab(\lambdab) \in \RR^p, \: \lambdab \in \Lambdab(\w) \}$. 

\begin{lemma}
\label{lem:LZuhc}
$\Lambdab(\w)$ and $\Zb(\w)$ are u.h.c. correspondences.
\end{lemma}
\begin{proof}
Since $\Vb$ is u.h.c., by lemma~\ref{comp_w_cont}, the continuity of $(\v^g)_{g \in \G} \mapsto (\|\v^g\|)_{g \in \G}$ shows that $\Lambdab(\w)$ is u.h.c. and the continuity of  $\lambdab \mapsto \big ( \sum_{g \ni i} \lambdag \big )_{1\leq i\leq p}$ shows that $Z_i(\w)$ is u.h.c..
\end{proof}

\begin{lemma}
\label{zeta_prop}
For all $i$ such that $\w_i \neq 0$, $Z_i(\w)$ is a singleton, and if we denote this unique value by $\zeta_i(\w)$ then the function $\w' \mapsto \zeta_i(\w')$ is uniquely defined in a neighborhood of $\w$ and it is continuous at $\w$.
\end{lemma}
\begin{proof}
Uniqueness of $\zeta_i(\w)$ at $\w$ such that $\w_i \neq 0$ is granted by the fact that if $\w_i \neq 0$, then $\alpha_i \neq 0$, $\alpha_i$ is unique (cf lemma \ref{lem:alpha}) and the proof of lemma~\ref{lem:trick} shows that $\zeta_i=\frac{w_i}{\alpha_i}$. Thus, $\zeta_i(\w)$ is unique, but so is $\zeta_i(\w')$ for $\w'$ in a small neighborhood of $\w$ since $\w'_i\neq 0$. 

Moreover we have $\zeta_i(\w)=\sum_{g \in \G} \lambdag$ for any $\lambdab \in \Lambdab(\w)$.  Finally the upper hemicontinuity of $\w \mapsto Z_i(\w)$ shown in the previous lemma implies the continuity of $\zeta_i$.
\end{proof}

\begin{lemma}
\label{Llhc}
Let $\mathcal{S}=\{\u \in \RR^p \, | \, \supp{\u} \subset \Jw\}$.
Consider $\w$ such that $\forall i \in \Jw$ and for all $\u$  in a neighborhood of $\zv$ in $\mathcal{S}$, $Z_i(\w+\u)$ is a singleton, then
if $\Pi_{\Gw}$ denotes the projection  on $\{\lambdab \in \RR^{\m}\, |\, \lambdab_{\Gw^c}=\zv \}$ we have that
\begin{eqnarray*}
\Lambdab|_{\Jw}^{
\Gw}: \mathcal{S} & \twoheadrightarrow & \RR^{|\Gw|} \\ \w' &\mapsto& \Pi_{\Gw}\Lambdab(\w')
\end{eqnarray*}
 is a lower hemicontinuous correspondence at $\w$.
\end{lemma}
\begin{proof}
Let $\B \in \RR^{p \times \m}$ the adjacency matrix associated to $\G$, defined by $\B_{ig}=1$ if $i \in g$ and $0$ else. To simplify notations we denote $\tilde{\B}=\B_{{\Jw} \Gw}$ the submatrix obtained by keeping rows in $\Jw$ and columns in $\Gw$, $\tilde{\zetab}= \zetab_{\Jw}(\w')$ and $\tilde{\Lambdab}=\Pi_{\Gw}\Lambdab(\w')$.
Given $\tilde{\zetab}$, then $\tilde{\Lambdab}=\{\tilde{\lambdab} \in \RR_+^{|\Gw|} \, | \, \tilde{\zetab}=\tilde{\B} \tilde{\lambdab} \}$ which means that if $\tilde{\B}^+$ denotes the Moore-Penrose pseudo-inverse of  $\tilde{\B}$ then $\tilde{\Lambdab}=\big (\tilde{\B}^+ \tilde{\zetab} + \mathcal{K}er(\tilde{\B}) \big ) \cap \RR^{|\Gw|}_+$. 

We now show that this correspondence is l.h.c.. The uniqueness of $\tilde{\zetab}$ implies its continuity, since by lemma~\ref{lem:LZuhc}, $Z_i(\w)$ is u.h.c.. Denoting by $\Hb$ a matrix whose columns form a basis of $\mathcal{K}er(\tilde{\B})$, $\h^g$ and $\b^g$ the $g^\text{th}$ row of $\Hb$ and $\tilde{\B}^+$ respectively, then an element of $\tilde{\Lambdab}$ is of the form $(\b^g \tilde{\zetab}+\h^g \q)_{g \in {\Gw}}$ for some $\q$. Given an element $\tilde{\B}^+ \tilde{\zetab}+\Hb\q \in U \cap \RR^{|\Gw|}_+$, we show that there exists an element $\lambda(\w+\u,\q')\eqdef \tilde{\B}^+ \tilde{\zetab}(\w+\u)+\Hb\q' \in U \cap \RR^{|\Gw|}_+$ for $\u$ in neighborhood of $\zv$ in $\mathcal{S}$. Without loss of generality we can take $U$ a cartesian product of open sets $U=\bigotimes_{g \in \Gw} U_g$.

Let $\mathcal{Q}=\{\q' \mid \tilde{\B}^+ \tilde{\zetab}(\w)+\Hb\q' \in \RR^{|\Jw|}_+\}$.
For all $g \in \Gs$, there exists $\q^{(g)} \in \mathcal{Q}$ such that $\b^{g} \tilde{\zetab}+\h^{g} \q^{(g)}>0$. Set $\q'=(1-\epsilon) \,\q+\frac{\epsilon}{|\Gs|}\sum_{g \in \Gs} \q^{(g)}$.
For $\epsilon$ sufficiently small, $\lambda_g(\w,\q') \in U_g \cap \RR_+^*$, for all $g \in \Gs$ so that for $\u$ sufficiently small $\lambda_g(\w+\u,\q') \in U_g \cap \RR_+^*$ as well. For all $g \notin \Gs, \: \Lambda_g(\w)=\{0\}$ and since $\Lambda$ is u.h.c., for any $\eta>0$, for $\u$ sufficiently small we have $\Lambda_g(\w+\u) \subset [0, \eta), \, g \notin  \Gs$. Choosing $\eta$ such that $\forall  g \notin \Gs, \, [0,\eta) \subset U_g$ shows the result.
\end{proof}

\subsection{Continuity properties of $\Gw$ and $\Gs$}\label{sec:continuity2}
\begin{lemma}
\label{Gbar_ngbrs}
There exists a neighborhood $U$ of $\zv$ in $\RR^p$ such that for all $\u \in U$ with $\supp{\u} \subset \Jw(\w)$, $\Gw(\w+\u)\subset \Gw(\w)$.
\end{lemma}
\begin{proof}
By definition of $\Gw(\w+\u)$, if $g \in \Gw(\w+\u)$, then $\alphab_g(\w+\u)$ is unique by lemma~\ref{Aunq}, since $g \subset \Jw(\w+\u)$. For any $g \in \Gw(\w+\u)$,  $g \cap \Jw(\w) \neq \varnothing$; indeed if $g \cap \Jw(\w)=\varnothing$, then $\w_g=\u_g=\zv$.
 If $g \subset \Jw(\w)$, $\alphab_g(\w)$ is unique and since $\alphab_g(\w+\u)$ is unique, the upper hemicontinuity of $\mathcal{A}$ implies that $\alphab_g$ is continuous at $\w$
so that $(\|\alphab_g(\w+\u)\|=1 \Rightarrow \|\alphab_g(\w)\|=1)$. If $g \backslash \Jw(\w) \neq \varnothing$, then it has to be the case that $\alphab_{g \backslash \Jw(\w)}(\w+\u)=\zv$, because it is indeed a possible value for $\alphab_{g \backslash \Jw}(\w+\u)$ (given that $\w_{g \backslash \Jw(\w)}=\u_{g \backslash \Jw(\w)}=\zv$) and because $\alphab_g(\w+\u)$ is unique. This implies that $\|(\alphab_{g \cap \Jw(\w)}(\w+\u)\|=1$ and since $\alphab_{g \cap \Jw(\w)}(\w)$ is unique, upper hemicontinuity of  $\mathcal{A}$ implies that $\w' \mapsto \alphab_{g \cap \Jw(\w)}(\w')$ is continuous at $\w$ so that we have by continuity $\|\alphab_g(\w)\| \geq \|\alphab_{g \cap \Jw(\w)}(\w)\|=1$ which proves that $\|\alphab_g(\w)\| =1$; but this is a contradiction because this would imply $g \in \Gw$ and therefore $g \subset \Jw$.
\end{proof}

\begin{lemma}
\label{wgs_is_ngs}
Let $\mathcal{D}_{\Jw}=\{\u \in \RR^p \mid \|\u\| \leq \,1, \u_{\Jw^c}=\zv  \}$; then
$$\Gw(\w)=\bigcap_{\epsilon>\,0} \:\bigcup_{\u \in \mathcal{D}_{\Jw}} \Gs(\w+\epsilon\, \u).$$
\end{lemma}
\begin{proof} 
One inclusion is already shown by the previous Lemma~\ref{Gbar_ngbrs}. For the other inclusion, let $\vb$ be an optimal decomposition of $\w$ and $\alphab$ the unique element of $\mathcal{A}(\w)$ such that $\alphab_{\Jw^c}=\zv$.
Let $\lambdag=\|\v_g\|$. The case of $g \in \Gs(\w)$ is straightforward, and we concentrate therefore on $g \in \Gw(\w) \backslash \Gs(\w)$. By lemma~\ref{lem:alpha}, we have $\w=\sum_{g \in \Gs} \lambdag \alphab_g$. Consider $\w_{(g_0,\epsilon)}=\w+\epsilon \, \alphab_{g_0}$ for some $g_0 \in \Gw(\w) \backslash \Gs(\w)$.  By construction, $\alphab \in \mathcal{D}_{\Jw}$ and for all $\betab \in \RR^p$ such that  $\Oo^*(\betab)\leq 1$ we have
$$\w_{(g_0,\epsilon)}^\top\betab = \sum_{g \in \G} \lambdag\, \alphab_g^\top \betab_g + \epsilon\,  \alphab_{g_0}^\top \betab_{g_0} \leq \sum_{g \in \G} \lambdag + \epsilon=\w_{(g_0,\epsilon)}^\top \alphab$$
 which shows that $\bar{\v}'$ defined by $\v'_{g_0}=\epsilon\, \alphab_{g_0}$ and $\v'_g=\v_g, \: g \neq g_0$ is an optimal decomposition of $\w_{(g_0,\epsilon)}$ with group-support $\Gs(\w) \cup g_0$. Since this is true for any $\epsilon$ and any $g_0 \in \Gw(\w) \backslash \Gs(\w)$, this proves the statement.
\end{proof}

\subsection{Proof of Lemma \ref{Vlhc}}\label{sec:proof_lemma_Vlhc}
We know from Lemma \ref{Vuhc} that $\w\mapsto\Vb(\w)$ is a compact-valued u.h.c. correspondence.
If $\supp{\w}=\Jw$ then lemma~\ref{zeta_prop} implies that for all $i \in \Jw$, $\zeta_i(\w+\u)$ is unique for all $\u$ in a neighborhood of $\zv$. From lemma~\ref{Llhc}, this implies that $\u \mapsto \Pi_{\Gw}\Lambdab(\w+\u)$ is l.h.c at $\u=\zv$. This extends to $\u \mapsto \Lambdab(\w+\u)$ since we know from Lemma \ref{Gbar_ngbrs} that there exists a neighborhood of zero such that, for all $\u$ in that neighborhood, $\Pi_{\Gw^c}\Lambdab(\w+\u)=\zv$.
Given that $\Vb(\w+\u)=\alphab(\w+\u) \Lambdab(\w+\u)$, since $\alphab(\w)$ is l.h.c. from Lemma \ref{Auhc} and since a product of l.h.c.\ correspondences is l.h.c.\ (cf.\ Lemma \ref{product}), we have shown that $\u \mapsto \Vb(\w+\u)$ is also l.h.c. at $\u=\zv$.
\hfill\BlackBox\\[2mm]


\section{Partial group-support recovery}
\label{sec:part_cons}
Theorem~\ref{theo:part_cons}, which only assumes hypothesis (H1), does not  give a lower bound (in the sense of inclusion) for $\Gs(\w)$, suggesting that hypothesis (H2) is necessary to guarantee group-support recovery. In this section, we first consider an example in which 
$\Gs(\w)$ is strictly included in $\Gs(\ws)$.

{\bf Example with partial recovery}. Take $\G=\big \{ \{0,1,2\}, \{0,1,3\}, \{0,2,3\}\big \}$ for $\w=(w_0,w_1,w_2,w_3) \in \RR^4$. It is easy to check that $\lambda_{\{0,1,2\}}=\gamma (|w_1|+|w_2|-|w_3|)_+, \:\lambda_{\{0,1,3\}}=\gamma (|w_1|+|w_3|-|w_2|)_+$ and
$\lambda_{\{0,2,3\}}=\gamma (|w_2|+|w_3|-|w_1|)_+$ with $\gamma$ determined by the equation $\sum_{i=0}^2\frac{w_i^2}{\zeta_i^2}=1$.
In particular if we consider $\ws=(1,0,0,0)$, then taking the identity as the design matrix and assuming independent Gaussian noise, we have
$y=(1+\epsilon_0,\epsilon_1,\epsilon_2,\epsilon_3)$ with $\epsilon_i$ i.i.d. $\mathcal{N}(0,\sigma^2)$. Thus solving the first order approximation of the KKT in the neighborhood of $\ws$ we get $\w=((1+\epsilon_0-\lambda)_+,\epsilon_1,\epsilon_2,\epsilon_3)$.  We have $\Gs(\ws)=\Gw(\ws)=\G$ but for any value of $\sigma^2$,  with probability $\mu,\mu,\mu$ and $1-3 \mu$, $\Gs(\w)$ takes respectively the values $\G \backslash \{0,1,2\}$, $\G \backslash \{0,1,3\}$,$\G \backslash \{0,2,3\}$ and $\G$, with $\mu \approx 0.216$.

However, the following lemma shows that the group-support recovered contains at least the group-support of one of the decomposition of the true support.

\begin{lemma}
\label{supp_inclus}
If $\w_n$ is a sequence converging to $\w$, then denoting $\gsupp{\vb}$ the group support of a decomposition $\vb$, we have 
$$\exists n_0, \forall n \geq n_0, \forall \vb_n \in \Vb(\w_n),\exists \vb \in \Vb(\w), \: \gsupp{\vb} \subset \gsupp{\vb_n}.$$
\end{lemma}
\begin{proof}
Reason by contradiction and assume that
$$\forall n_0, \exists n \geq n_0, \exists \vb_n \in \Vb(\w_n),\forall \vb \in \Vb(\w), \: \gsupp{\vb} \nsubseteq \gsupp{\vb_n}.$$
We can therefore extract a subsequence $(\w_{\varphi(n)})_n$ with this property and the corresponding subsequence $(\vb_{\varphi(n)})_n$ illustrating it.
There exists at least one $\mathcal{G}_0 \in 2^{|\mathcal{G}|}$ such that there are infinitely many elements $\vb_{\varphi(n)}$ in the subsequence which satisfies $\gsupp{\vb_{\varphi(n)}}=\mathcal{G}_0$. We consider the subsequence $(\vb_{\varphi'(n)})_n$ composed of those elements.
From the sequence $(\vb_{\varphi'(n)})_n$, since we can assume without loss of generality it lives in the compact set $\{\vb \mid \forall g \in \mathcal{G}, \|\v^g\| \leq 2 \|\w\| \}$, we can extract a converging subsequence $(\vb_{\varphi''(n)})_n$. Since $(\w_{\varphi''(n)})_n$ converges to $\w$ and by upper hemicontinuity of
$\Vb(\cdot)$ the subsequence $(\vb_{\varphi''(n)})_n$ converges to an optimal decomposition $\vb_\infty$ of $\w$. This implies that $\gsupp{\vb_\infty} \subset \mathcal{G}_0=\gsupp{\vb_{\varphi''(n)}}$ which is a contradiction.
\end{proof}

The simpler example with $\G=\big \{ \{1,2\},\{2,3\} \big \}$ and $\ws=(0,1,0)$ could be expected to be problematic since $(0,1,\epsilon)$ and $(\epsilon,1,0)$ have respectively group-support $\big \{\{2,3\} \big \}$ and $\big \{\{1,2\} \big \}$. However, this case is consistent since it can be shown that $\w_1$ and $\w_3$ are almost surely non-zero, which implies that both groups are part of the group-support.

\section{Derivations for the illustrative examples}
\subsection{Graph Lasso for the cycle of length 3}
\label{sec:cycle_three}
We consider the overlap norm in $\RR^3$ with groups $\G=\big \{ \{1,2\}, \{1,3\}, \{2,3\}\big \}$.
If $\alphab$ denotes a dual variable. The dual norm takes the form:
$$\Oo^*(\alphab) \eqdef \max \big (\|(\al_1,\al_2)\|, \|(\al_1,\al_3)\|,\|(\al_2,\al_3)\|\big )$$
By Fenchel duality,
$\displaystyle \Oo(\w)= \max_{\alphab \in \RR^3}\:\alphab^\top \w \:\: \text{s.t.} \:\: \max_{g \in \G} \|\alphab_g\|^2\leq 1$. Consider the Lagrangian
\begin{equation*}
\begin{split}
L^*(\alphab,\l,\w)& =-( \al_1 w_1+\al_2 w_2+\al_3 w_3)\\ & \quad +\frac{1}{2} \big [
(\l_{12}+\l_{13}) \, \al_1^2 +
(\l_{12}+\l_{23}) \, \al_2^2 +
(\l_{13}+\l_{23}) \, \al_3^2
 -(\l_{12}+\l_{13}+\l_{23}) \big ]
\end{split}
\end{equation*}
 and consider the optimization problem
$\quad \displaystyle \min_{\alphab \in \RR^p} L^*(\alphab,\l,\w) \quad \text{s.t.} \quad \l_{g} \geq 0,\: g \in \G.$\\
A singular point of the Lagrangian satisfies
\begin{equation}
\label{eq:singpt}
w_1=(\l_{12}+\l_{13}) \,\al_1,  \quad
w_2=(\l_{12}+\l_{23}) \,\al_2, \quad
w_3=(\l_{13}+\l_{23}) \, \al_3.
\end{equation}
\subsubsection{At most two groups are active}
\label{sec:at_most_two}
Assume that $\l_{13}=0$. Note that this case reduces to the case of $\G=\big \{ \{1,2\}, \{2,3\}\big \}$, which is of interest on its own.
Eq. \ref{eq:singpt} simplifies and the singular points of the Lagrangian solve
\begin{equation}
\label{eq:singpt2}
w_1=(\l_{12}) \, \al_1, \quad 
w_2=(\l_{12}+\l_{23}) \, \al_2, \quad 
w_3=(\l_{23}) \, \al_3.
\end{equation}
We assume first that $\l_{12}>0, \, \l_{23}>0, \,|w_1|>0, \, |w_3|>0$.
Since, by complementary slackness,
$\|\al_{12}\|=1$ and $\|\al_{23}\|=1$, using \eqref{eq:singpt}, we have
\begin{equation}
\label{eq:KKT2bk1}
\frac{w_1^2}{\l_{12}^2} + \frac{w_2^2}{(\l_{12}+\l_{23})^2}=1 \qquad \text{and} \qquad
\frac{w_2^2}{(\l_{12}+\l_{23})^2} + \frac{w_3^2}{\l_{23}^2}=1.\\
\end{equation}
So that $\frac{w_1^2}{\l_{12}^2}=\frac{w_2^2}{\l_{23}^2}$ or equivalently $\l_{23}=\frac{|w_3|}{|w_1|} \l_{12}$ and by substitution in
\eqref{eq:KKT2bk1} we get respectively:
$$
\l_{12}=\frac{|w_1|}{|w_1|+|w_3|} \|(w_2, |w_1|+|w_3| )\| \quad \text{and} \quad \l_{23}=\frac{|w_3|}{|w_1|+|w_3|} \|(w_2, |w_1|+|w_3| )\|.
$$
Substituting these expressions for $\l_{12}$ and $\l_{23}$ in the singular point equations \eqref{eq:singpt2}, we get:
\begin{equation}
\al_1=\sign(w_1) \frac{|w_1|+|w_3|}{\|(w_2, |w_1|+|w_3| )\| } \quad \text{and} \quad  \al_2=\frac{w_2}{\|(w_2, |w_1|+|w_3| )\| }.
\end{equation}
$\al_3$ has a similar expression as $\al_1$, where the roles of $w_3$ and $w_1$ are exchanged.
Finally, the decomposition is:
\begin{equation}
v_{12}=
\begin{pmatrix}
w_1,
\frac{ |w_1|}{|w_1|+|w_3|}\,w_2
\end{pmatrix}^\top
 \quad \text{and} \quad
 v_{23}=
\begin{pmatrix}
\frac{ |w_3|}{|w_1|+|w_3|} \, w_2 ,
w_3
\end{pmatrix}^\top,
\end{equation}
and the norm then takes the closed form
$\displaystyle \Oo(w)=\|\,(w_2,|w_1|+|w_3|)\,\|$.
Remains to consider the cases where $w_1=0$, or $w_3=0$, which we do not develop here.

\subsubsection{All groups are active}
We first consider the case $\l_{12}>0, \l_{13}>0, \l_{23}>0$.
By complementary slackness we have
$\|\al_g\|=1, \: g \in \G.$ 
Introducing $\zeta_1=\l_{12}+\l_{13}, \:\zeta_2=\l_{12}+\l_{23}$ and $\zeta_3=\l_{13}+\l_{23}$, \eqref{eq:singpt} rewrites as
$$
\displaystyle \frac{w_1^2}{\zeta_1^2} + \frac{w_2^2}{\zeta_2^2}=1, \quad
\displaystyle \frac{w_2^2}{\zeta_2^2} + \frac{w_3^2}{\zeta_3^2}=1, \quad
\displaystyle  \frac{w_1^2}{\zeta_1^2} + \frac{w_3^2}{\zeta_3^2}=1.
$$
which taking pairwise differences yields:
\begin{equation}
\label{eq:allprop}
\frac{1}{\gamma} \eqdef 
\frac{w_1^2}{\zeta_1^2}=\frac{w_2^2}{\zeta_2^2}=\frac{w_3^2}{\zeta_3^2} 
\end{equation}
Or in other words:
\begin{equation*}
\begin{pmatrix}
|w_1| \\|w_2| \\|w_3| \\
\end{pmatrix}
=\frac{1}{\gamma}
\begin{pmatrix}
\zeta_1 \\ \zeta_2 \\ \zeta_3\\
\end{pmatrix}
= \frac{1}{\gamma}
\begin{pmatrix}
1&1&0 \\1&0& 1 \\0& 1& 1\\
\end{pmatrix}
\begin{pmatrix}
\l_{12} \\ \l_{13} \\ \l_{23} \\
\end{pmatrix}
\end{equation*}
 which yields
 $$\l_{12}=\gamma (|w_1|+|w_2|-|w_3|), \quad \l_{13}=\gamma (|w_1|+|w_3|-|w_2|), \quad \l_{23}=\gamma (|w_2|+|w_3|-|w_1|).$$
 But since we have assumed $\l_g>0$, the solution found is only valid if no coordinate dominates in the sense that $\w \in  \mathcal{W}_{\text{bal}}$ with
 \begin{equation*}
 \mathcal{W}_{\text{bal}}\eqdef \left \{\w \, \in \RR^3 |
 |w_1| \leq |w_2|+|w_3|, \;
 |w_2| \leq |w_1|+|w_3|, \;
 |w_3| \leq |w_1|+|w_2|
 \right \}
 \end{equation*}
 By re-substituting \eqref{eq:allprop} in \eqref{eq:singpt}, we can solve for $\gamma$ and find that
 $$\boxed{\:\alphab=\frac{1}{\sqrt{2}}\, \text{sign}(\w) \quad \text{and thus} \quad \Oo(\w)=\frac{1}{\sqrt{2}}\, \|\w\|_1\:}$$
 The unit ball of the norm therefore has some flat faces.
 Finally, since $(\v_g)_g$ is an optimal decomposition of $w$ we have $\v_g=\l_g \alphab_g$,
 the decomposition is unique and can be written
$$
\v_{\{12\}}=\frac{1}{2}
\begin{pmatrix}
w_1 +(|w_2|-|w_3|)\,\text{sign}(w_1)\\
w_2 +(|w_1|-|w_3|)\,\text{sign}(w_2)\\
\end{pmatrix},
\quad
\v_{\{13\}}=\frac{1}{2}
\begin{pmatrix}
w_1 +(|w_3|-|w_2|)\,\text{sign}(w_1)\\
w_3 +(|w_1|-|w_2|)\,\text{sign}(w_3)\\
\end{pmatrix},
$$ $$\text{and} \qquad 
\v_{\{23\}}=\frac{1}{2}
\begin{pmatrix}
w_2 +(|w_3|-|w_1|)\,\text{sign}(w_2)\\
w_3 +(|w_2|-|w_1|)\,\text{sign}(w_3)\\
\end{pmatrix}.
$$
If $\w \notin \mathcal{W}_{\text{bal}}$, then one of $\l_{12},\l_{13}$ or $\l_{23}$ equals $0$, and this reduces to the situation where only two groups are active which we considered in section~\ref{sec:at_most_two} above.

\subsubsection{Closed form expression for the norm}
Finally, summarizing the analysis, we obtain the closed form expression:
\begin{equation*}
\Om{w}=
\begin{cases}
\frac{1}{\sqrt{2}}\, \|\w\|_1 & \text{if}\; \w \in \mathcal{W}_{\text{bal}} \\
\min \begin{cases}
 \|\,(w_1,|w_2|+|w_3|)\,\| , \\ \|\,(w_2,|w_1|+|w_3|)\,\| , \\ \|\,(w_3,|w_1|+|w_2|)\,\| \end{cases}
& \text{else.}
\end{cases}
\end{equation*}
\subsection{Graph Lasso for the cycle of length 4}
\label{sec:cycle_four}
We consider here the case where the groups are $\G=\big \{\{1,2\},\{1,3\},\{2,4\},\{3,4\} \big \}$. This case is interesting because we will show that non-sparse $\w$ on the cycle always admit several optimal decompositions.
The dual norm takes the form:
$$\Omax(\alphab) \eqdef \max \big (\|(\al_1,\al_2)\|, \|(\al_1,\al_3)\|,\|(\al_2,\al_4)\| , \|(\al_3,\al_4)\|\big )$$
We use again Fenchel duality, write
$ \displaystyle \Oo(\w)= \max_{\alphab \in \RR^4} \alphab^\top \w \:\: \text{s.t.} \:\: \Omax(\alphab)^2\leq 1$ and
we construct the Lagrangian:
\begin{equation*}
\begin{split}
L^*(\alphab,\l,\w)& =-( \al_1 w_1+\al_2 w_2+\al_3 w_3+\al_4 w_4)\\ & \quad +\frac{1}{2} \big [
\zeta_1\, \al_1^2 +
\zeta_2\, \al_2^2 +
\zeta_3\, \al_3^2 +
\zeta_4\, \al_4^2 -(\l_{12}+\l_{23}+\l_{24}+\l_{34}) \big ]
\end{split}
\end{equation*}
with $\zeta_1=\l_{12}+\l_{23},\, \zeta_2 =\l_{12}+\l_{24}, \zeta_3=\l_{13}+\l_{34}$ and $\zeta_4=\l_{24}+\l_{34}$
A singular point of the Lagrangian satisfies $w_i=\zeta_i \alpha_i, \: 1\leq i \leq 4$.

\subsection{All groups are active}
We first consider the case $\l_{12}, \l_{13}, \l_{24}, \l_{34}>0$.
By complementary slackness
\begin{equation}
\label{eq:cs}
\|\al_g\|=1, \: g \in \G \tag{CS}
\end{equation}
which, using \eqref{eq:singpt},rewrites as
\begin{equation}
\label{dualltwoconstraints}
\displaystyle \frac{w_1^2}{\zeta_1^2} + \frac{w_2^2}{\zeta_2^2}=1, \quad
\displaystyle  \frac{w_1^2}{\zeta_1^2} + \frac{w_3^2}{\zeta_3^2}=1,\quad
\displaystyle \frac{w_2^2}{\zeta_2^2} + \frac{w_4^2}{\zeta_4^2}=1\quad \text{and}\quad
\displaystyle \frac{w_3^2}{\zeta_3^2} + \frac{w_4^2}{\zeta_4^2}=1.
\end{equation}

Taking differences between pairs of equations above that share a common variable $w_i$ we get
$$
\begin{cases}
|w_1| (\l_{24}+\l_{34})=|w_4|(\l_{12}+\l_{13})\\
|w_2| (\l_{13}+\l_{34})=|w_3|(\l_{12}+\l_{24})\\
\end{cases}
$$
Thus, isolating $\l_{12}$ in both equations and eliminating it yields
$$
\frac{|w_1|}{|w_4|}(\l_{24}+\l_{34})-\l_{13}=\frac{|w_2|}{|w_3|}(\l_{13}+\l_{34})-\l_{24}
$$
Now isolating $\l_{13}$ we get
$$\l_{13}=\left (1+\frac{|w_2|}{|w_3|} \right )^{-1} \left ( \frac{|w_1|}{|w_4|}(\l_{24}+\l_{34}) + \l_{24} - \frac{|w_2|}{|w_3|} \l_{34}\right )$$
Adding $\l_{34}$ on both sides yields
$$\l_{13}+\l_{34}=\frac{\left (1+ \frac{|w_1|}{|w_4|} \right ) \l_{24}+ \left ( 1+\frac{|w_2|}{|w_3|} \right ) \l_{34}}{1+\frac{|w_2|}{|w_3|}}$$
Inserting this expression into the only equation of (\ref{dualltwoconstraints}) which doesn't contain $\l_{12}$ we get
$$\frac{w_3^2\left (1+\frac{|w_2|}{|w_3|} \right )^2}{\left (1+ \frac{|w_1|}{|w_4|} \right )^2(\l_{24}+\l_{34})^2}+\frac{w_4^2}{(\l_{24}+\l_{34})^2}=1$$ which reduces to
\begin{equation}
\zeta_4 \eqdef \l_{24}+\l_{34}=\frac{|w_4|}{|w_1|+|w_4|} \left [ \left ( |w_2|+|w_3| \right )^2+ \left ( |w_1|+|w_4| \right )^2 \right ]^{\frac{1}{2}}
\end{equation}

By symmetry, we get similar expressions for $\l_{12}+\l_{13}$, $\l_{12}+\l_{24}$, and $\l_{13}+\l_{34}$.
Since $\Om{w}=\l_{12}+\l_{13}+\l_{24}+\l_{34}$, we get immediately that
$$\boxed{\Om{w}=\left [ \left ( |w_2|+|w_3| \right )^2+ \left ( |w_1|+|w_4| \right )^2 \right ]^{\frac{1}{2}}=\|(|w_1|+|w_4|,|w_2|+|w_3|)\|}$$
The above derivation gave us values for $\zeta_1,\zeta_2,\zeta_3,\zeta_4$. We discuss now the existence and the uniqueness of the $(\l_g)_g$. Given the vectors $\zeta \in \RR^{4}$ and $\l \in \RR^{4}$ we have $\zeta=B \l$ where $\B$ is the incidence matrix of the groups,
with $\B_{ig}=\mathbf{1}_{\{i \in g\}}$. To be precise we have
$$
\begin{pmatrix}
\zeta_1\\ \zeta_2 \\ \zeta_3 \\ \zeta_4
\end{pmatrix}
=
\begin{bmatrix}
1 &1& 0& 0 \\ 1 &0& 1& 0 \\ 0 &1& 0& 1 \\ 0& 0 &1 &1\\
\end{bmatrix}
\begin{pmatrix}
\l_{12} \\ \l_{13} \\ \l_{24} \\ \l_{34} \\
\end{pmatrix}
$$

Clearly, in this case,  $\B$ is not invertible,
and the kernel of $\B$ is the span of $(-1, 1, 1, -1)^T$. Since the matrix is symmetric, $\mathcal{K}er(\B)=\mathcal{I}m(\B)^T$,
and since $\zeta_1+\zeta_4=\Oo(\w)=\zeta_2+\zeta_3$, we have $\zeta_1-\zeta_2+\zeta_3-\zeta_4=0$.
The vector $\l$ exists provided the pre-image of $\zeta_i$ has a non-empty intersection with the positive orthant. Moreover, if all
$\l$ are positive then the solution is not unique.
The Moore-Penrose pseudo-inverse of $\B$ is
$$
\B^+=\frac{1}{8}
\begin{bmatrix}
3& 3& -1& -1 \\ 3& -1 & 3 &-1 \\ -1 & 3 & -1 &  3 \\ -1& -1 & 3 & 3 \\
\end{bmatrix}.
$$

Since $\zeta_1+\zeta_4=\zeta_2+\zeta_3=\omega \eqdef \Oo(\w)$, the set of solutions is given by
$$
\begin{pmatrix}
\l_{12} \\ \l_{13} \\ \l_{24} \\ \l_{34} \\
\end{pmatrix}
=
\B^+ \cdot
\begin{pmatrix}
\zeta_1\\ \zeta_2 \\ \omega-\zeta_2 \\ \omega-\zeta_1
\end{pmatrix} +
 \frac{\delta}{2}
\begin{pmatrix}
-1 \\ 1 \\ 1 \\ -1
\end{pmatrix}
=\frac{1}{2}
\begin{pmatrix}
\zeta_1 + \zeta_2 -\delta \\ \zeta_1 - \zeta_2 +\delta\\ \zeta_2-\zeta_1 +\delta\\ 2 \omega-\zeta_1-\zeta_2-\delta
\end{pmatrix}
$$

for values of $\delta$ such that $\l_g \geq 0$. The latter constraint implies that we necessarily have
$$|\zeta_2-\zeta_1| \leq \delta \leq \min(\zeta_1+\zeta_2,\, 2 \omega-\zeta_1-\zeta_2)$$
W.l.o.g., we assume that $\zeta_1\leq \zeta_2\leq\omega-\zeta_2\leq \omega-\zeta_1$.
In that case the set of solutions in $\lambdab$ is parametrized by $\nu \in [0,1]$ with
$$\l_{12}=\nu\, \zeta_1, \quad \l_{13}=(1-\nu)\, \zeta_1, \quad \l_{24}=\zeta_2-\nu \, \zeta_1, \quad \l_{34}=\omega- \zeta_2-(1-\nu) \zeta_1.$$

In particular, we see that setting $\nu=0$ or $\nu=1$ respectively removes $\{1,2\}$ and $\{1,3\}$ from the group-support of $\vb$.

The case considered here is an example of the situation where the decomposition is not unique, which is characterised by lemma \ref{lem:uniqueness} in the next section.

\section{Uniqueness of the decomposition}\label{app:uniqueness}
In this section we give necessary and sufficient conditions for the support to be unique.
As in lemma \ref{Llhc}, we consider $\B$ the incidence matrix of the groups defined by $\B_{ig}=\mathbf{1}_{\{i \in g \}}$. As before we denote $\Gs$ the strong group-support, $\Js=\cup_{g \in \Gs} \, g$ 
and $J_0=\supp{\w}$. Denote by $\B_{J_0 \Gs}$ the submatrix of $\B$ whose rows are indexed by elements of the support of $\w$ and whose columns are indexed by elements of $\Gs$.
\begin{lemma}
\label{lem:uniqueness}
The decomposition is unique if and only if $\B_{J_0 \Gs}$ has full row rank.
\end{lemma}
\begin{proof}
By lemma \ref{lem:bijection}, the uniqueness of the decomposition is equivalent to the uniqueness of the solution $\lambdab$ to problem (\ref{eq:trick}), which we can rewrite

\begin{equation}
\label{bidual}
\min_{\lambdab  \in\,  {\RR}^{\m}_+}\;  \frac{1}{2} \sum_{i \in J_0} \frac{w_i^2}{ { \hspace{1mm} \, \displaystyle \sum_{g \ni \, i} \l_g \, } }+ \frac{1}{2} \sum_{g \in \Gs} \l_g.
\end{equation}
Notice that only the terms indexed by $i \in J_0$ and $g \in \Gs$ contribute.
Since the objective is a proper closed convex function with no direction of recession, this optimization problem admits at least one solution (the proof is the same as for \ref{lem:basic1}).
Since the gradient of the previous objective depends on $\l_g$ only through $\zeta_i=\sum_{g \ni \, i} \l_g, \; i \in J_0$, then any other vector $\lambdab_{\Gs} $ such that $\zetab_{J_0}=\B_{J_0  \Gs} \lambdab_{\Gs}$ is also solution. 
It is therefore clear that it is sufficient that the kernel of $\B_{J_0  \Gs}$ is not trivial, i.e., $\B_{J_0  \Gs}$ is row rank deficient, to have multiple solutions. Indeed let $\Hb \in \RR^{|J_0| \times K}$ be a basis of the kernel of $\B_{J_0  \Gs}$ and consider that, 
by definition of $\Gs$, for all $g \in \Gs, \: \l_g >0$. 
As a consequence, there must exist a neighborhood $\mathcal{U}$  of $0$ in $\RR^{K}$ such that for all $\q \in \mathcal{U}$,
$\lambdab_{\Gs}+ \Hb \q$ has positive components.  Since $\zetab_{J_0} = \B_{J_0  \Gs} (\lambdab_{\Gs}+ \Hb \q)$, we have that $\lambdab_{\Gs}+ \Hb \q$ is another solution of the KKT conditions. 

We now prove that $\B_{J_0  \Gs}$  being of full row rank is sufficient to ensure the uniqueness of the decomposition. 
Indeed, we show next that when $\B_{J_0  \Gs}$ is of full row rank, the hessian of the objective, restricted to the non-zero $\l_g$ of (\ref{bidual}) is positive definite, so that the objective is strictly convex and the optimum is therefore unique.
The hessian is $\Qb=(\Qb_{gg'})_{g, g' \in \Gs}$ with
$$\Qb_{gg'}=\sum_{i \in \, g \, \cap g'} \frac{w_i^2}{\left ( \sum_{\tilde{g} \ni \, i } \l_{\tilde{g}} \right )^3}=\B_{J_0  \Gs}^{\top} \, \Db \, \B_{J_0  \Gs} \quad \text{and} \quad \Db=\text{diag}\left (w_i^2\left ( \textstyle \sum_{\tilde{g} \ni \, i } \l_{\tilde{g}} \right )^{-3} \right )_{\!i \in \, J_0}.$$
Since $D$ is a diagonal matrix with non-zero coefficients, $H$ is p.s.d. iff $\B_{J_0  \Gs}$ is full row rank which concludes the proof.
\end{proof}

%% file: lglasso-hal.bbl
\begin{thebibliography}{52}
\providecommand{\natexlab}[1]{#1}
\providecommand{\url}[1]{\texttt{#1}}
\expandafter\ifx\csname urlstyle\endcsname\relax
  \providecommand{\doi}[1]{doi: #1}\else
  \providecommand{\doi}{doi: \begingroup \urlstyle{rm}\Url}\fi

\bibitem[Agarwal et~al.(2011)Agarwal, Negahban, and
  Wainwright]{Agarwal2011Noisy}
A.~Agarwal, S.~Negahban, and M.J. Wainwright.
\newblock Noisy matrix decomposition via convex relaxation: Optimal rates in
  high dimensions.
\newblock Technical Report 1102.4807, arXiv, 2011.
\newblock URL \url{http://arxiv.org/abs/1102.4807}.

\bibitem[Bach(2008{\natexlab{a}})]{Bach2008Consistencya}
F.~Bach.
\newblock Consistency of the group lasso and multiple kernel learning.
\newblock \emph{J. Mach. Learn. Res.}, 9:\penalty0 1179--1225,
  2008{\natexlab{a}}.
\newblock URL \url{http://jmlr.csail.mit.edu/papers/v9/bach08b.html}.

\bibitem[Bach(2009)]{Bach2009Exploring}
F.~Bach.
\newblock Exploring large feature spaces with hierarchical multiple kernel
  learning.
\newblock In \emph{Adv. Neural. Inform. Process Syst.}, volume~21, pages
  105--112, 2009.

\bibitem[Bach(2010)]{Bach2010Structured}
F.~Bach.
\newblock Structured sparsity-inducing norms through submodular functions.
\newblock Technical Report 1008.4220, arXiv, 2010.
\newblock URL \url{http://arxiv.org/abs/1008.4220}.

\bibitem[Bach et~al.(2011)Bach, Jenatton, Mairal, and
  Obozinski]{Bach2011Optimization}
F.~Bach, R.~Jenatton, J.~Mairal, and G.~Obozinski.
\newblock Optimization with sparsity-inducing penalties.
\newblock Technical Report hal-00613125, HAL, 2011.
\newblock URL \url{http://hal.archives-ouvertes.fr/hal-00613125/fr/}.

\bibitem[Bach(2008{\natexlab{b}})]{Bach2008Bolasso}
F.~R. Bach.
\newblock Bolasso: model consistent lasso estimation through the bootstrap.
\newblock In \emph{ICML '08: Proceedings of the 25th international conference
  on Machine learning}, pages 33--40, New York, NY, USA, 2008{\natexlab{b}}.
  ACM.
\newblock ISBN 978-1-60558-205-4.
\newblock \doi{http://doi.acm.org/10.1145/1390156.1390161}.

\bibitem[Bach et~al.(2004)Bach, Lanckriet, and Jordan]{Bach2004Multiple}
F.~R. Bach, G.~R.~G. Lanckriet, and M.~I. Jordan.
\newblock Multiple kernel learning, conic duality, and the {SMO} algorithm.
\newblock In \emph{Proceedings of the Twenty-First International Conference on
  Machine Learning}, page~6, New York, NY, USA, 2004. ACM.
\newblock \doi{http://doi.acm.org/10.1145/1015330.1015424}.

\bibitem[Baraniuk et~al.(2010)Baraniuk, Cevher, Duarte, and
  Hegde]{Baraniuk2010Model}
R.G. Baraniuk, V.~Cevher, M.F. Duarte, and C.~Hegde.
\newblock Model-based compressive sensing.
\newblock \emph{Information Theory, IEEE Transactions on}, 56\penalty0
  (4):\penalty0 1982--2001, 2010.

\bibitem[Berge(1959)]{Berge1959Espaces}
C.~Berge.
\newblock \emph{Espaces topologiques et fonctions multivoques}.
\newblock Dunod, Paris, 1959.

\bibitem[Bickel et~al.(2009)Bickel, Ritov, and
  Tsybakov]{Bickel2009Simultaneous}
P.~J. Bickel, Y.~Ritov, and A.~Tsybakov.
\newblock Simultaneous analysis of {L}asso and {D}antzig selector.
\newblock \emph{Ann. Stat.}, 37\penalty0 (4):\penalty0 1705--1732, 2009.

\bibitem[Border(1985)]{Border1985Fixed}
K.~C. Border.
\newblock \emph{Fixed point theorems with applications to economics and game
  theory}.
\newblock Cambridge University Press, Cambridge, UK, 1985.

\bibitem[Candes et~al.(2009)Candes, Li, Ma, and Wright]{Candes2009Robust}
E.J. Candes, X.~Li, Y.~Ma, and J.~Wright.
\newblock Robust principal component analysis?
\newblock \emph{Journal of the ACM}, 58\penalty0 (1):\penalty0 1--37, 2009.

\bibitem[Chandrasekaran et~al.(2010)Chandrasekaran, Recht, Parrilo, and
  Willsky]{Chandrasekaran2010Convex}
V.~Chandrasekaran, B.~Recht, P.A. Parrilo, and A.S. Willsky.
\newblock The convex geometry of linear inverse problems.
\newblock Technical Report 1012.0621, arXiv, 2010.
\newblock URL \url{http://arxiv.org/abs/1012.0621}.

\bibitem[Chen et~al.(1998)Chen, Donoho, and Saunders]{Chen1998Atomic}
S.~S. Chen, D.~L. Donoho, and M.~Saunders.
\newblock Atomic decomposition by basis pursuit.
\newblock \emph{SIAM J. Sci. Comput.}, 20\penalty0 (1):\penalty0 33--61, 1998.
\newblock \doi{10.1137/S1064827596304010}.
\newblock URL \url{http://dx.doi.org/10.1137/S1064827596304010}.

\bibitem[Chen et~al.(2011)Chen, Xu, Caramanis, and Sanghavi]{Chen2011Robust}
Y.~Chen, H.~Xu, C.~Caramanis, and S.~Sanghavi.
\newblock Robust matrix completion with corrupted columns.
\newblock Technical Report 1102.2254, arXiv, 2011.
\newblock URL \url{http://arxiv.org/abs/1102.2254}.

\bibitem[Chuang et~al.(2007)Chuang, Lee, Liu, Lee, and
  Ideker]{Chuang2007Network-based}
H.-Y. Chuang, E.~Lee, Y.-T. Liu, D.~Lee, and T.~Ideker.
\newblock Network-based classification of breast cancer metastasis.
\newblock \emph{Mol. Syst. Biol.}, 3:\penalty0 140, 2007.
\newblock \doi{10.1038/msb4100180}.
\newblock URL \url{http://dx.doi.org/10.1038/msb4100180}.

\bibitem[Efron et~al.(2004)Efron, Hastie, Johnstone, and
  Tibshirani]{Efron2004Least}
B.~Efron, T.~Hastie, I.~Johnstone, and R.~Tibshirani.
\newblock Least angle regression.
\newblock \emph{Ann. Stat.}, 32\penalty0 (2):\penalty0 407--499, 2004.

\bibitem[He and Carin(2009)]{He2009Exploiting}
L.~He and L.~Carin.
\newblock Exploiting structure in wavelet-based {B}ayesian compressive sensing.
\newblock \emph{IEEE Transactions on Signal Processing}, 57:\penalty0
  3488--3497, 2009.

\bibitem[Hiriart-Urruty and Lemar{\'e}chal(1994)]{Hiriart1994Convex}
J.B. Hiriart-Urruty and C.~Lemar{\'e}chal.
\newblock \emph{Convex Analysis and Minimization Algorithms I.: Fundamentals}.
\newblock Springer-Verlag, 1994.

\bibitem[Huang and Zhang(2010)]{Huang2010benefit}
J.~Huang and T.~Zhang.
\newblock The benefit of group sparsity.
\newblock \emph{Ann. Stat.}, 38\penalty0 (4):\penalty0 1978--2004, 2010.
\newblock \doi{10.1214/09-AOS778}.
\newblock URL \url{http://dx.doi.org/10.1214/09-AOS778}.

\bibitem[Huang et~al.(2009)Huang, Zhang, and Metaxas]{Huang2009Learning}
J.~Huang, T.~Zhang, and D.~Metaxas.
\newblock Learning with structured sparsity.
\newblock Technical Report 0903.3002, arXiv, 2009.
\newblock URL \url{http://arxiv.org/abs/0903.3002}.

\bibitem[Jacob et~al.(2009)Jacob, Obozinski, and Vert]{Jacob2009Group}
L.~Jacob, G.~Obozinski, and J.-P. Vert.
\newblock Group lasso with overlap and graph lasso.
\newblock In \emph{ICML '09: Proceedings of the 26th Annual International
  Conference on Machine Learning}, pages 433--440, New York, NY, USA, 2009.
  ACM.
\newblock ISBN 978-1-60558-516-1.
\newblock \doi{http://doi.acm.org/10.1145/1553374.1553431}.

\bibitem[Jalali et~al.(2010)Jalali, Ravikumar, Sanghavi, and
  Ruan]{Jalali2010Dirty}
Ali Jalali, Pradeep Ravikumar, Sujay Sanghavi, and Chao Ruan.
\newblock A dirty model for multi-task learning.
\newblock In J.~Lafferty, C.~K.~I. Williams, J.~Shawe-Taylor, R.S. Zemel, and
  A.~Culotta, editors, \emph{Adv. Neural. Inform. Process Syst.}, pages
  964--972. Kaufmann publishers, 2010.

\bibitem[Jenatton et~al.(2009)Jenatton, Audibert, and
  Bach]{Jenatton2009Structured}
R.~Jenatton, J.-Y. Audibert, and F.~Bach.
\newblock Structured variable selection with sparsity-inducing norms.
\newblock Technical Report 0904.3523, arXiv, 2009.
\newblock URL \url{http://fr.arxiv.org/abs/0904.3523}.

\bibitem[Jenatton et~al.(2011)Jenatton, Mairal, Obozinski, and
  Bach]{Jenatton2011Proximal}
R.~Jenatton, J.~Mairal, G.~Obozinski, and F.~Bach.
\newblock Proximal methods for hierarchical sparse coding.
\newblock \emph{J. Mach. Learn. Res.}, 12\penalty0 (Jul):\penalty0 2297--2334,
  2011.
\newblock URL \url{http://jmlr.csail.mit.edu/papers/v12/jenatton11a.html}.

\bibitem[Knight and Fu(2000)]{Knight2000Asymptotics}
K.~Knight and W.~Fu.
\newblock Asymptotics for lasso-type estimators.
\newblock \emph{Ann. Stat.}, 28\penalty0 (5):\penalty0 1356--1378, 2000.
\newblock \doi{doi:10.1214/aos/1015957397}.
\newblock URL \url{http://dx.doi.org/10.1214/aos/1015957397}.

\bibitem[Kolar et~al.(2011)Kolar, Lafferty, and Wasserman]{Kolar2011Union}
M.~Kolar, J.~Lafferty, and L.~Wasserman.
\newblock Union support recovery in multi-task learning.
\newblock \emph{J. Mach. Learn. Res.}, 12:\penalty0 2415--2435, 2011.

\bibitem[Lanckriet et~al.(2004)Lanckriet, Cristianini, Bartlett, El~Ghaoui, and
  Jordan]{Lanckriet2004Learning}
G.R.G. Lanckriet, N.~Cristianini, P.~Bartlett, L.~El~Ghaoui, and M.I. Jordan.
\newblock Learning the kernel matrix with semidefinite programming.
\newblock \emph{J. Mach. Learn. Res.}, 5:\penalty0 27--72, 2004.
\newblock URL \url{http://www.jmlr.org/papers/v5/lanckriet04a.html}.

\bibitem[Leng et~al.(2004)Leng, Lin, and Wahba]{Leng2004note}
C.~Leng, Y.~Lin, and G.~Wahba.
\newblock A note on the {L}asso and related procedures in model selection.
\newblock \emph{Statistica Sinica}, 16\penalty0 (4):\penalty0 1273--1284, 2004.

\bibitem[Lounici(2008)]{Lounici2008Sup-norm}
K.~Lounici.
\newblock Sup-norm convergence rate and sign concentration property of lasso
  and dantzig estimators.
\newblock \emph{Electron. J. Statist.}, 2:\penalty0 90--102, 2008.
\newblock \doi{10.1214/08-EJS177}.
\newblock URL \url{http://dx.doi.org/10.1214/08-EJS177}.

\bibitem[Lounici et~al.(2010)Lounici, Pontil, Tsybakov, and Van
  De~Geer]{Lounici2010Oracle}
K.~Lounici, M.~Pontil, A.B. Tsybakov, and S.~Van De~Geer.
\newblock Oracle inequalities and optimal inference under group sparsity.
\newblock Technical Report 1007.1771, arXiv, 2010.
\newblock URL \url{http://arxiv.org/abs/1007.1771}.
\newblock To appear in the Annals of Statistics.

\bibitem[Lounici et~al.(2009)Lounici, Pontil, Tsybakov, and van~de
  Geer]{Lounici2009Taking}
Karim Lounici, Massimiliano Pontil, Alexandre~B. Tsybakov, and Sara van~de
  Geer.
\newblock Taking advantage of sparsity in multi-task learning.
\newblock In \emph{Proceedings of COLT}, 2009.

\bibitem[Maurer and Pontil(2011)]{Maurer2011Structured}
A.~Maurer and M.~Pontil.
\newblock Structured sparsity and generalization.
\newblock Technical Report 1108.3476, arXiv, 2011.
\newblock URL \url{http://arxiv.org/abs/1108.3476}.

\bibitem[Meier et~al.(2008)Meier, van~de Geer, and B{\"u}hlmann]{Meier2008The}
L.~Meier, S.~van~de Geer, and P.~B{\"u}hlmann.
\newblock The group lasso for logistic regression.
\newblock \emph{J. R. Stat. Soc. Ser. B}, 70\penalty0 (1):\penalty0 53--71,
  2008.
\newblock \doi{10.1111/j.1467-9868.2007.00627.x}.
\newblock URL \url{http://dx.doi.org/10.1111/j.1467-9868.2007.00627.x}.

\bibitem[Micchelli et~al.(2011)Micchelli, Morales, and
  Pontil]{Micchelli2011Regularizers}
C.A. Micchelli, J.M. Morales, and M.~Pontil.
\newblock Regularizers for structured sparsity.
\newblock Technical Report 1010.0556, arXiv, 2011.
\newblock URL \url{http://arxiv.org/abs/1010.0556}.

\bibitem[Mosci et~al.(2010)Mosci, Villa, Verri, and Rosasco]{Mosci2010Primal}
S.~Mosci, S.~Villa, A.~Verri, and L.~Rosasco.
\newblock A primal-dual algorithm for group sparse regularization with
  overlapping groups.
\newblock In J.~Lafferty, C.~K.~I. Williams, J.~Shawe-Taylor, R.S. Zemel, and
  A.~Culotta, editors, \emph{Adv. Neural. Inform. Process Syst.}, pages
  2604--2612. Kaufmann publishers, 2010.

\bibitem[Negahban and Wainwright(2011)]{Negahban2011Simultaneous}
S.N. Negahban and M.J. Wainwright.
\newblock Simultaneous support recovery in high dimensions: Benefits and perils
  of block $\ell_ 1/\ell_\infty$-regularization.
\newblock \emph{Information Theory, IEEE Transactions on}, 57\penalty0
  (6):\penalty0 3841--3863, 2011.

\bibitem[Nocedal and Wright(2006)]{Nocedal2006Numerical}
J.~Nocedal and S.~Wright.
\newblock \emph{Numerical optimization}.
\newblock Springer, 2006.

\bibitem[Obozinski and F.(2011)]{Obozinski2011Convex}
G.~Obozinski and Bach F.
\newblock Convex relaxation of combinatorial penalties.
\newblock Technical report, 2011.
\newblock In preparation.

\bibitem[Obozinski et~al.(2010)Obozinski, Taskar, and
  Jordan]{Obozinski2010Joint}
G.~Obozinski, B.~Taskar, and M.I. Jordan.
\newblock Joint covariate selection and joint subspace selection for multiple
  classification problems.
\newblock \emph{Statistics and Computing}, 20\penalty0 (2):\penalty0 231--252,
  2010.

\bibitem[Percival(2011)]{Percival2011Theoretical}
D.~Percival.
\newblock Theoretical properties of the overlapping groups lasso.
\newblock Technical Report 1103.4614, arXiv, 2011.
\newblock URL \url{http://arxiv.org/abs/1103.4614}.

\bibitem[Rakotomamonjy et~al.(2008)Rakotomamonjy, Bach, Canu, and
  Grandvalet]{Rakotomamonjy2008SimpleMKL}
A.~Rakotomamonjy, F.~Bach, S.~Canu, and Y.~Grandvalet.
\newblock {SimpleMKL}.
\newblock \emph{J. Mach. Learn. Res.}, 9:\penalty0 2491--2521, 2008.

\bibitem[Rockafellar(1997)]{Rockafellar1997Convex}
R.T. Rockafellar.
\newblock \emph{Convex Analysis}.
\newblock Princeton Univ. Press, 1997.

\bibitem[Roth and Fischer(2008)]{Roth2008The}
V.~Roth and B.~Fischer.
\newblock The group-lasso for generalized linear models: uniqueness of
  solutions and efficient algorithms.
\newblock In \emph{ICML '08: Proceedings of the 25th international conference
  on Machine learning}, pages 848--855, 2008.

\bibitem[Subramanian et~al.(2005)Subramanian, Tamayo, Mootha, Mukherjee, Ebert,
  Gillette, Paulovich, Pomeroy, Golub, Lander, and
  Mesirov]{Subramanian2005Gene}
A.~Subramanian, P.~Tamayo, V.~K. Mootha, S.~Mukherjee, B.~L. Ebert, M.~A.
  Gillette, A.~Paulovich, S.~L. Pomeroy, T.~R. Golub, E.~S. Lander, and J.~P.
  Mesirov.
\newblock Gene set enrichment analysis: a knowledge-based approach for
  interpreting genome-wide expression profiles.
\newblock \emph{Proc. Natl. Acad. Sci. USA}, 102\penalty0 (43):\penalty0
  15545--15550, Oct 2005.
\newblock \doi{10.1073/pnas.0506580102}.
\newblock URL \url{http://dx.doi.org/10.1073/pnas.0506580102}.

\bibitem[Tibshirani(1996)]{Tibshirani1996Regression}
R.~Tibshirani.
\newblock Regression shrinkage and selection via the lasso.
\newblock \emph{J. R. Stat. Soc. Ser. B}, 58\penalty0 (1):\penalty0 267--288,
  1996.

\bibitem[van~de Geer(2010)]{Geer2010L1}
S.~van~de Geer.
\newblock $\ell_1$-regularization in high-dimensional statistical models.
\newblock In \emph{Proceedings of the International Congress of
  Mathematicians}, volume~4, pages 2251--2369, 2010.

\bibitem[van~de Vijver et~al.(2002)van~de Vijver, He, van't Veer, Dai, Hart,
  Voskuil, Schreiber, Peterse, Roberts, Marton, Parrish, Atsma, Witteveen,
  Glas, Delahaye, van~der Velde, Bartelink, Rodenhuis, Rutgers, Friend, and
  Bernards]{Vijver2002gene-expression}
M.~J. van~de Vijver, Y.~D. He, L.~J. van't Veer, H.~Dai, A.~A.~M. Hart, D.~W.
  Voskuil, G.~J. Schreiber, J.~L. Peterse, C.~Roberts, M.~J. Marton,
  M.~Parrish, D.~Atsma, A.~Witteveen, A.~Glas, L.~Delahaye, T.~van~der Velde,
  H.~Bartelink, S.~Rodenhuis, E.~T. Rutgers, S.~H. Friend, and R.~Bernards.
\newblock A gene-expression signature as a predictor of survival in breast
  cancer.
\newblock \emph{N. Engl. J. Med.}, 347\penalty0 (25):\penalty0 1999--2009, Dec
  2002.
\newblock \doi{10.1056/NEJMoa021967}.
\newblock URL \url{http://dx.doi.org/10.1056/NEJMoa021967}.

\bibitem[Wainwright(2009)]{Wainwright2009Sharp}
M.~J. Wainwright.
\newblock Sharp thresholds for high-dimensional and noisy sparsity recovery
  using $\ell_1$-constrained quadratic programming (lasso).
\newblock \emph{IEEE T. Inform. Theory.}, 55\penalty0 (5):\penalty0 2183--2202,
  2009.
\newblock \doi{10.1109/TIT.2009.2016018}.
\newblock URL \url{http://dx.doi.org/10.1109/TIT.2009.2016018}.

\bibitem[Yuan and Lin(2006)]{Yuan2006Model}
M.~Yuan and Y.~Lin.
\newblock Model selection and estimation in regression with grouped variables.
\newblock \emph{J. R. Stat. Soc. Ser. B}, 68\penalty0 (1):\penalty0 49--67,
  2006.

\bibitem[Zhao and Yu(2006)]{Zhao2006model}
P.~Zhao and B.~Yu.
\newblock On model selection consistency of lasso.
\newblock \emph{J. Mach. Learn. Res.}, 7:\penalty0 2541, 2006.
\newblock URL \url{http://jmlr.csail.mit.edu/papers/v7/zhao06a.html}.

\bibitem[Zhao et~al.(2009)Zhao, Rocha, and Yu]{Zhao2008Grouped}
P.~Zhao, G.~Rocha, and B.~Yu.
\newblock Grouped and hierarchical model selection through composite absolute
  penalties.
\newblock \emph{Ann. Stat.}, 37\penalty0 (6A):\penalty0 3468--3497, 2009.

\end{thebibliography}
